\documentclass[a4paper, 11pt]{report}

\usepackage[T1]{fontenc}
\usepackage[utf8]{inputenc}

\usepackage{lmodern}
\usepackage{psl-cover}

\usepackage{color}
\usepackage{natbib}
\usepackage[colorlinks=true, citecolor=blue, urlcolor=blue]{hyperref}
\hypersetup{colorlinks,linkcolor=darkgray}

\usepackage{epigraph}  
\usepackage[scale=0.75]{geometry}
\usepackage{tgheros}	

\usepackage{appendix}
\usepackage{chngcntr}
\usepackage{etoolbox}

\AtBeginEnvironment{subappendices}{%
  \clearpage
  \section*{Appendices}
  \addcontentsline{toc}{section}{Appendices}
  \counterwithin{figure}{section}
  \counterwithin{table}{section}
}

\makeatletter  
\newcommand{\neutralize}[1]{\expandafter\let\csname c@#1\endcsname\count@}
\makeatother

\newtheorem{thm}{Theorem}

\newtheorem{lem}{lemma}

\usepackage{float}

\usepackage{minitoc}

\usepackage{algorithm}
\usepackage{algpseudocode}
\usepackage{amsmath}
\usepackage{amssymb}
\usepackage{import}
\usepackage{amsthm}
\usepackage{caption}
\usepackage{subcaption}

\usetikzlibrary{calc}

\usepackage{hyperref}
\usepackage{url}       
\usepackage{graphicx}
\usepackage{amsthm}    
\usepackage{nicefrac} 
\usepackage{booktabs}  
\usepackage{wrapfig}
\usepackage{multirow}
\usepackage{enumitem}

\usepackage[utf8]{inputenc} 
\usepackage[T1]{fontenc}    
\usepackage{amsfonts}       
\usepackage{microtype}      
\usepackage{xcolor}         

\usepackage{lettrine}
\usepackage{tcolorbox}

\newtheorem{theorem}{Theorem}[section]

\newtheorem{definition}[theorem]{Definition}
\newtheorem{lemma}[theorem]{Lemma}

\newtheorem{assumption}[theorem]{Assumption}
\newtheorem{remark}[theorem]{Remark}

\newcommand{\BEQ}{\begin{equation}}
\newcommand{\EEQ}{\end{equation}}

\newcommand{\BA}{\begin{array}}
\newcommand{\EA}{\end{array}}

\newcommand{\eg}{{\it e.g.}}

\newcommand{\reals}{{\mathbb R}}

\newcommand{\diag}{\mathop{\bf diag}}

\newcommand\mathbff{\mathbf}

\newcommand{\Real}{{\mathbb R}}
\newcommand{\st}{~~~\text{s.t.}~~~}
\newcommand\vs{\vspace*{-0.15cm}}
\newcommand\wb{{\mathbf w}}

\newcommand\ellipsebyfoci[4]{
       \path[#1] let \p1=(#2), \p2=(#3), \p3=( $(\p1)!.5!(\p2)$ )
        in \pgfextra{
            \pgfmathsetmacro{\angle}{atan2(\y2-\y1,\x2-\x1)}
            \pgfmathsetmacro{\focal}{veclen(\x2-\x1,\y2-\y1)/2/1cm}
            \pgfmathsetmacro{\lentotcm}{\focal*2*#4}
            \pgfmathsetmacro{\axeone}{(\lentotcm - 2 * \focal)/2+\focal}
            \pgfmathsetmacro{\axetwo}{sqrt((\lentotcm/2)*(\lentotcm/2)-\focal*\focal}
        }
        (\p3) ellipse[x radius=\axeone cm,y radius=\axetwo cm, rotate=\angle];
}

\newcommand\sbullet[1][.5]{\mathbin{\vcenter{\hbox{\scalebox{#1}{$\bullet$}}}}}
\newcommand\Ical{{\mathcal I}}
\newcommand\Iint{\mathring{\mathcal I}}
\newcommand\Xcal{{\mathcal X}}

\newcommand{\xb}{\mathbf{x}}
\newcommand{\Pb}{\mathbf{P}}
\newcommand{\Wb}{\mathbf{P}}
\newcommand{\zb}{\mathbf{z}}
\newcommand{\Hcal}{\mathcal{H}}
\newcommand{\kappab}{\boldsymbol{\kappa}}

\title{On Inductive Biases for Machine Learning in Data Constrained Settings}

\author{Grégoire Mialon}

\institute{l'\'Ecole Normale Supérieure}
\doctoralschool{\'Ecole Doctorale De Sciences Mathématiques de Paris}{386}
\specialty{Mathématiques \\ Appliquées}
\date{19 Janvier 2022}

\jurymember{1}{Gabriel Peyré}{Researcher, CNRS}{Rapporteur et Président du jury}
\jurymember{2}{Alexandre Gramfort}{Researcher, INRIA}{Rapporteur}
\jurymember{3}{Michael Bronstein}{Professor, Oxford/Twitter}{Examinateur}
\jurymember{4}{Pascal Frossard}{Professor, EPFL}{Examinateur}
\jurymember{5}{Anna Korba}{Assistant Professor, ENSAE/CREST}{Examinatrice}
\jurymember{6}{Alexandre d'Aspremont}{Researcher, CNRS}{Directeur de thèse}
\jurymember{7}{Julien Mairal}{Researcher, INRIA}{Directeur de thèse}

\frabstract{

    Apprendre à partir de données limitées est l'un des plus gros problèmes du deep learning. Les approches courantes et populaires de cette question consistent à entraîner un modèle sur d'énormes quantités de données, étiquetées ou non, avant de réentraîner le modèle sur un ensemble de données d'intérêt, plus petit, appartenant à la même modalité. Intuitivement, cette technique permet au modèle d'apprendre d'abord une représentation générale pour un certain type de données, telles que des images. Moins de données seront ensuite nécessaires pour apprendre une tâche spécifique pour cette modalité particulière. Bien que cette approche appelée « apprentissage par transfert » soit très efficace dans des domaines tels que la vision par ordinateur ou le traitement du langage naturel, elle ne résout pas les problèmes courants du deep learning tels que l'interprétabilité des modèles ou le besoin global en données. Cette thèse explore une réponse différente au problème de l'apprentissage de modèles expressifs dans des contextes où les données sont plus rares. Au lieu de s'appuyer sur de grands ensembles de données pour apprendre les paramètres d'un réseau de neurones, nous remplacerons certains de ces paramètres par des fonctions mathématiques connues reflétant la structure des données. Très souvent, ces fonctions seront puisées dans la riche littérature des méthodes à noyau. En effet, de nombreux noyaux peuvent être interprétés, et/ou permettre un apprentissage avec peu de données. Notre approche s'inscrit dans le cadre des « biais inductifs », qui peuvent être définis comme des hypothèses sur les données disponibles restreignant l'espace des modèles à explorer lors de l'apprentissage. Dans les deux premiers chapitres de la thèse, nous démontrons l'efficacité de cette approche dans le cadre de séquences, telles que des phrases en langage naturel ou des séquences protéiques, et de graphes, tels que des molécules. Nous soulignons également la relation entre notre travail et les progrès récents du deep learning. Le dernier chapitre de cette thèse se concentre sur les modèles d'apprentissage automatique convexes. Ici, plutôt que de proposer de nouveaux modèles, nous nous demandons quelle proportion des échantillons d'un jeu de données est vraiment nécessaire pour apprendre un « bon » modèle. Plus précisément, nous étudions le problème du filtrage sûr des échantillons, c'est-à-dire l'exécution de tests simples afin d'éliminer les échantillons non informatifs d'un ensemble de données avant même d'entraîner un modèle d'apprentissage automatique, sans affecter le modèle optimal. De telles techniques peuvent être utilisées pour compresser des jeux de données ou extraire des échantillons rares.

}

\enabstract{

    Learning with limited data is one of the biggest problems of deep learning. Current, popular approaches to this issue consist in training models on huge amounts of data, labelled or not, before re-training the model on a smaller dataset of interest from the same modality. Intuitively, this technique allows the model to learn a general representation for some kind of data first, such as images. Then, fewer data should be required to learn a specific task for this particular modality. While this approach coined as "transfer learning" is very effective in domains such as computer vision or natural langage processing, it does not solve common problems of deep learning such as model interpretability or the overall need for data. This thesis explores a different answer to the problem of learning expressive models in data constrained settings. Instead of relying on big datasets to learn the parameters of a neural network, we will replace some of them by known functions reflecting the structure of the data. Very often, these functions will be drawn from the rich litterature of kernel methods. Indeed, many kernels can be interpreted, and/or allow for learning with few data. Our approach falls under the hood of "inductive biases", which can be defined as hypothesis on the data at hand restricting the space of models to explore during learning. In the first two chapters of the thesis, we demonstrate the effectiveness of this approach in the context of sequences,  such as sentences in natural language or protein sequences, and graphs, such as molecules. We also highlight the relationship between our work and recent advances in deep learning. The last chapter of this thesis focuses on convex machine learning models. Here, rather than proposing new models, we wonder which proportion of the samples in a dataset is really needed to learn a "good" model. More precisely, we study the problem of safe sample screening, i.e, executing simple tests to discard uninformative samples from a dataset even before fitting a machine learning model, without affecting the optimal model. Such techniques can be used to compress datasets or mine for rare samples.

}

\frkeywords{Apprentissage automatique, apprentissage profond, méthodes à noyau, optimisation convexe}
\enkeywords{Machine learning, deep learning, kernel methods, convex optimization}

\begin{document}

\dominitoc

\maketitle{}

\chapter*{Abstract}

Learning with limited data is one of the biggest problems of deep learning. Current, popular approaches to this issue consist in training models on huge amounts of data, labelled or not, before re-training the model on a smaller dataset of interest from the same modality. Intuitively, this technique allows the model to learn a general representation for some kind of data first, such as images. Then, fewer data should be required to learn a specific task for this particular modality. While this approach coined as "transfer learning" is very effective in domains such as computer vision or natural langage processing, it does not solve common problems of deep learning such as model interpretability or the overall need for data. This thesis explores a different answer to the problem of learning expressive models in data constrained settings. Instead of relying on big datasets to learn the parameters of a neural network, we will replace some of them by known functions reflecting the structure of the data. Very often, these functions will be drawn from the rich litterature of kernel methods. Indeed, many kernels can be interpreted, and/or allow for learning with few data. Our approach falls under the hood of "inductive biases", which can be defined as hypothesis on the data at hand restricting the space of models to explore during learning. In the first two chapters of the thesis, we demonstrate the effectiveness of this approach in the context of sequences,  such as sentences in natural language or protein sequences, and graphs, such as molecules. We also highlight the relationship between our work and recent advances in deep learning. The last chapter of this thesis focuses on convex machine learning models. Here, rather than proposing new models, we wonder which proportion of the samples in a dataset is really needed to learn a "good" model. More precisely, we study the problem of safe sample screening, i.e, executing simple tests to discard uninformative samples from a dataset even before fitting a machine learning model, without affecting the optimal model. Such techniques can be used to compress datasets or mine for rare samples.

\chapter*{Résumé}

Apprendre à partir de données limitées est l'un des plus gros problèmes du deep learning. Les approches courantes et populaires de cette question consistent à entraîner un modèle sur d'énormes quantités de données, étiquetées ou non, avant de réentraîner le modèle sur un ensemble de données d'intérêt, plus petit, appartenant à la même modalité. Intuitivement, cette technique permet au modèle d'apprendre d'abord une représentation générale pour un certain type de données, telles que des images. Moins de données seront ensuite nécessaires pour apprendre une tâche spécifique pour cette modalité particulière. Bien que cette approche appelée « apprentissage par transfert » soit très efficace dans des domaines tels que la vision par ordinateur ou le traitement du langage naturel, elle ne résout pas les problèmes courants du deep learning tels que l'interprétabilité des modèles ou le besoin global en données. Cette thèse explore une réponse différente au problème de l'apprentissage de modèles expressifs dans des contextes où les données sont plus rares. Au lieu de s'appuyer sur de grands ensembles de données pour apprendre les paramètres d'un réseau de neurones, nous remplacerons certains de ces paramètres par des fonctions mathématiques connues reflétant la structure des données. Très souvent, ces fonctions seront puisées dans la riche littérature des méthodes à noyau. En effet, de nombreux noyaux peuvent être interprétés, et/ou permettre un apprentissage avec peu de données. Notre approche s'inscrit dans le cadre des « biais inductifs », qui peuvent être définis comme des hypothèses sur les données disponibles restreignant l'espace des modèles à explorer lors de l'apprentissage. Dans les deux premiers chapitres de la thèse, nous démontrons l'efficacité de cette approche dans le cadre de séquences, telles que des phrases en langage naturel ou des séquences protéiques, et de graphes, tels que des molécules. Nous soulignons également la relation entre notre travail et les progrès récents du deep learning. Le dernier chapitre de cette thèse se concentre sur les modèles d'apprentissage automatique convexes. Ici, plutôt que de proposer de nouveaux modèles, nous nous demandons quelle proportion des échantillons d'un jeu de données est vraiment nécessaire pour apprendre un « bon » modèle. Plus précisément, nous étudions le problème du filtrage sûr des échantillons, c'est-à-dire l'exécution de tests simples afin d'éliminer les échantillons non informatifs d'un ensemble de données avant même d'entraîner un modèle d'apprentissage automatique, sans affecter le modèle optimal. De telles techniques peuvent être utilisées pour compresser des jeux de données ou extraire des échantillons rares.

\begin{flushright}

    \textit{À mes grands-parents}

\end{flushright}

\chapter*{Remerciements}

Comme tout un chacun, j’ai pour habitude de lire les remerciements d’une thèse avec une attention teintée d’attentes. Y figurerai-je, dans ces remerciements ? Vais-je y glaner des détails personnels sur l’impétrant, aujourd’hui chercheur reconnu ? Non, les remerciements ne sont pas une tâche à prendre à la légère.

Mon aventure dans le monde de l’intelligence artificielle a commencé grâce à Jean Ponce et Jean-Philippe Vert, qui m’ont gentiment permis de suivre le master recherche Mathématiques, Vision, Apprentissage, et accepté de discuter avec moi de l’opportunité de faire une thèse.

J’ai ensuite eu la chance d'être encadré par Julien Mairal et Alexandre d’Aspremont. Ceux-ci m'ont bien souvent prodigué le conseil qui débloquait. Je veux ici les remercier sincèrement pour leur confiance, et pour avoir su créer un climat de travail serein durant ces trois années. Je ne me suis jamais senti esseulé.

Ce doctorat a été effectué au sein de l’Institut national de recherche en informatique et en automatique (Inria), encore trop peu connu en France, dont la mission est le développement de la recherche et de la valorisation en sciences et techniques de l'information et de la communication. L’Inria, en particulier le centre de Paris, m’a fourni un environnement de travail plus que confortable. Je pense ainsi à l'efficacité des démarches administrative, à la prévenance de la célèbre équipe communication, celle toute teintée de patience d'Éric Fleury et, plus prosaïquement, au confort de mon bureau. 

Plus précisément, j'ai fait partie de deux équipes de l’Inria : Sierra, à Paris, et Thoth, à Grenoble. J’ai été magnifiquement bien accueilli dans les deux. Sierra, c’est l’éclectisme scientifique, mais aussi des discussions quotidiennes. Je veux ici remercier Francis Bach pour l’attention qu’il porte au bien-être des doctorants. Thoth, c’est la chaleur de l’équipe et, évidemment, les montagnes. J’aimais mes visites à Grenoble. Merci aux assistantes des deux équipes, Hélène Bessin-Rousseau et Nathalie Gillot, pour leur gentillesse et leur patience, qui m’ont grandement facilité la vie.

Cette thèse n’aurait pas été possible sans financement, apporté dans mon cas par l’European Research Council (ERC). Je vais ici profiter de la gourmandise avec laquelle ces remerciements seront peut-être parcourus pour détailler ce système de financement méritant d'être connu au-delà du monde de la recherche scientifique, et qui m’a permis de mener ma recherche sereinement et librement. L’ERC est un organe de l'Union européenne chargé de coordonner les efforts de la recherche entre les États membres de l'Union, et la première agence de financement pan-européenne pour une « recherche à la frontière de la connaissance ». Elle propose sur concours différents financements. Pour un jeune chercheur souhaitant se lancer sur ses propres thématiques, l’ERC propose une Starting Grant : 1,5 millions d’euros sur 5 ans, avec une très grande liberté d’utilisation. L’ERC Starting Grant sur laquelle j’ai été financé a été obtenue dans mon cas par Julien Mairal en 2016 (Projet SOLARIS). J'ai donc, durant ces trois ans, toujours eu les moyens de participer à une conférence, une école d'été, ou plus simplement me rendre à Grenoble : là encore, tous les chercheurs n'ont pas cette chance.

La présente thèse a bénéficié du supercalculateur Jean Zay. Merci à ses équipes pour son aide et sa gestion, offrant un outil à la fois puissant et simple d'utilisation.

Il me reste à présent à remercier celles et ceux dont j'ai eu la chance de croiser le chemin, en demandant par avance pardon à mon lecteur en cas d'oubli. 
Merci, donc, à Adrien, Alessandro, Francis, Pierre et Umut pour Sierra, et à Justin pour Willow : votre porte était toujours ouverte pour moi et mes diverses interrogations. J'ai une pensée particulière pour Jean-Paul Lomond, dont je garde l'image d'un scientifique inspirant et qui avait le souci d'expliquer au public les grands enjeux de son domaine.
Merci à Ghislain, Loïc, et Xavier pour avoir abondamment conseillé le quasi-novice que j'étais sur le développement.
Merci à mes comparses de bureau, Benjamin, Manon, Mathieu, Radu, et Thomas : vous avez su endurer mes digressions.
Merci à Alex, Alexandre, Antoine, Bruno, Hadrien, Julia, Lenaic, Loucas\&Raphaël, Margaux, Rémi, Robin, Ronan, Thomas, Ulysse, Vivien, Yana, Yann, bref, à ma génération de doctorants et ceux du dessus, pour les grands moments scientifiques et moins scientifiques que nous avons partagé. 
Merci à Luca, notre pièce rapportée de Dyogene et à Arthur, le grand frère de thèse que j'ai connu en dehors du laboratoire.
Merci à Antoine, Bertille, Boris, Céline, Elliot, Eloïse, Etienne, Gaspard, Gautier, Oumayma, Pierre-Cyril, Pierre-Louis, Ricardo, Thomas, Wilson, et Yann, la génération suivante, avec laquelle j'aurais aimé passer plus de temps. 
Merci à Alberto, Dexiong, Mathilde, Thomas, et Valentin, les compagnons de toujours à Grenoble, et l'une de mes motivations pour y retourner.
Merci à Alexander, Alexandre, Andrei, Daan, Florent, Houssam, Juliette, Konstantin, Margot, Michael, Mikita, Minttu, Pauline, Theo, et Vladyslav : c'était un plaisir de vous voir à Grenoble et Montbonnot.  
Merci au jury, Alexandre Gramfort et Gabriel Peyré pour les rapporteurs, Michael Bronstein, Pascal Frossard, et Anna Korba pour les examinateurs : c'était une fierté de vous présenter mon travail et un plaisir d'en discuter.

Je remercie enfin ceux qui m’ont entouré ou m’entourent encore de leur amitié et de leur amour. Je pense en particulier à ma famille, ainsi qu'à celle qui est aujourd'hui mon étoile. 

\begin{flushright}
    Paris, le 27 janvier 2022 
\end{flushright}

\tableofcontents

\chapter{Introduction}
\label{chapt:1_intro}
\pagenumbering{arabic}  
\section{The state of machine learning}

\lettrine[lines=2, slope=-0.5em, lhang=0.5,nindent=0pt]{R}{ecent} advances in machine learning, namely deep learning~\citep{goodfellow-et-al-2016}, are slowly starting to deliver. Deep language models such as BERT~\citep{devlin2018} equip search engines to improve the relevance of the results. Researchers in bioinformatics can rely on strong models for protein structure prediction~\citep{jumper2021}, a crucial task to understand the properties of proteins, the building block of life, which would otherwise being costly in terms of time and money. Some programming task could soon be automated with new kinds of programming languages taking instructions in natural language as input~\citep{chen2021evaluating}. These three innovations share one common characteristic: their underlying deep neural networks were trained on very large scale datasets. For example, BERT was trained on 3.3 billion words from a corpus of books and English Wikipedia, representing 33,000 books, or ten times the minimum amount of books a library should contain according to UNESCO. Several language models have followed, trained on datasets a few order of magnitude bigger and this important trend is emerging for many other data modalities. Such an approach to machine learning is an answer to one of its biggest problem: generalization. Roughly, how to handle new, unseen data. We will briefly delve into the definition and history of machine learning, before elaborating on the generalization problem and how it is currently being addressed with pre-trained models. Then, we will explain why and how the following thesis provides a different answer to this problem and finally deliver its outline and contributions.

\subsection{Past and present of machine learning} 

\paragraph{What is machine learning?} In its best known form, machine learning is nothing else but a computer program learning to map an input A to an output B. To do so, the program requires training examples of A and corresponding B, a model whose parameters will be adjusted to fit the examples, and a metric for the model to know whether it does good or not and adjust its parameters accordingly. More precisely, and in many cases, training a machine learning model consists in solving the following problem coined as Empirical Risk Minimization (ERM)\ \textit{i.e} given a dataset of $n$ examples $x_i$ whose desired mapping is $y_i$, finding a prediction function $f$ parameterized by $\theta$ solving:
\begin{equation}
    \text{min}_{\theta} \frac{1}{n} \sum_{i=1}^n \ell(f_{\theta}(x_i), y_i) + \lambda \Omega(f_{\theta}),
\end{equation}
where $\ell$ is a loss term reflecting the discrepancy between a prediction $f_{\theta}(x_i)$ and the ground truth $y_i$, and $\Omega$ a term penalizing the complexity of the solution, according to Occam's razor.

Contrarily to pure optimization, whose objective would consist in optimizing the metric for the training data only, machine learning aims at finding a model which will also do well on new, \textit{unseen} data, an ability coined as \textit{generalization}. This is why we often add constraints to the ERM problem such as $\Omega$. As such, machine learning is a subfield of artificial intelligence, whose goal is to create computers and programs exhibiting ``intelligent" behavior - rigorously defining intelligence being an open question~\citep{chollet2019measure}.

But, why would we want a computer program to \textit{learn} to perform a task, such as predicting the content of an image, instead of directly implementing rules for this task? Because, in many cases, rules defined by a human programmer will be incomplete or nearly impossible to formulate. Think for example about a program designed to identify an animal from its picture. First, a dog could be distinguished from a cat by looking at its nose. But what if the nose is obfuscated? This would require to add an additional rule.
And, if this rule cannot be applied, another one, etc. Second, rules can be difficult or even impossible to formulate: how to express what a nose is in terms of pixels? In many cases, letting a program learn its own rules will circumvent this difficulties and expand the spectrum of tasks that can be automated as well as the level of performance. Until recently, and excepted older approaches to neural networks, machine learning mostly relied on models that could be learned using convex optimization: the training procedure often consisted in solving a convex optimization problem, which is well-understood today~\citep{boyd_van}.
Many of these models were linear, or achieved non-linearity via techniques such as the kernel trick (see below) and could lack of expressiveness or ability to generalize compared to more recent approaches such as deep neural networks. Somehow symmetrically, training procedures for deep learning are still not understood today and are an important research topic.

Finally, note that there is much more to machine learning than solving problems such as ERM. For example, an important area of machine learning is unsupervised learning, a setting where we only have samples $x_i$ but no labels $y_i$ and we generally seek to learn a general representation of the data that could be used for downstream tasks: formulating the learning task is therefore a problem in itself. 

\paragraph{Why is it important?} 

Recent advances in machine learning, namely deep learning (see~\citet{goodfellow-et-al-2016} for a detailed introduction), raised high expectations in the past years for its promise to help humans executing always more complex tasks with various degrees of supervision:
\begin{itemize}
    \item[-] Machine learning could (and is in the process of) automate tasks with a level of performance that were previously thought to be a human backyard such as image generation (with techniques such as Generative Adversarial Networks~\citep{goodfellow2014generative} and Variational Autoencoders~\citep{kingma2014autoencoding} among others), sentence generation in natural language~\citep{brown2020language} or simply playing games such as Go~\citep{schrittwieser2020mastering} to give but a few examples.
	\item[-] Even in contexts where human supervision is essential, machine learning could still be used to handle large amounts of data that would otherwise be intractable for a human, and find subtle patterns in complex data. In bioinformatics for example, the continuous development and declining cost of high-throughput technologies led to a dramatic increase in the quantity of omics data (\textit{i.e}, genomics, proteomics, metabolomics, metagenomics and transcriptomics) available to practitioners. One famous instance of this trend is the human genome: in the last twenty years, its cost of sequencing (3.3GB) fell from \$100,000,000 to \$1,000. In this context, machine learning could typically be used as a tool by researchers to find complex relationships between genes, metabolisms and environment resulting in onset and progression of diseases.   
\end{itemize}

\paragraph{Where does it come from?} 

Machine learning as a field emerged progressively during the last 80 years, and takes its foundations in statistics, probabilities, linear algebra, and computer science. Exhaustively retracing the history of machine learning would be out of the scope of this introduction and we will simply review (subjectively) pivotal moments of the field. At the heart of machine learning lies linear model fitting, such as linear regression, discovered in the $\text{XVIII}^e$ century or Support Vector Machine~\citep{cortes1995support}, which rely on convex optimization for solving the training problem. Such models have been flourishing since the beginning of machine learning and are still of importance today for their simplicity both in terms of interpretability and from a computational point of view. Years 2010 have seen sparse models blossom: adding penalties such as the $\ell_1$ norm to the training problem encourages the learned models to be sparse, making those even more interpretable and fast to optimize~\citep{bach2012optimization}. It is now well-known that a combination of increase in the quantity of data and compute power, briefly coined as the Big Data era, triggered the emergence of deep learning which is turning the state of the art in many fields upside down. Indeed, deep learning had been known for a while, but required more data than classical machine learning models to deliver its full potential. In 2012, a Convolutional Neural Network (CNN) designed for image classification, AlexNet, outperformed its non deep competitors at the ImageNet challenge for the first time~\citep{krizhevsky2012}. Since then, CNNs have been largely adopted as the cornerstone of many computer vision technologies such as image classification, face recognition, character recognition, and many others. This moment was also the spark that ignited massive research and investment in deep learning. In 2016, Lee Sedol, one of the very best Go players at the moment, was defeated by AlphaGo, an algorithm relying on deep reinforcement learning~\citep{silver2016mastering} to select the best possible moves. This event demonstrated that deep learning was able to achieve superhuman performances in fields that were previously thought to be out of reach for computers. In 2018, BERT, a language model based on transformers, a deep neural network architecture introduced the year before~\citep{vaswani2017}, beat other deep networks on important natural language processing benchmarks~\citep{devlin2018}. This work triggered a great deal of research on transformers and huge models trained on similarly huge datasets. Deep learning is still reaching new, impressive milestones as of today.  

\subsection{Generalization and other problems of deep learning}

Although deep learning generated high expectations by outperforming the state of the art on a wide range of tasks, its deployment in real life applications has been dampened for various reasons.

\begin{itemize}
    \item[-] Deep models can express complicated functions but require a lot of data to be able to generalize, \textit{i.e} make good predictions on data which has not been seen during the training. This can be problematic if one wants to use these models in data constrained settings such as studying the onset of rare diseases. In France for example, it is estimated that there are less than 30,000 people per rare disease, which means even less samples to study (2019). More generally, using many samples is in contradiction with animal intelligence, which seems to acquire deep knowledge of its environment by learning from a small amount of samples. For example, most humans need few pictures of horses to be able to recognize a new horse! Some may even be able to recognize unseen animals from their description. Thinking about how to learn with minimal amount of samples is probably an interesting path towards advancing machine learning. 
    \item[-] Deploying deep learning based technologies into real life products, some of them being critical objects such as surveillance systems or self-driving cars, may require interpretability to ensure the system complies with the norm, among other requirements\footnote{\href{https://eur-lex.europa.eu/legal-content/EN/TXT/?uri=CELEX:52021PC0206}{European Commission's proposal for a Regulation on Artificial Intelligence, 2021}}. Interpretability is also a desirable property of machine learning models in the setting of scientific discovery, as discovering key predictors for a particular phenomenon can suggest new avenues of research. 
    \item[-] It is well-known that deep learning models lack of robustness to perturbations that can happen when the model is deployed. Thus, slight modifications to the input of deep models can trigger important responses at the output level. This is the cornerstone of adversarial attacks in vision, where humanly imperceptible modifications of the input image (generally a few pixels) trick the network into predicting an obviously wrong label~\citep{szegedy2014intriguing}. This is an important vulnerability of deep models, which could be exploited to cause accidents with self-driving cars or avoid detection by facial recognition technologies. In Natural Language Processing, it is also possible to retrieve pieces of the training data by querying the model in the correct way~\citep{carlini2020extracting}. This could be a problem as many large models are pre-trained on huge, private datasets potentially containing sensitive data. 
    \item[-] Deep learning is generally associated to high computational and environmental costs \citep{strubell2019energy}. Indeed, the lack of theoretical understanding of the optimization procedures of deep neural networks requires to perform many training iterations to yield a satisfying model in terms of accuracy of predictions. Moreover, the size of the models and datasets, which have grown bigger and bigger in the past few years, and the popularity of such techniques in research and industry made the share of machine learning in overall computations overwhelming. Discussing the energy efficiency of these models would be a research topic on its own: the fact that the hardware is specializing for deep learning, improving its energy efficiency, has to be taken into account, but does it balance the cost of developing and producing these new chips? The energy mix of the country providing electricity to the computing infrastructure also has to be taken into account. Hence, it is very difficult to conclude regarding the real cost of deep learning. Be that as it may, getting to the same level of performance as AlexNet required 44 less compute in 2019 than in 2012~\citep{hernandez2020measuring}. 
\end{itemize}

Significant proportions of machine learning research are dedicated to solving these problems, which are obviously not an exhaustive list of machine learning problems worth studying. This thesis focuses on the generalization problem yet may include a discussion on the other issues throughout the chapters.

\subsection{A promising but limited answer to the generalization problem: transfer learning from pre-trained models}

Today, a popular answer to the problem of generalization lies in transfer learning from large pre-trained models. 

\begin{definition}[Transfer learning]
    Transfer learning consists in training a model that has already been trained on a dataset A, on a generally similar dataset B. By ``similar", we mean a dataset of the same modality but potentially different distribution.
\end{definition}
The intuition behind is that knowledge acquired while learning on A could be reused to handle B. For example, a computer vision model trained to classify images of animals could be used as a backbone to classify images of cars since all natural images share features such as edges, curves, or angles. Transfer learning is currently and typically used as follows: a model is trained on a huge dataset on a pretext task, with the aim to learn a general representation that will do well on different downstream tasks. Indeed, it is then possible to train the general model on a smaller, specialized dataset of interest on which training a deep model from scratch would have yielded poor performance or simply been impossible due to the lack of data.
For example, language models are trained by predicting masked words in a sentence using the other words. The assumption is that the model will learn a language representation that will enable it to do well on a specialized task, for example predicting whether a movie review is positive or not. 

This idea is not new and may have been formulated for the first time in 1976~\citep{bozi1976} (although in Croatian). The same wave that carried deep learning, \textit{i.e} unprecedented availability of computational power and data, enabled to pre-train models on always larger datasets. This was also allowed by the progress of self-supervised learning~\citep{he2020momentum}, a learning framework which does not require annotated data, thus giving access to massive, unlabelled corpus (labelling data, a base requirement of supervised learning, is generally costly). Models can now be trained on datasets a few orders of magnitude bigger than what was done a few years ago. When the author began to work on his thesis, transfer learning mainly consisted in using various convolutional neural networks trained on ImageNet, a classification dataset of 1 million images with 1000 labels, as a backbone for solving other computer vision tasks. However, after the emergence of transformers and BERT, this approach became standard for other data modalities. In the field of NLP, language models are now pre-trained on bigger and bigger datasets~\citep{brown2020language}, and for various languages~\citep{martin2020camembert}, before being fine-tuned for downstream tasks such as natural language understanding~\citep{wang2020}. This practice has also been extended to new data modalities such as protein sequences~\citep{rives2019biological} or graphs as will be explained in Chapter~\ref{chapt:3_graphit}. Models pre-trained on datasets one or two orders of magnitude bigger than ImageNet are also emerging in computer vision~\citep{dosovitskiy2021an}. In the same way as code of deep architectures was publicly released on GitHub, pre-trained models are now often made available and centralized on platforms such as the \texttt{transformers} library from HuggingFace~\citep{wolf2019}.

Pre-trained models are a promising answer to the problem of generalization in machine learning postulating that it is possible to learn a general purpose representation of a given modality, such as images. Indeed, natural images share low-level features such as edges, curves or colors. Avoiding to learn these features from scratch each time a model is trained enables to focus on learning features specific to the dataset, such as wheels or car door is model has to classify models of cars, thus allowing to use deep learning on relatively small datasets. Pre-trained models are currently a popular research topic. 

\paragraph{The limits of transfer learning.}

Pre-trained models are currently changing the practice of machine learning.
Although these models are quite successful at providing general purpose representations, they do not fundamentally solve the problems that deep learning is facing. In fact, they even underlines some issues: pre-trained models require many data by definition (although they are meant to be trained only once and for all), they are costly to train and manipulate because of their size, they are prone to bias and to adversarial or privacy attacks. GPT-3~\citep{brown2020language}, one of the most famous and strongest language models, was trained on half a trillion words, including Wikipedia and branches of the internet. This is a few thousand times the amount of words a human will hear and read in its lifetime. Although GPT-3 provides impressive performance in terms of natural language understanding or generation, it has also been put in evidence that its family of models, namely Large Language Models (LLM) have some of the undesirable effects described above such as a high computational cost and tendency to exhibit biases~\citep{bender2021on}. Moreover, not all domains provide access to enough data: in many cases, pre-trained models are simply not available, think about the example of rare diseases. Another potential problem of large pre-trained models is their decreasing return with respect to the quantity of data: it becomes more and more difficult to train models on bigger datasets with seemingly diminishing gains in terms of performance on downstream tasks: the current trend may not be sustainable. Finally, data efficient models are believed to be a path towards machine learning models that generalize better. For all these reasons, finding other solutions to improve the generalization of machine learning models is an important problem. 

\subsection{Motivation: inductive biases, or another solution to the generalization problem}

We provide a different answer to the generalization problem. Instead of training bigger models with more data, we will focus on data constrained settings, \textit{i.e} datasets ranging from a few hundreds to a few tenths of thousands of samples. We will keep using expressive models such as neural networks yet reduce the number of learnable parameters by replacing them by known mathematical functions, an approach belonging to the wide family of inductive biases. We now define the term inductive bias, relying heavily on the definition provided by~\cite{battaglia2018relational}.

\begin{definition}[Inductive bias]
    An inductive bias is an hypothesis allowing to prioritize one solution or interpretation over another, independent of the observed data.
    It expresses assumptions about either the data-generating process (linear least squares reflect the assumption that the data is a line corrupted by Gaussian noise) or the space of solutions (the $\ell_1$ penalty generally promotes a sparse solution). 
\end{definition}

\paragraph{Why are they useful?} Inductive bias are useful as they often trade flexibility of the learning model for \textit{improved sample complexity}, a useful property in data constrained settings. Another useful feature of many inductive biases is their \textit{interpretability}, as they are the result of a human made assumption on the data generating process or the space of solutions. A simple example of inductive bias in machine learning is the $\ell_1$ norm used to penalize the weights of the model, typically in the setting of linear regression. First, this term acts as a regularization: it prevents the model of overfitting by enforcing a simple solution to the training problem. Moreover, it can be interpreted, in the sense that the solution will be sparse in terms of features used to make a prediction. Intuitively, the model will select the most relevant features, a useful behavior in domains such as genomics where there are a few particular genes of interest among thousands of others when it comes to predict the onset of a particular disease.
As explained above, not every domain has access to vast amounts of data, such as stock prices in finance or more generally domains with rare events. Other domains may have plenty of data available but require interpretability of the models, such as researchers studying phenomenon in natural science, hence looking for interpretable patterns to further investigate and/or express theories.

\paragraph{Goal of this thesis.} The object of this thesis is to introduce algorithms for learning in data constrained settings and for various data modalities such as sequences or graphs when pre-trained models are not available or advantageous. By constrained data setting, we mean datasets ranging from a few hundreds to a few tenths of thousands of samples. To do so, we will heavily rely on the previously discussed inductive biases, and particularly on kernel methods which enable to encode inductive biases directly into a model. This thesis also builds on one of the most recent deep learning architecture for dealing with set input data, the transformer.

\section{Preliminaries}

In this section, we introduce concepts that will be used in at least two of the three chapters of this thesis.

\subsection{Kernel methods.}

\paragraph{What are kernel methods?} Kernel methods have been extensively used in machine learning, see~\citet{scholkopf2001learning} for a detailed introduction. They consist in mapping data living in some space $\mathcal{X}$ to a high or infinite dimensional space $\mathcal{H}$ via a kernel function $K: \mathcal{X} \times \mathcal{X} \mapsto \mathbb{R}$. The intuition behind this technique is that data lying in a high dimensional space may become more easily separable for linear models. Intuitively, $K$ acts as a measure of similarity between a pair of elements $(x, x')$ living in $\mathcal{H}$. $\mathcal{H}$ is a Reproducing Kernel Hilbert Space (RKHS) associated to the positive definite kernel $K$ via a mapping function $\Phi: \mathcal{X} \mapsto \mathcal{H}$: we have
\begin{equation}
K(x, x') = \langle \Phi(x), \Phi(x') \rangle_{\mathcal{H}}.
\end{equation}

\begin{definition}[Positive definite kernel]
    $K: \mathcal{X} \times \mathcal{X} \mapsto \Real$ is a positive definite kernel if it is symmetric, \textit{i.e} $\forall (x, x') \in \mathcal{X}^2, K(x, x') = K(x', x)$ and satisfies $\forall N \in \mathbb{N}, (x_1, \dots, x_N) \in \mathcal{X}^N \textit{ and } (a_1, \dots, a_N) \in \Real^N: \sum_{i=1}^N \sum_{j=1}^N a_i a_j K(x_i, x_j) \geq 0$.
\end{definition}

It can be shown that positive definiteness of $K$ is equivalent to the reproducing property: for any function $f \in \mathcal{H}$, we have
\begin{equation}
    f(x) = \langle f, \Phi(x) \rangle_{\mathcal{H}} \text{ for all } x \in \mathcal{H}.
\end{equation}

In particular, it is possible to use kernel methods with linear models when the solution to the empirical risk minimization problem can be written in terms of scalar products in $\mathcal{X}$ between pairs of samples $(x, x')$. We will exploit this property to extend some results in the non-linear case in Chapter~\ref{chapt:4_screening}. Then, one only needs to replace these products by the scalar product in $\mathcal{H}$, \textit{i.e} by $K(x, x')$, thus obtaining potentially non-linear decision functions. This classical approach has been coined as the kernel trick. The kernel trick has an implication that will be exploited in the context of this thesis: it will be possible to work on non-vectorial data such as graphs or strings, by replacing the scalar products by a well-chosen comparison functions between these objects. For example, in Chapter~\ref{chapt:2_otke}, we will derive an embedding for sets starting from a comparison function between two sets, which is based on the 2-Wasserstein distance between two sets of points $\mathbf{x}$ and $\mathbf{x'}$ living respectively in $\Real^{n \times d}$ and $\Real^{n' \times d}$.

Finally, the Representer theorem states that any solution to the empirical risk minimization problem has the form:
\begin{equation}
    \forall x \in \mathcal{X}, f(x) = \sum_{i=1}^n \alpha_i K(x_i, x).
\end{equation}
Hence, $f$ can be expressed in terms of $n$ evaluations of the kernel $K$ between the input and the $n$ known samples.

\paragraph{A famous example of kernel: the Gaussian kernel.} The Gaussian kernel, or radial basis function (RBF) will often be used in the context of this thesis. For a pair of elements of $\mathcal{X}$ $(x, x')$, the Gaussian kernel with bandwidth $\sigma$ is
\begin{equation}
    K(x, x') = e^{-\frac{||x - x'||^2_2}{2\sigma^2}}.
\end{equation}
The Gaussian kernel has interesting properties: all the points in the feature space actually lie on the unit sphere since we have:
\begin{equation}
    K(x, x) = 1 = ||\Phi(x)||^2_{\mathcal{H}}.
\end{equation}
It is also possible to express the distance induced by the Gaussian kernel in the feature space in terms of distance between $x$ and $x'$:
\begin{equation}
    d_K(x, x') = \sqrt{2 \left( 1 - e^{-\frac{||x - x'||^2_2}{2\sigma^2}} \right)}.
\end{equation}

Going further, it is also possible to express the distance between a set of known samples $\mathcal{S} = \{x_1, \dots, x_n\}$ and a new sample $x$:
\begin{equation}
    d_K(x, \mathcal{S}) = \sqrt{C - \frac{2}{n} \sum_{i=1}^n e^{-\frac{||x - x_i||^2_2}{2\sigma^2}} }.
\end{equation}
Then, given two sets $\mathcal{S}_1$ and $\mathcal{S}_2$ and a sample $x$, we can use this distance to know which set $x$ belongs to, using a decision function induced by $K$, as shown in Figure~\ref{fig:discrimination}.

\begin{figure}
    \centering
    \includegraphics[scale=.3]{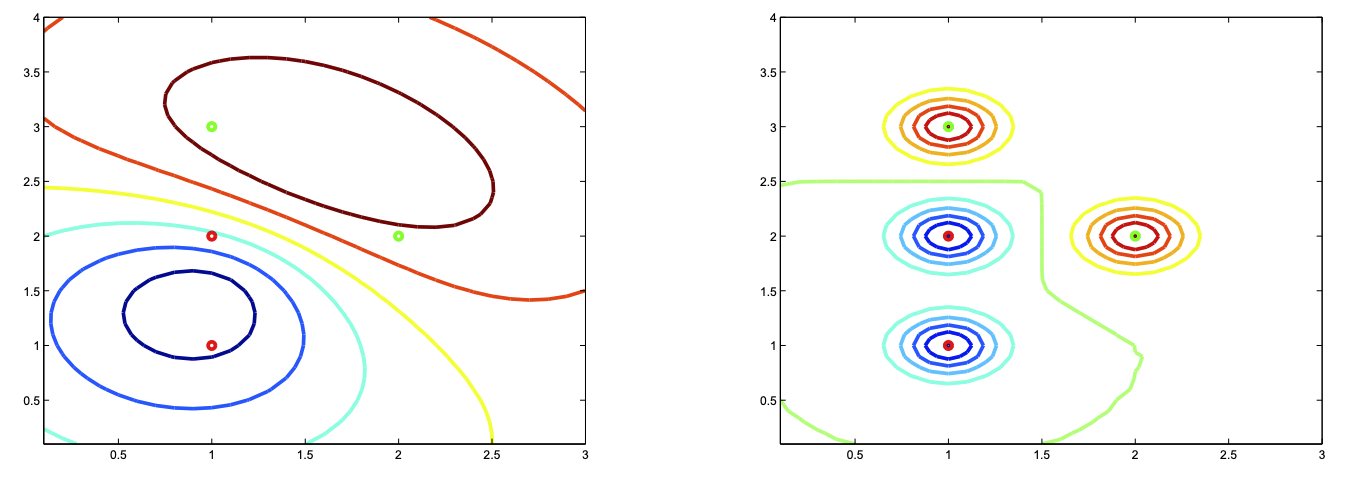}
    \caption{Points in red belong to $\mathcal{S}_1$, points in green belong to $\mathcal{S}_2$. Level curves for the decision function $f =  d_K(x, \mathcal{S}_1) - d_K(x, \mathcal{S}_2$) are plotted. \textit{Left}: $\sigma = 1$. \textit{Right}: $\sigma = 0.2$ (from Jean-Philippe Vert and Julien Mairal course on Kernel Methods, 2021).}
    \label{fig:discrimination}
\end{figure}

\paragraph{Kernel methods  for large scale machine learning.} Kernel methods enable to learn non-linear decision functions $f$ using tools from linear optimization. But, reformulating an ERM problem including kernel comparisons for all pairs of input samples may be intractable in the context of today's large datasets, with a number of sample typically greater than $100,000$: such setting requires to compute a kernel matrix which is quadratic in the number of samples. However, in many cases, it is possible to circumvent this issue by computing a low-rank approximation of the kernel matrix which is done in Nyström's approximation described below, or using random Fourier features~\citep{rahimi2007random}. In the latter setting, a randomized feature map $\Psi$ is used, such that:
\begin{equation}
    K(x, x') \approx \langle \Psi(x), \Psi(x') \rangle_{\Real^d}.
\end{equation}
For example, \citet{rahimi2007random} show that it is possible to approximate a shift invariant kernel, \textit{i.e} a kernel such that $K(x, x')$ is a function of $x - x'$, by drawing samples from the Fourier transform of $K$ and applying a cosine transformation on these samples using $x$. The resulting mapping $\Psi(x)$ is the concatenation of these transforms for each sample:  
\begin{equation}
    \Psi(x) = \sqrt{\frac{2}{D}} \left( \text{cos} (\omega_1 x + b_1), \dots, \text{cos} (\omega_D x + b_D) \right),
\end{equation}
with $\omega_i$ a sample drawn from the Fourier transform $p$ of $K$, and $b_i$ uniformly sampled from $[\ 0, 2 \pi ]\ $.

\paragraph{Kernel methods for deep learning.} Deep learning made kernel methods seemingly obsolete, as it allows to learn more expressive decision functions from the data. However, kernel methods remain relevant in the context of deep learning. We now provide two instances of kernel methods used in the context of deep learning, and whose present thesis is in line with.

\begin{itemize}
    \item[-]  The RKHS norm $||f||_{\mathcal{H}}$ acts as a natural regularization term because it controls the variations of the models prediction according to the geometry induced by $\Phi$. Indeed, for all pairs $(x, x')$:
        \begin{equation}
            |f(x) - f(x')| \leq ||f||_{\mathcal{H}} ||\Phi(x) - \Phi(x')||_{\mathcal{H}}.
        \end{equation}
        \citet{bietti2019group} showed that it is possible to construct a multi-layer hierarchical kernel whose RKHS contains CNNs but with smooth activations instead of ReLus. Then, \citet{bietti2019a} used approximations of the corresponding RKHS norm to regularize the training of real CNNs (with ReLus activations hence not contained in the RKHS), using various upper-bounds or lower-bounds of $||f||_{\mathcal{H}}$ as a penalty term in the empirical risk minimization objective, thus improving generalization and robustness of deep convolutional neural networks in the context of small datasets of images. For example, $||f||_{\mathcal{H}}$ can be upper-bounded as:
        \begin{equation}
            ||f||_{\mathcal{H}} \leq \omega(||W_1||, \dots, ||W_L||),
        \end{equation}
        with $\omega$ a function increasing in all its arguments and $||W_i||$ the spectral norm of the i-th weight matrix of the network. Then, it is possible to regularize the training problem by using the sum of the spectral norms of the weight matrices as a penalty term.
    \item[-] It is possible to get an approximation $\Psi$ for the kernel mapping $\Phi$ in the sense that $\langle \Psi(x), \Psi(x') \rangle_{\Real^d} \approx \langle \Phi(x), \Phi(x') \rangle_{\mathcal{H}} $~\citep{williams2001using}. In fact, the resulting approximation will have a neural network-like structure, \textit{i.e} alternating matrix multiplications and non-linearities~\citep{mairal2016end}:
        \begin{equation}
            \Psi(x) = K(Z, Z)^{-\nicefrac{1}{2}} K(Z, x),
        \end{equation}
        where $Z \in \Real^{k \times d}$ is a set of $k$ anchor points. $K(Z, x) \in \Real^k$ denotes the comparison between each anchor and the current query $x$, and $K(Z, Z) \in \Real^{k \times k}$ is the Gram matrix for the anchors. By choosing a suitable kernel $K$ depending on the data at hand, we are therefore encoding a prior in our model. As opposed to neural networks, it is possible to learn $\Psi$ without labelled data. In a supervised context, $\Psi$ can still be learned using back-propagation. This technique naturally lends itself to contexts where few labelled data is available. We will therefore rely on it in Chapters~\ref{chapt:2_otke} and~\ref{chapt:3_graphit}.
\end{itemize}

\subsection{Transformers}

\paragraph{A recently proposed neural network architecture.} Transformers are a kind of deep learning architecture introduced by~\citet{vaswani2017}. This architecture relies mostly on the attention mechanism introduced in the context of machine translation~\citep{Bahdanau2015}. As opposed to common neural network architectures such as fully-connected or convolutional, it takes a set of features such as the words of a sentence or patches composing an image as input, and output a label or another set of features. More precisely, transformers consist in two parts: the encoder and the decoder.

The encoder takes a set of $n$ elements ${X}$ in $\Real^{n \times d_{\text{in}}}$. It outputs another set in $\Real^{n \times d_{\text{out}}}$. Globally, the feature map $X$ is updated via:
\begin{equation*}
    {X} = {X} + \text{Attention}(Q, K, V).
\end{equation*}
The update term is obtained via the self-attention mechanism:
\begin{equation}
    \text{Attention}(Q, K, V) = \text{softmax}\left( \frac{QK^\top}{\sqrt{d_{out}}} \right) V \in \Real^{n \times d_{\text{out}}},
\end{equation}
with $Q^\top = W_Q {X}^\top$  and $K^\top = W_K {X}^\top$  respectively the query and key matrices, $V^\top = W_V {X}^\top$ the value matrix. $W_Q, W_K, W_V$ in $\Real^{d_{\text{out}} \times d_{\text{in}}}$ are learned projection matrices.

Then, Layer Normalization~\citep{ba2016layer} is applied to the output: each element of the output is normalized via statistics computed over all the hidden units in the same layer. Finally, a fully connected feed-forward neural network is applied to each element of the updated feature map $X$. This neural network is the same for every elements in $X$. Again, the feature map is updated with a residual connection. The sequence of operations forms a block which can be repeated multiple times to build a deeper transformer encoder.

The decoder part of the transformer takes two inputs: the output of the encoder and another input sequence (that will be iteratively built when the transformer is used for translation for example). Here, the attention weights are computed via an attention mechanism between the encoder output sequence and the decoder input sequence. We will not further detail the decoder here since we will not make use of it in the thesis.

\paragraph{Position encoding.} The sequence of operations detailed above does not take into account the position of the elements in the input set. It is said that the transformer encoder is ``permutation equivariant'': a permutation of two elements in the input set will simply result in the same permutation in the output set, without further change in the feature map. To circumvent this problem, positional information is added to the input elements when needed. This approach is called position encoding. Position encoding can either be learnable vectors or scalars that are added to the input elements (absolute position encoding) or at the attention mechanism level for a pair of input elements (relative position encoding). For euclidean data such as sentences and images which can be seen respectively as 1D or 2D grids, it is also possible to come up with hand-crafted coordinate systems such as in~\citet{vaswani2017}:
\begin{align}
    & PE(i, 2k) = \text{sin}\left(\frac{i}{1000^{2k / d_{in}}}\right) \\
    & PE(i, 2k + 1) = \text{cos}\left(\frac{i}{1000^{2k / d_{in}}}\right),
\end{align}
where $i$ is the index of the element in the set, and $k$ the index of the dimension. Position encoding in transformers is less trivial when it comes to non-euclidean data structure such as graphs, as will be detailed in Chapter~\ref{chapt:3_graphit}.

\paragraph{Early success of transformers.} Transformers can be used when one wants to deal with sets of features or sequences (by sequence, we mean a set where the order of the elements matter). Many data modalities can in fact be seen as set or sequences: sentences can be seen as a sequence of word embeddings, images can be seen as a sequence of pixels or patches, and graphs can be seen as Transformers met early success in Natural Language Processing, quickly outperforming state-of-the-art models on various benchmarks as mentioned above.

Transformers recently met success in computer vision~\citep{dosovitskiy2021an} or bioinformatics~\citep{rives2019biological}. This architecture is remarkable because it puts the paradigm ``to one data modality corresponds one preferred architecture'' into question~\footnote{\href{https://gregoiremialon.github.io/talk/nlpmeetup/nlpmeetup.pdf}{How Natural Language Processing is reshaping Machine Learning}}. For example, it has been shown in the context of NLP that transformers scale better than their LSTMs counterpart as can be seen in Figure~\ref{fig:power_law}, probably due to their improved use of long contexts~\cite{kaplan2020scaling}. LSTMs were however the de facto architecture when it came to dealing with sequences. Today, transformers are the cornerstone of most language models.

Current hardware is suited to large matrix multiplications required by the transformer architecture, which enables to train bigger and bigger transformers on bigger and bigger datasets of text, or images, or even both either in the traditional setting of supervised learning, or via the more and more popular self-supervised learning. Ours minds should remain open to different models in the future, but the surprising success of this simple architecture is what motivated its study in this thesis in Chapter~\ref{chapt:2_otke} and~\ref{chapt:3_graphit}. Relevant information on the architecture will be given in these chapters.

\begin{figure}
\centering
\includegraphics[scale=.5]{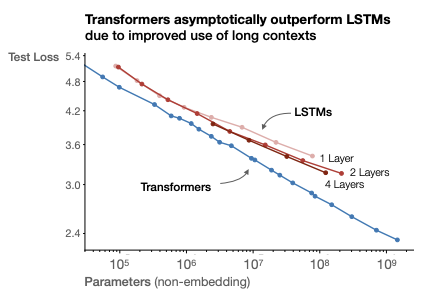}
\caption{Transformers scale better than LSTMs for a given, big dataset. From~\citet{kaplan2020scaling}.}
\label{fig:power_law}
\end{figure}

\section{Outline of the thesis and contributions}

\paragraph{Chapter~\ref{chapt:2_otke}}introduces an embedding inspired from optimal transport to learn on sequences of varying size such as sentences in natural language processing or proteins in bioinformatics, with few annotated data. It can be learned with or without supervision and be used alone or as a global pooling mechanisms within deep architectures. Finally, our embedding has strong links with attention and transformers and can actually be seen as one of the first versions of linearised attention. 

\paragraph{Chapter~\ref{chapt:3_graphit}} builds on the idea of global adaptive pooling introduced in Chapter~\ref{chapt:2_otke} and introduces one of the first transformers for learning on graphs, GraphiT. Transformers have an interesting inductive bias when it comes to graph but cannot take their structure, a crucial piece of information, into account. We overcome this limitation by introducing two mechanisms for incorporating the structure of a graph into the transformers architecture. GraphiT is able to outperform popular Graph Neural Networks, which are the current de facto architecture for learning on graphs, on datasets of various sizes. GraphiT also demonstrates promising visualization capabilities that could be helpful for scientific applications where interpretation is important, for example in chemo-informatics as molecules can be seen as graphs. 

\paragraph{Chapter~\ref{chapt:4_screening}}still seeks to learn with few samples while moving away from deep learning, and focuses instead on models that can be learned by solving a convex training problem. More precisely, we explore the problem of safe sample screening, where one wants to delete as many useless samples as possible from the dataset before training a model, \textit{without modifying the optimal solution}. This is an important task in domains where there is a large amount of trivial observations that are useless for learning. We introduce general rules working for non strongly convex problems as well as a new regularization mechanism to design regression or classification losses allowing for safe screening. 

\newpage

\paragraph{This manuscript is based on the following material:} 

\begin{center}
\begin{tcolorbox}[width=\linewidth, sharp corners=all, colback=white!95!black]
    \begin{itemize}
        \item[-]  A. Bietti*, G. Mialon*, D. Chen, J. Mairal. "A Kernel Perspective for Regularizing Deep Neural Networks" (ICML, 2019). 
        \item[-] G. Mialon, A. d'Aspremont, J. Mairal. "Screening Data Points in Empirical Risk Minimization via Ellipsoidal Regions and Safe Loss Functions" (AISTATS, 2020).
        \item[-] G. Mialon*, D. Chen*, A. d'Aspremont, J. Mairal. "A Trainable Optimal Transport Embedding for Feature Aggregation and its Relationship to Attention" (ICLR, 2021).
        \item[-] G. Mialon*, D. Chen*, M. Selosse*, J. Mairal. "GraphiT: Encoding Graph Structure in Transformers" (preprint arXiv:2106.05667).
    \end{itemize}
\end{tcolorbox}
\end{center}

\cleardoublepage

\chapter{Embedding Sets of Features with Optimal Transport Kernels}
\label{chapt:2_otke}

This chapter tackles the problem of learning on sets of features, motivated by the
need of performing pooling operations in long biological sequences of varying 
sizes, with long-range dependencies, and possibly few labeled data.
To address this challenging task, we introduce a parametrized representation of fixed size, which  embeds and then aggregates elements from a given input set according to the optimal transport plan between the set and a trainable reference.
Our approach scales to large datasets and allows end-to-end training of the 
reference, while also providing a simple unsupervised learning mechanism
with small computational cost. Our aggregation technique admits two 
useful interpretations: it may be seen as a mechanism related to attention layers in neural networks, or it 
may be seen as a scalable surrogate of a classical optimal transport-based
kernel. We experimentally demonstrate the effectiveness of our approach on
biological sequences, achieving state-of-the-art results for protein fold
recognition and detection of chromatin profiles tasks, and, as a proof of concept, we show promising results for
processing natural language sequences. We provide an open-source
implementation of our embedding that can be used alone or as a module in larger
learning models at \url{https://github.com/claying/OTK}.

\paragraph{This chapter is based on the following material:} 

\begin{center}
\begin{tcolorbox}[width=\linewidth, sharp corners=all, colback=white!95!black]
    \begin{itemize}
        \item[-] G. Mialon*, D. Chen*, A. d'Aspremont, J. Mairal. "A Trainable Optimal Transport Embedding for Feature Aggregation and its Relationship to Attention" (ICLR, 2021).
    \end{itemize}
\end{tcolorbox}
\end{center}

\section{Introduction}
\label{section:introduction}
Many scientific fields such as bioinformatics or natural language processing (NLP) require processing sets of features with positional information (biological sequences such as in Figure~\ref{fig:covid}, or sentences represented by a set of local features). These objects are
delicate to manipulate due to varying lengths and potentially long-range dependencies between their elements. For many tasks, the difficulty is even greater since the sets can be arbitrarily large, or only provided with few labels, or both.

\begin{figure}
\centering
\texttt{CUU GAC AAA GUU GAG GCU GAA GUG CAA AUU GAU AGG UUG AUC ACA GGC} \\
\texttt{L \ \ D \ \ K \ \ V \ \ E \ \ A \ \ E \ \ V \ \ Q \ \ I \ \ D \ \ R \ \ L \ \ I \ \ T \ \ G}
\caption{Short part of mRNA sequence for the SARS-Cov-2 spike protein, lying at the surface of the virus. The spike protein is an important part of the infection mechanism and the target of potential treatments. Each triplet codes for an amino acid, represented below.}
\label{fig:covid}
\end{figure}

 
Deep learning architectures specifically designed for sets have recently been proposed~\citep{Lee2019,Skianis2020}.  Our experiments show that these architectures perform well for NLP tasks, but achieve mixed performance for long biological sequences of varying size with few labeled data. Some of these models use attention~\citep{Bahdanau2015}, a classical mechanism for aggregating features. Its typical implementation is the transformer~\citep{vaswani2017}, which has shown to achieve state-of-the-art results for many sequence modeling tasks, \textit{e.g}, in NLP~\citep{devlin2018} or in bioinformatics~\citep{rives2019biological}, when trained with self supervision on large-scale data. Beyond sequence modeling, we are interested in this chapter in finding a good representation for sets of features of potentially diverse sizes, with or without positional information, when the amount of training data may be scarce.
To this end, we introduce a trainable embedding, which can operate directly on the feature set or be combined with existing deep approaches.

More precisely, our embedding marries ideas from optimal transport (OT) theory~\citep{Peyre2019} and kernel methods~\citep{scholkopf2001learning}. We call this embedding OTKE (Optimal Transport Kernel Embedding).
Concretely, we embed feature vectors of a given set to a reproducing kernel Hilbert space (RKHS) and then perform a weighted pooling operation, with weights given by the transport plan between the set and a trainable reference. To gain scalability, we then obtain a finite-dimensional embedding by using kernel approximation techniques~\citep{williams2001using}. 
The motivation for using kernels is to provide a non-linear transformation of the input features before pooling, whereas optimal transport allows to align the features on a trainable reference with fast algorithms~\citep{Cuturi2013b}. Such combination provides us with a theoretically grounded, fixed-size embedding that can be learned either without any label, or with supervision.
Our embedding can indeed become adaptive to the problem at hand, by optimizing the reference with respect to a given task. It can operate on large sets with varying size, model long-range dependencies when positional information is present, and scales gracefully to large datasets. We demonstrate its effectiveness on biological sequence classification tasks, including protein fold recognition and detection of chromatin profiles where we achieve state-of-the-art results. We also show promising results in natural language processing tasks, where our method outperforms strong baselines.
\vs
\paragraph{Contributions.} In summary, our contribution is three-fold. 
\begin{itemize}
    \item[-] We propose a new method to embed sets of features of varying sizes to fixed size representations that are well adapted to downstream machine learning tasks, and whose parameters can be learned in either unsupervised or supervised fashion.
    \item[-] We demonstrate the scalability and effectiveness of our approach on biological and natural language sequences.
    \item[-] We provide an open-source implementation of our embedding that can be used alone or as a module in larger learning models.
\end{itemize}

\section{Related Work}
\label{section:related_work}
\paragraph{Kernel methods for sets and OT-based kernels.} 
The kernel associated with our embedding belongs to the family of match kernels~\citep{Lyu2004, Tolias2013}, which compare all pairs of features between two sets via a similarity function. Another line of research builds kernels by matching features through the Wasserstein distance. A few of them are shown to be positive definite~\citep{Gardner2018} and/or fast to compute~\citep{Rabin2011, Kolouri2016}. Except for few hyper-parameters, these kernels yet cannot be trained end-to-end, as opposed to our embedding that relies on a trainable reference.
Efficient and trainable kernel embeddings for biological sequences have also been proposed by~\citet{Dexiong2019a,Dexiong2019b}. Our work can be seen as an extension of these earlier approaches by using optimal transport rather than mean pooling for aggregating local features, which performs significantly better for long sequences in practice.

\vs
\paragraph{Deep learning for sets.}  
Deep Sets~\citep{Zaheer2017} feed each element of an input set into a feed-forward neural network. The outputs are aggregated following a simple pooling operation before further processing. \citet{Lee2019} propose a Transformer inspired encoder-decoder architecture for sets which also uses latent variables.
\citet{Skianis2020} compute some comparison costs between an input set and reference sets. These costs are then used as features in a subsequent neural network. The reference sets are learned end-to-end. Unlike our approach, such models do not allow unsupervised learning. We will use the last two approaches as baselines in our experiments.
\vs
\paragraph{Interpretations of attention.} 
Using the transport plan as an ad hoc attention score was proposed by~\citet{LChen2019} in the context of network embedding to align data modalities. Our work goes beyond and uses the transport plan as a principle for pooling a set in a model, with trainable parameters. \citet{tsai2019transformer} provide a view of Transformer's attention via kernel methods, yet in a very different fashion where attention is cast as kernel smoothing and not as a kernel embedding.

\section{Proposed Embedding}
\label{section:embedding}
\subsection{Preliminaries}
\label{section:preliminaries}
We handle sets of features in $\Real^d$ and consider sets $\mathbf{x}$ living in
\begin{equation*}
    \mathcal{X} = \left\{ \mathbf{x} ~|~ \mathbf{x} = \{\mathbff{x}_1, \dots, \mathbff{x}_n\} \text{ such that } \mathbff{x}_1, \dots, \mathbff{x}_n \in \Real^d ~~\text{for some}~n \geq 1 \right\}.
\end{equation*}
Elements of $\mathcal{X}$ are typically vector representations of local data structures, such as $k$-mers for sequences, patches for natural images, or words for sentences. The size of $\mathbf{x}$ denoted by $n$ may vary, which is not an issue since the methods we introduce may take a sequence of any size as input, while providing a fixed-size embedding. We now revisit important results on optimal transport and kernel methods, which will be useful to describe our embedding and its computation algorithms. 
\vs
\paragraph{Optimal transport.} 
Our pooling mechanism will be based on the transport plan between $\mathbf{x}$ and~$\mathbf{x'}$ seen as weighted point clouds or discrete measures, which is a by-product of the optimal transport problem~\citep{villani2008,Peyre2019}. OT has indeed been widely used in alignment problems~\citep{Grave2019}. Throughout the chapter, we will refer to the Kantorovich relaxation of OT with entropic regularization, detailed for example in~\citep{Peyre2019}. 
Let $\mathbf{a}$ in $\Delta^n$ (probability simplex) and $\mathbf{b}$ in $\Delta^{n'}$ be the weights of the discrete
measures $\sum_i \mathbf{a}_i \delta_{\mathbf{x}_i}$ and $\sum_j \mathbf{b}_j \delta_{\mathbf{x}'_j}$ with respective
locations $\mathbf{x}$ and $\mathbf{x'}$, where $\delta_{\mathbf{x}}$ is the Dirac at position $\mathbf{x}$. Let
$\mathbf{C}$ in $\Real^{n \times n'}$ be a matrix representing the pairwise costs for aligning the elements of
$\mathbf{x}$ and $\mathbf{x'}$. The entropic regularized Kantorovich relaxation of OT from $\mathbf{x}$ to
$\mathbf{x'}$ is
\begin{equation}
\min_{\mathbf{P} \in U(\mathbf{a}, \mathbf{b}) } \sum_{ij} \mathbf{C}_{ij} \mathbf{P}_{ij} - \varepsilon H(\mathbf{P}), 
\label{eq:ot}
\end{equation}
where $H(\mathbf{P}) = - \sum_{ij} \mathbf{P}_{ij} (\log(\mathbf{P}_{ij}) - 1)$ is the entropic regularization with parameter $\varepsilon$, which controls the sparsity of $\mathbf{P}$, and $U$ is the space of admissible couplings between $\mathbf{a}$ and $\mathbf{b}$:
\begin{equation*}
    U(\mathbf{a}, \mathbf{b}) = \{\mathbf{P} \in \Real_+^{n \times {n'}} : \mathbf{P}\mathbf{1}_n = \mathbf{a} \text{ and } \mathbf{P}^{\top}\mathbf{1}_{n'} = \mathbf{b}\}.
\end{equation*}
The problem is typically solved by using a matrix scaling procedure known as Sinkhorn's algorithm~\citep{sinkhorn1967,Cuturi2013b}.
In practice, $\mathbf{a}$ and $\mathbf{b}$ are uniform measures since we consider the mass to be evenly distributed between the points. $\mathbf{P}$ is called the transport plan, which carries the information on how to distribute the mass of $\mathbf{x}$ in $\mathbf{x'}$ with minimal cost. A simple representation of the problem can be seen in Figure~\ref{fig:ot}. 
Our method uses optimal transport to align features of a given set to a learned reference.

\begin{figure}
\centering
\includegraphics[scale=.5]{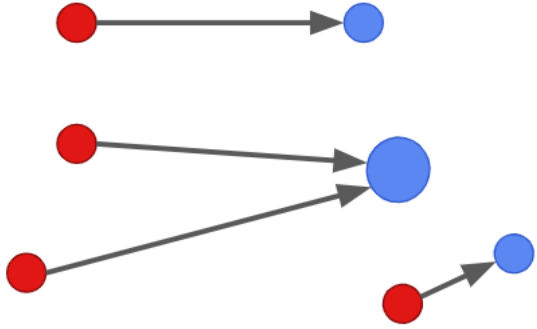}
\caption{Discrete optimal transport consists in finding the transport plan (between points with masses) minimizing a given transportation cost based on the distance between the masses, from the red ones to the blue ones. Here, all masses are identical except the biggest, which is twice the others.}
\label{fig:ot}
\end{figure}

\paragraph{Sinkhorn's Algorithm: Fast Computation of $\mathbf{P}_{\kappa}(\mathbf{x}, \mathbf{\mathbf{z}})$.} Without loss of generality, we consider here $\kappa$ the linear kernel. We recall that $\mathbf{P}_{\kappa}(\mathbf{x}, \mathbf{\mathbf{z}})$ is the solution of an optimal transport problem, which can be efficiently solved by Sinkhorn's algorithm~\citep{Peyre2019} involving matrix multiplications only. Specifically, Sinkhorn's algorithm is an iterative matrix scaling method that takes the opposite of the pairwise similarity matrix $\mathbf{K}$ with entry $\mathbf{K}_{ij}:=\langle \mathbf{x}_i, \mathbf{z}_j\rangle$ as input $\mathbf{C}$ and outputs the optimal transport plan $\mathbf{P}_{\kappa}(\mathbf{x}, \mathbf{\mathbf{z}})=\text{Sinkhorn}(\mathbf{K},\varepsilon)$. Each iteration step $\ell$ performs the following updates
\begin{equation}\label{eq:sinkhorn}
    \mathbf{u}^{(\ell+1)}=\frac{1/n}{\mathbf{E}\mathbf{v}^{(\ell)}}~~\text{and}~~\mathbf{v}^{(\ell+1)}=\frac{1/p}{\mathbf{E}^{\top}\mathbf{u}^{(\ell)}},
\end{equation}
where $\mathbf{E}=e^{\mathbf{K}/\varepsilon}$. Then the matrix $\text{diag}(\mathbf{u}^{(\ell)})\mathbf{E}\text{diag}(\mathbf{v}^{(\ell)})$ converges to $\mathbf{P}_{\kappa}(\mathbf{x}, \mathbf{\mathbf{z}})$ when $\ell$ tends to $\infty$.
However when $\varepsilon$ becomes too small, some of the elements of a matrix product $\mathbf{E} \mathbf{v}$ or $\mathbf{E}^{\top}\mathbf{u}$ become null and result in a division by 0. To overcome this numerical stability issue, computing the multipliers $\mathbf{u}$ and $\mathbf{v}$ is preferred (see \textit{e.g.} \cite[Remark 4.23]{Peyre2019}).
This algorithm can be easily adapted to a batch of data points $\mathbf{x}$, and with possibly varying lengths via a mask vector masking on the padding positions of each data point $\mathbf{x}$, leading to GPU-friendly computation. More importantly, all the operations above at each step are differentiable, which enables $\mathbf{z}$ to be optimized through back-propagation. Consequently, this module can be injected into any deep networks.

\vs
\paragraph{Kernel methods.} Kernel methods~\citep{scholkopf2001learning} map data living in a space  $\mathcal{X}$ to a reproducing
kernel Hilbert space $\mathcal{H}$, associated to a positive definite
kernel $K$ through a mapping function $\varphi: \mathcal{X} \to \mathcal{H}$,
such that $K(\mathbf{x}, \mathbf{x'}) = \langle \varphi(\mathbf{x}), \varphi(\mathbf{x'})
\rangle_{\mathcal{H}}$. Even though $\varphi(\mathbf{x})$ may be infinite-dimensional, classical kernel approximation techniques~\citep{williams2001using} provide finite-dimensional embeddings $\psi(\mathbf{x})$ in $\Real^k$ such that $K(\mathbf{x},\mathbf{x}') \approx \langle \psi(\mathbf{x}), \psi(\mathbf{x}')\rangle$. Our embedding for sets relies in part on kernel method principles and on such a finite-dimensional approximation.

\paragraph{Attention and transformers.} We clarify the concept of attention --- a mechanism yielding a context-dependent embedding for each element of $\mathbf{x}$ --- as a special case of non-local operations~\citep{Wang2017, Buades2011}, so that it is easier to understand its relationship to the OTK. Let us assume we are given a set $\mathbf{x} \in \mathcal{X}$ of length $n$. A non-local operation on $\mathbf{x}$ is a function $\Phi: \mathcal{X} \mapsto \mathcal{X}$ such that
\begin{equation*}
    \Phi(\mathbf{x})_i = \sum_{j=1}^{n} w(\mathbf{x}_i, \mathbf{x}_j) v(\mathbf{x}_j)=\mathbf{W}(\mathbf{x})_i^{\top} \mathbf{V}(\mathbf{x}),
\end{equation*}
where $\mathbf{W}(\mathbf{x})_i$ denotes the $i$-th column of $\mathbf{W}(\mathbf{x})$, a weighting function, and $\mathbf{V}(\mathbf{x})$ corresponds to $[v(\mathbf{x}_1), \dots, v(\mathbf{x}_n)]^\top$, an embedding. In contrast to operations on local neighborhood such as convolutions, non-local operations theoretically account for long range dependencies between elements in the set. In attention and the context of neural networks, $w$ is a \textit{learned} function reflecting the \textit{relevance} of each other elements $\mathbf{x}_j$ with respect to the element $\mathbf{x}_i$ being embedded and given the task at hand. In the context of the chapter, we compare to a type of attention coined as \textit{dot-product self-attention}, which can typically be found in the encoder part of the transformer architecture ~\citep{vaswani2017}. Transformers are neural network models relying mostly on a succession of an attention layer followed by a fully-connected layer. Transformers can be used in sequence-to-sequence tasks --- in this setting, they have an encoder with self-attention and a decoder part with a variant of self-attention ---, or in sequence to label tasks, with only the encoder part. The chapter deals with the latter. The name self-attention means that the attention is computed using a dot-product of linear transformations of $\mathbf{x}_i$ and $\mathbf{x}_j$, and $\mathbf{x}$ attends to itself only. In its matrix formulation, dot-product self-attention is a non-local operation whose matching vector is
\begin{equation*}
    \mathbf{W}(\mathbf{x})_i = \text{Softmax}\left( \frac{W_Q \mathbf{x}_i \mathbf{x}^\top W_K^\top}{\sqrt{d_k}} \right),
    \label{eq:attention}
\end{equation*}
where $W_Q \in \Real^{n \times d_k}$ and $W_K \in \Real^{n \times d_k}$ are learned by the network.
In order to know which $\mathbf{x}_j$ are relevant to $\mathbf{x}_i$, the network computes scores between a query for $\mathbf{x}_i$ ($W_Q \mathbf{x}_i$) and keys of all the elements of $\mathbf{x}$ ($W_K \mathbf{x}$). The softmax turns the scores into a weight vector in the simplex. Moreover, a linear mapping $\mathbf{V}(\mathbf{x}) = W_V \mathbf{x}$, the values, is also learned.  $W_Q$ and $W_K$ are often shared \citep{kitaev2020reformer}. A drawback of such attention is that for a sequence of length $n$, the model has to store an attention matrix $\mathbf{W}$ with size $O(n^2)$. More details can be found in~\cite{vaswani2017}. 

\subsection{Optimal Transport Embedding and Associated Kernel}

We now present the OTKE, an embedding and pooling layer which aggregates a variable-size set or sequence of features into a fixed-size embedding. We start with an infinite-dimensional variant living in a RKHS, before introducing the finite-dimensional embedding that we use in practice.
\vs
\paragraph{Infinite-dimensional embedding in RKHS.}
Given a set $\xb$ and a (learned) reference $\zb$ in $\Xcal$ with~$p$ elements, we consider an embedding $\Phi_\zb(\xb)$ which performs the following operations: (i) initial embedding of the elements of~$\xb$ and $\zb$ to a RKHS $\Hcal$; (ii) alignment of the elements of~$\xb$ to the elements of $\zb$ via optimal transport; (iii) weighted linear pooling of the elements $\xb$ into $p$ bins, producing an embedding $\Phi_\zb(\xb)$ in~$\Hcal^p$, which is illustrated in Figure~\ref{fig:otk}.

Before introducing more formal details, we note that our embedding relies on two main ideas:
\begin{itemize}
    \item[-] \emph{Global similarity-based pooling using references.} Learning on large sets with long-range interactions may benefit from pooling to reduce the number of feature vectors. Our pooling rule follows an inductive bias akin to that of self-attention: elements that are relevant to each other for the task at hand should be pooled together. 
To this end, each element in the reference set corresponds to a pooling cell, where the elements of the input set are aggregated through a weighted sum. The weights simply reflect the similarity between the vectors of the input set and the current vector in the reference. Importantly, using a reference set enables to reduce the size of the ``attention matrix'' from quadratic to linear in the length of the input sequence.
\item[-] \emph{Computing similarity weights via optimal transport.} A computationally efficient similarity score between two elements is their dot-product~\citep{vaswani2017}. In this chapter, we rather consider that elements of the input set should be pooled together if they align well with the same part of the reference. Alignment scores can efficiently be obtained by computing the transport plan between the input and the reference sets: Sinkhorn's algorithm indeed enjoys fast solvers~\citep{Cuturi2013b}.
\end{itemize}

We are now in shape to give a formal definition.
\begin{definition}[\bfseries The optimal transport kernel embedding] Let $\mathbf{x} = (\mathbf{x}_1, \dots, \mathbf{x}_n)$ in $\Xcal$ be an input set of feature vectors and $\mathbf{z} = (\zb_1, \ldots, \zb_p)$ in $\Xcal$ be a reference set with $p$ elements. Let $\kappa$ be a positive definite kernel, \emph{e.g.}, Gaussian kernel, with RKHS $\mathcal{H}$ and  $\varphi: \Real^d \to \mathcal{H}$, its associated kernel embedding. Let $\kappab$ be the $n \times p$ matrix which carries the comparisons $\kappa(\xb_i,\zb_j)$, before alignment.

Then, the transport plan between $\xb$ and $\zb$, denoted by the $n \times p$ matrix $\Pb(\xb,\zb)$, is defined as the unique solution of~(\ref{eq:ot}) when choosing the cost  $\mathbf{C} = -\kappab$, and our embedding is defined as
    \begin{equation*}
    \Phi_{\mathbf{z}}(\mathbf{x}) := \sqrt{p} \times \left( \sum_{i =1}^n \mathbf{P}(\mathbf{x}, \mathbf{z})_{i1} \varphi(\mathbf{x}_i), ~\dots~, \sum_{i=1}^n \mathbf{P}(\mathbf{x}, \mathbf{\mathbf{z}})_{ip}  \varphi(\mathbf{x}_i) \right) = \sqrt{p} \times \mathbf{P}(\mathbf{x}, \mathbf{\mathbf{z}})^\top \varphi(\mathbf{x}),
    \end{equation*}
    where $\varphi(\mathbf{x}) := [\varphi(\mathbf{x}_1), \dots, \varphi(\mathbf{x}_n)]^\top$.
    \label{eq:ot_emb}
\end{definition}
    Interestingly, it is easy to show that our embedding
     $\Phi_{\mathbf{z}}(\mathbf{x})$ is associated to the positive definite kernel  
    \begin{equation}
    K_{\mathbf{z}}(\mathbf{x}, \mathbf{x'}) :=  \sum_{i,i'=1}^n \Wb_{\zb}(\xb,\xb')_{ii'} \kappa(\xb_i, \xb'_{i'})
= \langle \Phi_{\mathbf{z}}(\mathbf{x}), \Phi_{\mathbf{z}}(\mathbf{x'}) \rangle,
    \label{eq:ot_kernel_approx}
    \end{equation}
with $\mathbf{P}_{\mathbf{z}}(\mathbf{x}, \mathbf{x'}) := p \times \mathbf{P}(\mathbf{x}, \mathbf{z}) \mathbf{P}(\mathbf{x'}, \mathbf{z})^\top$. This is a weighted match kernel, with weights given by
optimal transport in $\Hcal$.
The notion of pooling in the RKHS $\mathcal{H}$ of $\kappa$ arises naturally if $p \leq n$. The elements of $\mathbf{x}$ are non-linearly embedded and then aggregated in ``buckets'', one for each element in the reference $\mathbf{z}$, given the values of $\mathbf{P}(\mathbf{x}, \mathbf{z})$. This process is illustrated in Figure~\ref{fig:otk}. We acknowledge here the concurrent work by~\citet{kolouri2021wasserstein}, where a similar embedding is used for graph representation. We now expose the benefits of this kernel formulation, and its relation to classical non-positive definite kernel.

\begin{figure}
\centering
\begin{minipage}{0.5\linewidth}
\includegraphics[scale=0.7]{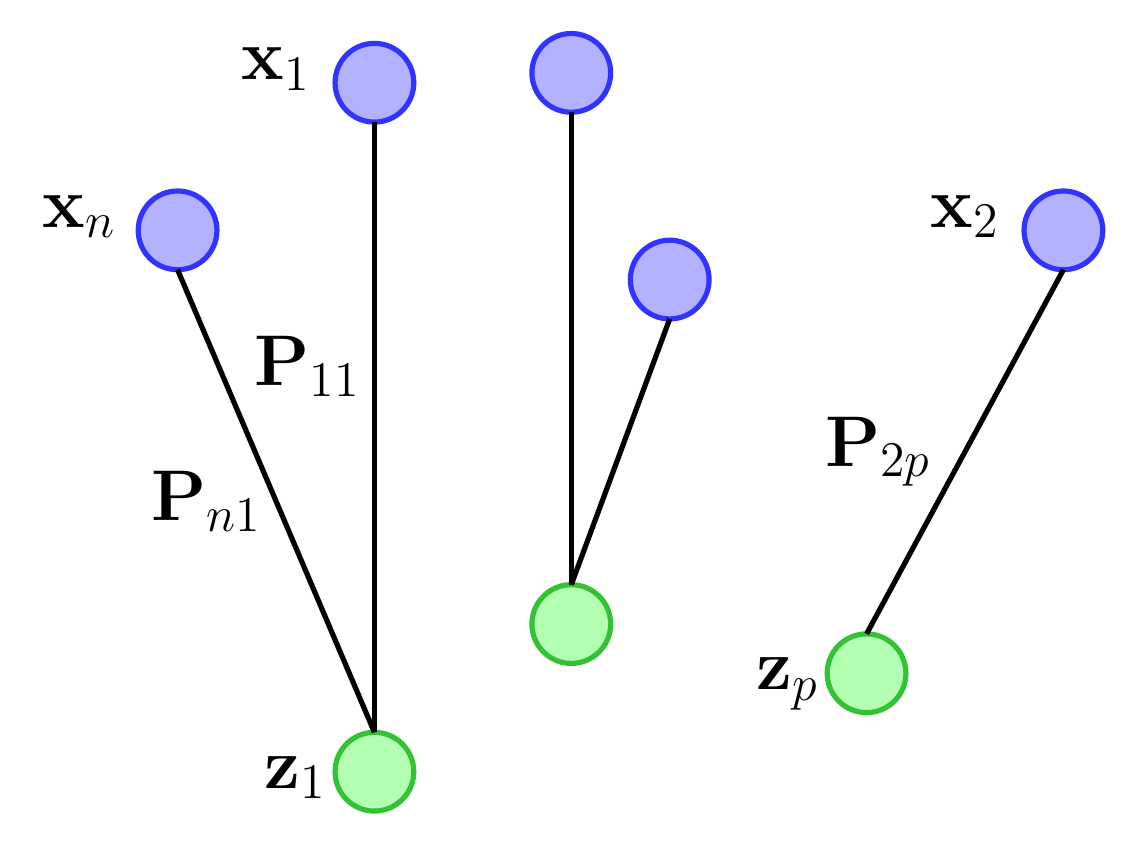}
\end{minipage} 
\begin{minipage}{0.5\linewidth}
\vspace{2cm}
\centering
\includegraphics[scale=0.7]{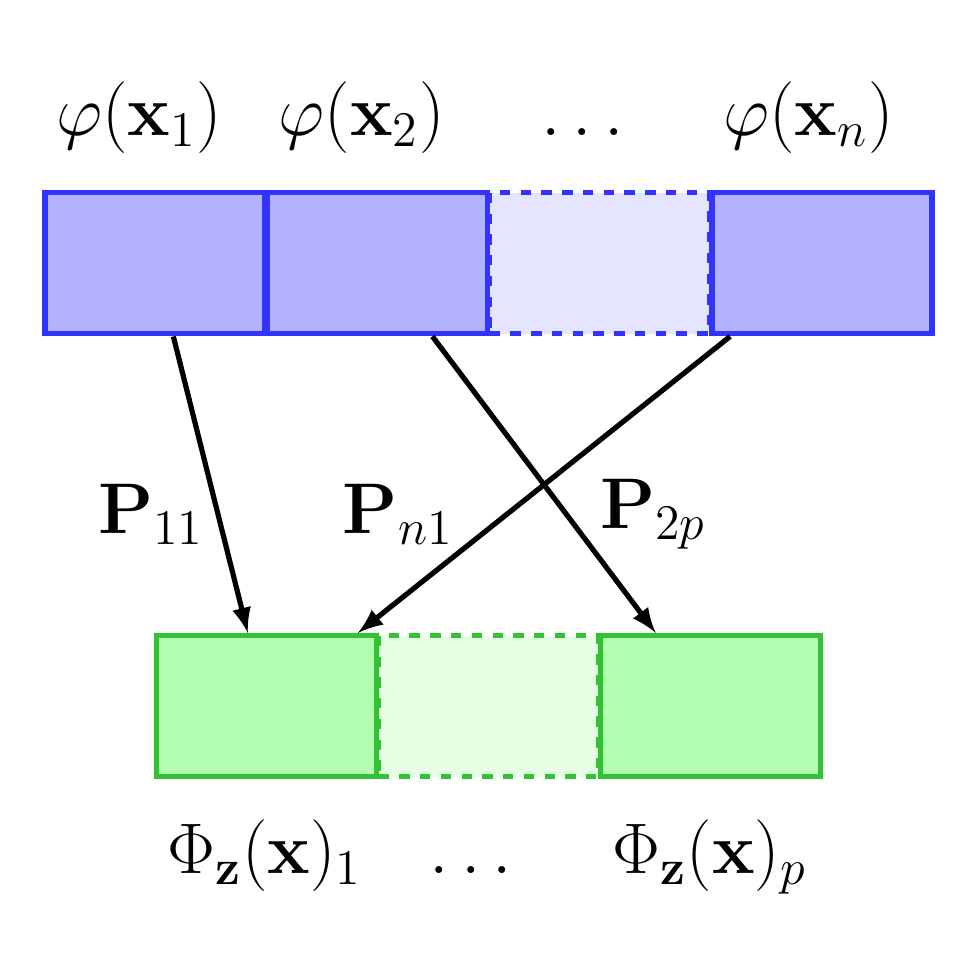}
\end{minipage}
\caption{The input point cloud $\mathbf{x}$ is transported onto the reference $\mathbf{z} = (\mathbf{z}_1, \dots, \mathbf{z}_p)$ (top), yielding the optimal transport plan $\mathbf{P}_{\kappa}(\mathbf{x}, \mathbf{z})$ used to aggregate the embedded features and form $\Phi_{\mathbf{z}}(\mathbf{x})$ (bottom).}\label{fig:otk}
\end{figure}

\paragraph{Kernel interpretation.}
Thanks to the gluing lemma \citep[see, \emph{e.g.},][]{Peyre2019}, $\mathbf{P}_{\mathbf{z}}(\mathbf{x}, \mathbf{x'})$ is a valid transport plan and, empirically, a rough approximation of $\mathbf{P}(\mathbf{x}, \mathbf{x'})$. $K_{\mathbf{z}}$ can therefore be seen as a surrogate of a well-known kernel~\citep{Rubner2000}, defined as
\begin{equation}
\label{eq:ot_kernel}
    K_{\text{OT}}(\mathbf{x}, \mathbf{x'}) := \sum_{i,i'=1}^n {\mathbf{P} (\mathbf{x}, \mathbf{x'})}_{ii'} \kappa(\mathbf{x}_i, \mathbf{x}'_{i'}).
\end{equation}
When the entropic regularization $\varepsilon$ is equal to $0$, $K_{\text{OT}}$ is equivalent to the 2-Wasserstein distance $W_2(\mathbf{x},\mathbf{x'})$ with ground metric $d_{\kappa}$ induced by kernel~$\kappa$. $K_{\text{OT}}$ is generally not positive definite (see~\citet{Peyre2019}, Chapter~$8.3$) and computationally costly (the number of transport plans to compute is quadratic in the number of sets to process whereas it is linear for $K_{\mathbf{z}}$). Now, we show the relationship between this kernel and our kernel $K_\zb$, which is proved in Appendix~\ref{subsection:ot_approx}.
\begin{lemma}[Relation between $\mathbf{P}(\mathbf{x}, \mathbf{x'})$ and $\mathbf{P}_{\mathbf{z}}(\mathbf{x}, \mathbf{x'})$ when $\varepsilon=0$]
\label{lemma:kot_approx}
    For any $\mathbf{x}$, $\mathbf{x'}$ and $\mathbf{z}$ in $\mathcal{X}$ with lengths $n$, $n'$ and $p$, by denoting $W_{2}^{\mathbf{z}}(\mathbf{x}, \mathbf{x'}):=\langle \mathbf{P}_{\mathbf{z}}(\mathbf{x}, \mathbf{x'}), d_{\kappa}^2(\mathbf{x}, \mathbf{x'}) \rangle^{\nicefrac{1}{2}}$ we have
    \begin{equation}
        |W_2(\mathbf{x},\mathbf{x'}) -W_{2}^{\mathbf{z}}(\mathbf{x}, \mathbf{x'})|
        \leq 2 \min (W_2(\mathbf{x}, \mathbf{z}), W_2(\mathbf{x'}, \mathbf{z})).\label{eq:approx}
    \end{equation}
\end{lemma}
This lemma shows that the distance $W_2^{\mathbf z}$ resulting from $K_{\mathbf{z}}$ is related to the Wasserstein distance~$W_2$; yet, this relation should not be interpreted as an approximation error as our goal is not to approximate~$W_2$, but rather to derive a trainable embedding $\Phi_{\zb}(\xb)$ with good computational properties. Lemma~\ref{lemma:kot_approx} roots our features and to some extent self-attention in a rich optimal transport literature. In fact, $W_2^{\mathbf z}$ is equivalent to a distance introduced by~\citet{Wang2013}, whose properties are further studied by~\citet{moosmuller2020linear}. A major difference is that $W_2^{\mathbf z}$ crucially relies on Sinkhorn's algorithm so that the references can be learned end-to-end, as explained below.

\subsection{From infinite-dimensional kernel embedding to finite dimension}
\label{subsubsection:nyst}
In some cases, $\varphi(\mathbf{x})$ is already finite-dimensional, which 
allows to compute the embedding $\Phi_{\mathbf z}(\mathbf x)$ explicitly.
This is particularly useful when dealing with large-scale data, as it 
enables us to use our method for supervised learning tasks without computing the
Gram matrix, which grows quadratically in size with the number of samples. When $\varphi$ is infinite or high-dimensional, it is nevertheless possible to use an approximation based on the Nystr\"om method~\citep{williams2001using}, which provides an embedding $\psi: {\mathbb R^d} \to \Real^k$ such that
\begin{displaymath}
    \langle \psi(\mathbf x_i), \psi(\mathbf x_j')  \rangle_{\Real^k} \approx \kappa( {\mathbf x}_i, {\mathbf x}_j').
\end{displaymath}
Concretely, the Nystr\"om method consists in projecting points from the RKHS~$\mathcal H$ onto a linear subspace~$\mathcal
F$, which is parametrized by $k$ anchor points ${\mathcal F} = \text{Span}( \varphi( \mathbf{w}_1 ), \ldots, \varphi(
\mathbf{w}_k))$. The corresponding embedding admits an explicit form $\psi({\mathbf x}_i) = \kappa({\mathbf w},{\mathbf
w})^{-1/2} \kappa({\mathbf w},{\mathbf x}_i)$, where $\kappa({\mathbf w},{\mathbf w})$ is the $k \times k$ Gram matrix of
$\kappa$ computed on the set $\wb=\{\wb_1, \ldots, \wb_k\}$ of anchor points and $\kappa(\wb,{\mathbf x}_i)$ is
in~$\Real^k$. Then, there are several ways to learn the anchor points: (a) they can be chosen as random points from data; (b) they can be defined as centroids obtained by K-means, see~\citet{Zhang2008}; (c) they can be learned by back-propagation for a
supervised task, see~\citet{mairal2016end}. 

Using such an approximation within our framework can be simply achieved by
(i) replacing $\kappa$ by a linear kernel and (ii) replacing each element $\xb_i$ by its embedding $\psi(\xb_i)$ in $\Real^k$ and considering a reference set with elements in $\Real^k$.
By abuse of notation, we still use $\mathbf{z}$ for the new parametrization. The embedding, which we use in practice in all our experiments, becomes simply
\begin{align}
    \Phi_{\mathbf{z}}(\mathbf{x}) & = \sqrt{p} \times \left( \sum_{i =1}^n \mathbf{P}(\psi(\mathbf{x}), \mathbf{z})_{i1} \psi(\mathbf{x}_i), ~\dots~, \sum_{i=1}^n \mathbf{P}(\psi(\mathbf{x}), \mathbf{\mathbf{z}})_{ip}  \psi(\mathbf{x}_i) \right) \nonumber \\
    & = \sqrt{p} \times \mathbf{P}(\psi(\mathbf{x}), \mathbf{\mathbf{z}})^\top \psi(\mathbf{x}) \in \Real^{p\times k},
\label{eq:otk_in_practice}
\end{align}
where $p$ is the number of elements in $\mathbf{z}$.
Next, we discuss how to learn the reference set~$\mathbf z$.

\subsection{Unsupervised and Supervised Learning of Parameter \texorpdfstring{$\mathbf{z}$}.}
\label{subsection:learning}
\paragraph{Unsupervised learning.} 
In the fashion of the Nyström approximation, the $p$ elements of $\mathbf{z}$ can be defined as the centroids obtained by K-means applied to all features from training sets in $\mathcal X$. 
A corollary of Lemma~\ref{lemma:kot_approx} suggests another algorithm: a bound on the deviation term between $W_2$ and $W_2^{\mathbf{z}}$ for $m$ samples ($\mathbf{x}^1, \dots, \mathbf{x}^m$) is indeed
\begin{equation}
\label{eq:approx_error_one_ref}
    \mathcal{E}^2 
    := \frac{1}{m^2} \sum_{i,j=1}^m |W_2(\mathbf{x}^i,\mathbf{x}^j) - W_{2}^{\mathbf{z}}(\mathbf{x}^i, \mathbf{x}^j)|^2
    \leq \frac{4}{m} \sum_{i=1}^m W_2^2(\mathbf{x}^i, \mathbf{z}).
\end{equation}
The right-hand term corresponds to the objective of the Wasserstein barycenter problem~\citep{Cuturi2013a}, which yields the mean of a set of empirical measures (here the $\mathbf{x}$'s) under the OT metric. The Wasserstein barycenter is therefore an attractive candidate for choosing $\mathbf{z}$. K-means can be seen as a particular case of Wasserstein barycenter when $m=1$~\citep{Cuturi2013a,ho2017multilevel} and is faster to compute. In practice, both methods yield similar results, see Appendix~\ref{section:add_experiments}, and we thus chose K-means to learn $\mathbf{z}$ in unsupervised settings throughout the experiments.
The anchor points $\mathbf{w}$ and the references $\mathbf{z}$ may be then computed using similar algorithms; however, their mathematical interpretation differs as exposed above. The task of representing features (learning $\mathbf{w}$ in $\mathbb{R}^d$ for a specific $\kappa$) is decoupled from the task of aggregating (learning the reference $\mathbf z$ in $\mathbb{R}^k$).
\vs
\paragraph{Supervised learning.} 
As mentioned in Section~\ref{section:preliminaries}, $\mathbf{P}(\psi(\mathbf{x}), \mathbf{z})$ is computed using Sinkhorn's algorithm, recalled in introduction, which can be easily adapted to batches of samples $\mathbf{x}$, with possibly varying lengths, leading to GPU-friendly forward computations of the embedding~$\Phi_{\mathbf{z}}$. More important, all Sinkhorn's operations are differentiable, which enables $\mathbf{z}$ to be optimized with stochastic gradient descent through back-propagation~\citep{genevay2018}, \emph{e.g.}, for minimizing a classification or regression loss function when labels are available. In practice, a small number of Sinkhorn iterations (\emph{e.g.}, 10) is sufficient to compute $\mathbf{P}(\psi(\mathbf{x}), \mathbf{z})$. Since the anchors $\mathbf{w}$ in the embedding layer below can also be learned end-to-end~\citep{mairal2016end}, our embedding can thus be used as a module injected into any model, \textit{e.g}, a deep network, as demonstrated in our experiments.

\subsection{Extensions}

\paragraph{Integrating positional information into the embedding.}

The discussed embedding and kernel do not take the position of the features into account, which may be problematic when dealing with structured data such as images or sentences. To this end, we borrow the idea of convolutional kernel networks, or CKN~\citep{mairal2016end,mairal2014convolutional}, and penalize the similarities exponentially with the positional distance between a pair of elements in the input and reference sequences. More precisely, we multiply $\mathbf{P}(\psi(\mathbf{x}), \mathbf{z})$ element-wise by a distance matrix $\mathbf{S}$ defined as:
\begin{equation*}
    \mathbf{S}_{ij}=e^{-\frac{1}{\sigma_{\text{pos}}^2}(\nicefrac{i}{n} - \nicefrac{j}{p})^2},
\end{equation*}
and replace it in the embedding. With similarity weights based \textit{both} on content and position, the kernel associated to our embedding can be viewed as a generalization of the CKNs (whose similarity weights are based on position only), with feature alignment based on optimal transport. When dealing with multi-dimensional objects such as images, we just replace the index scalar $i$ with an index vector of the same spatial dimension as the object, representing the positions of each dimension.
\vs
\paragraph{Using multiple references.} A naive reconstruction using different references $\mathbf{z}^1, \dots, \mathbf{z}^q$ in $\mathcal{X}$ may yield a better approximation of the transport plan. In this case, the embedding of $\mathbf{x}$ becomes
\begin{equation}
    \Phi_{\mathbf{z}^1, \dots, {\mathbf{z}}^q}(\mathbf{x})
    = \nicefrac{1}{\sqrt{q}}\left( \Phi_{\mathbf{z}^1}(\mathbf{x}), \dots, \Phi_{\mathbf{z}^q}(\mathbf{x}) \right),\label{eq:otkm}
\end{equation}
with $q$ the number of references (the factor $\nicefrac{1}{\sqrt{q}}$ comes from the mean).
Using~Eq.~(\ref{eq:approx}), we can obtain a bound similar to~(\ref{eq:approx_error_one_ref}) for a data set of $m$ samples ($\mathbf{x}^1, \dots, \mathbf{x}^m$) and $q$ references (see Appendix~\ref{subsection:multi_ref} for details). To choose multiple references, we tried a K-means algorithm with 2-Wasserstein distance for assigning clusters, and we updated the centroids as in the single-reference case. Using multiple references appears to be useful when optimizing $\mathbf{z}$ with supervision (see Appendix~\ref{section:add_experiments}).

\section{Relation between our Embedding and Self-Attention}
\label{section:relationship}
Our embedding and a single layer of transformer encoder, recalled in introduction, share the same type of inductive bias, \textit{i.e}, aggregating features relying on similarity weights. We now clarify their relationship. Our embedding is arguably simpler (see respectively size of attention and number of parameters in Table~\ref{fig:relationship}), and may compete in some settings with the transformer self-attention as illustrated in Section~\ref{section:experiments}.

\begin{table}
    \caption{Relationship between $\Phi_{\mathbf{z}}$ and transformer self-attention. $k$: a function describing how the transformer integrates positional information; $n$: sequence length; $q$: number of references or attention heads; $d$: dimension of the embeddings; $p$: number of supports in $\mathbf{z}$. Typically, $p \ll d$. In recent transformer architectures, positional encoding requires learning additional parameters ($\sim qd^2$).}
    \centering
    \begin{tabular}{c|cc} \toprule
    {} & {Self-Attention} & {$\Phi_{\mathbf{z}}$}  \\ \midrule
    {Attention score} & {$\mathbf{W}=W^\top Q$} & {$\mathbf{P}$} \\
    {Size of score} & {$O(n^2)$} & {$O(np)$} \\ 
    {Alignment w.r.t:} & {$\mathbf{x}$ itself} & {$\mathbf{z}$} \\ 
    {Learned + Shared} & {$W$ and $Q$} & {$\mathbf{z}$} \\
    {Nonlinear mapping} & {Feed-forward} & {$\varphi$ or $\psi$} \\
    {Position encoding} & {$k(t_i, {t'}_j)$} & {$e^{-\frac{1}{ \sigma_{\text{pos}}^2} (\frac{i}{n} - \frac{j}{n'})^2}$} \\
    {Nb. parameters} & {$\sim qd^2$} & {$qpd$} \\
    {Supervision} & {Needed} & {Not needed} \\
    \bottomrule
    \end{tabular}
    \label{fig:relationship}
\end{table}
\vs
\paragraph{Shared reference versus self-attention.} 

There is a correspondence between the values, attention matrix in the transformer and $\varphi$, $\mathbf{P}$ in Definition~\ref{eq:ot_emb}, yet also noticeable differences. On the one hand, $\Phi_{\mathbf{z}}$ aligns a given sequence $\mathbf{x}$ with respect to a reference $\mathbf{z}$, learned with or without supervision, and shared across the data set. Our weights are computed using optimal transport. On the other hand, a transformer encoder performs self-alignment: for a given $\mathbf{x}_i$, features are aggregated depending on a similarity score between $\mathbf{x}_i$ and the elements of $\mathbf{x}$ only. The similarity score is a matrix product between queries $Q$ and keys $K$ matrices, learned with supervision and shared across the data set. 
In this regard, our work complements a recent line of research questioning the dot-product, learned self-attention~\citep{Raganato2020, You2020}. Self-attention-like weights can also be obtained with our embedding by computing $ \mathbf{P}(\mathbf{x},\mathbf{z}_i)\mathbf{P}(\mathbf{x},\mathbf{z}_i)^{\top}$ for each reference $i$. In that sense, our work is related to recent research on efficient self-attention~\citep{wang2020, choro2020}, where a low-rank approximation of the self-attention matrix is computed.
\vs
\paragraph{Position smoothing and relative positional encoding.}
Transformers can add an absolute positional encoding to the input features~\citep{vaswani2017}. Yet, relative positional encoding~\citep{Dai2019} is a current standard for integrating positional information: the position offset between the query element and a given key can be injected in the attention score~\citep{tsai2019transformer}, which is equivalent to our approach. The link between CKNs and our kernel, provided by this positional encoding, stands in line with recent works casting attention and convolution into a unified framework~\citep{Andreoli2020}. In particular, \cite{cordonnier2020on} show that attention learns convolution in the setting of image classification: the kernel pattern is learned at the same time as the filters.
\vs
\paragraph{Multiple references and attention heads.} 
In the transformer architecture, the succession of blocks composed of an attention layer followed by a fully-connected layer is called a head, with each head potentially focusing on different parts of the input. Successful architectures have a few heads in parallel. The outputs of the heads are then aggregated to output a final embedding. A layer of our embedding with non-linear kernel $\kappa$ can be seen as such a block, with the references playing the role of the heads. As some recent works question the role of attention heads~\citep{Voita2019, Michel2019}, exploring the content of our learned references $\mathbf{z}$ may provide another perspective on this question. More generally, visualization and interpretation of the learned references could be of interest for biological sequences.

\paragraph{Efficient transformers.} \citep{mialon2021} was proposed at the same time as a line of work coined as ``efficient transformers". These works sought to improve transformers both in terms of memory and computational cost~\citep{Beltagy2020Longformer,kitaev2020reformer}. We refer the reader to the following survey~\citep{tay2020efficient}. Although our work does not aim at challenging transformers since we are rather interested in data constrained settings, it is interesting to note that OTKE fits in different families of the taxonomy proposed by~\cite{tay2020efficient}: anchors $\mathbf{z}$ can be seen as a memory and $\mathbf{P}(\mathbf{x},\mathbf{z}_i)$ as a factorized approximation of an attention matrix.

\section{Experiments}
\label{section:experiments}
We now show the effectiveness of our embedding OTKE in tasks where samples can be expressed as large sets with potentially few labels, such as in bioinformatics. We evaluate our embedding alone in unsupervised or supervised settings, or within a model in the supervised setting. We also consider NLP tasks involving shorter sequences and relatively more labels. 

\begin{center}
\begin{tcolorbox}[width=\linewidth, sharp corners=all, colback=white!95!black]

    Note that experiments in bioinformatics (SCOP1.75 and DeepSEA) were conducted by Dexiong Chen in the context of~\citet{mialon2021} and already presented in Dexiong Chen's PhD thesis~\citep{chen2020structured}. Hence, they are not a contribution of the present thesis. However, as these supplementary experiments provide insightful results and discussion, we kept it in the main manuscript for completeness.

\end{tcolorbox}
\end{center}

\subsection{Datasets, Experimental Setup and Baselines}
In unsupervised settings, we train a linear classifier with the cross entropy loss between true labels and predictions on top of the features provided by our embedding (where the references $\zb$ and Nyström anchors $\wb$ have been learned without supervision), or an unsupervised baseline. In supervised settings, the same model is initialized with our unsupervised method and then trained end-to-end (including $\zb$ and $\wb$) by minimizing the same loss. We use an alternating optimization strategy to update the parameters for both SCOP and SST datasets, as used by~\citet{Dexiong2019a,Dexiong2019b}. We train for 100 epochs with Adam on both data sets: the initial learning rate is 0.01, and get halved as long as there is no decrease in the validation loss for 5 epochs. The hyper-parameters we tuned include number of supports and references $p,q$, entropic regularization in OT $\varepsilon$, the bandwidth of Gaussian kernels and the regularization parameter of the linear classifier. The best values of $\varepsilon$ and the bandwidth were found stable across tasks, while the regularization parameter needed to be more carefully cross-validated. Additional results and  implementation details can be found in Appendix~\ref{section:add_experiments}.

\paragraph{Experiments on Kernel Matrices (only for small data sets).}

Here, we compare the optimal transport kernel $K_{\text{OT}}$~(\ref{eq:ot_kernel}) and its surrogate $K_{\mathbf{z}}$~(\ref{eq:ot_kernel_approx}) (with $\mathbf{z}$ learned without supervision) to common and other OT kernels. Although our embedding $\Phi_{\mathbf{z}}$ is scalable, the exact kernel require the computation of Gram matrices.
For this toy experiment, we therefore use $5000$ samples only of CIFAR-10 (images with $32 \times 32$ pixels), encoded without supervision using a two-layer convolutional kernel network~\citep{mairal2016end}. The resulting features are $3 \times 3$ patches living in $\Real^d$ with $d=256 \text{ or } 8192$. $K_{\text{OT}}$ and $K_{\mathbf{z}}$ aggregate existing features depending on the ground cost defined by $-\kappa$ (Gaussian kernel) given the computed weight matrix $\mathbf{P}$. In that sense, we can say that these kernels work as an adaptive pooling. We therefore compare it to kernel matrices corresponding to mean pooling and no pooling at all (linear kernel). We also compare to a recent positive definite and fast optimal transport based kernel, the Sliced Wasserstein Kernel~\citep{Kolouri2016} with $10$, $100$ and $1000$ projection directions. We add a positional encoding to it so as to have a fair comparison with our kernels. A linear classifier is trained from this matrices. Although we cannot prove that $K_{\text{OT}}$ is positive definite, the classifier trained on the kernel matrix converges when $\varepsilon$ is not too small. The results can be seen on Table~\ref{tab:kernel_comp}. Without positional information, our kernels do better than Mean pooling. When the positions are encoded, the Linear kernel is also outperformed. Note that including positions in Mean pooling and Linear kernel means interpolating between these two kernels: in the Linear kernel, only patches with same index are compared while in the Mean pooling, all patches are compared. All interpolations did worse than the Linear kernel. The runtimes illustrate the scalability of $K_{\mathbf{z}}$.

\begin{table}
    \small
    \centering
    \caption{Classification accuracies for $5000$ samples of CIFAR-10 using CKN features~\citep{mairal2016end} and forming Gram matrix. A random baseline would yield $10 \%$.}
    \begin{tabular}{l|cc} \toprule
    {Dataset} & \multicolumn{2}{c}{($3 \times 3$), $256$}  \\
    \midrule
    {Kernel} & {Accuracy} & {Runtime} \\
    \midrule
    {Mean Pooling} & {58.5} & {$\sim$ 30 sec} \\
    {Flatten} & {67.6} & {$\sim$ 30 sec} \\
    {Sliced-Wasserstein~\citep{Kolouri2016}} & {63.8} & {$\sim$ 2 min} \\
    {Sliced-Wasserstein~\citep{Kolouri2016} + sin. pos enc.~\citep{devlin2018}} & {66.0} & {$\sim$ 2 min} \\
    {$K_{OT}$} & {64.5} & {$\sim$ 20 min} \\
    {$K_{OT}$ + our pos enc.} & {67.1} & {$\sim$ 20 min} \\
    \midrule
    {$K_{\mathbf{z}}$} & {67.9} & {$\sim$ 30 sec} \\
    {$K_{\mathbf{z}}$ + our pos enc.} & {70.2} & {$\sim$ 30 sec} \\
    \bottomrule
    \end{tabular}
    \label{tab:kernel_comp}
\end{table}

\paragraph{CIFAR-10.}

Here, we test our embedding on the same data modality: we use CIFAR-10 features, \textit{i.e.}, $60,000$ images with $32\times32$ pixels and 10 classes encoded using a two-layer CKN~\citep{mairal2016end}, one of the baseline architectures for unsupervised learning of CIFAR-10, and evaluate on the standard test set. The very best configuration of the CKN yields a small number ($3\times3$) of high-dimensional ($16,384$) patches and an accuracy of $85.8\%$. We will illustrate our embedding on a configuration which performs slightly less but provides more patches ($16\times16$), a setting for which it was designed. 

The input of our embedding are unsupervised features extracted from a 2-layer CKN with kernel sizes equal to 3 and 3, and Gaussian pooling size equal to 2 and 1. We consider the following configurations of the number of filters at each layer, to simulate two different input dimensions for our embedding:
\begin{itemize}
    \item 64 filters at first and 256 at second layer, which yields a $16\times 16\times 256$ representation for each image.
    \item 256 filters at first and 1024 at second layer, which yields a $16\times 16\times 1024$ representation for each image.
\end{itemize}
Since the features are the output of a Gaussian embedding, $\kappa$ in our embedding will be the linear kernel. The embedding is learned with one reference and various supports using K-means method described in Section~\ref{section:embedding}, and compared to several classical pooling baselines, including the original CKN's Gaussian pooling with pooling size equal to 6. The hyper-parameters are the entropic regularization $\varepsilon$ and bandwidth for position encoding $\sigma_{\text{pos}}$. Their search grids are shown in Table~\ref{tab:hyper_cifar10} and the results in Table~\ref{tab:cifar-10}. Without supervision, the adaptive pooling of the CKN features by our embedding notably improves their performance. We notice that the position encoding is very important to this task, which substantially improves the performance of its counterpart without it.

\begin{table}
\small
    \centering
    \caption{Hyperparameter search range for CIFAR-10}
    \label{tab:hyper_cifar10}
    \begin{tabular}{lc}\toprule
         Hyperparameter & Search range \\ \midrule
         Entropic regularization $\varepsilon$ & $[1.0;0.1;0.01;0.001]$ \\
         Position encoding bandwidth $\sigma_{\text{pos}}$ & $[0.5;0.6;0.7;0.8;0.9;1.0]$ \\ \bottomrule
    \end{tabular}
\end{table}

\begin{table}
\small
    \centering
    \caption{Classification results using unsupervised representations for CIFAR-10 for two feature configurations (extracted from a $2$-layer unsupervised CKN with different number of filters). We consider here our embedding with one reference and different number of supports, learned with K-means, with or without position encoding (PE).}
    \label{tab:cifar-10}
    \begin{tabular}{lccc} \toprule
    Method & Nb. supports & $16\times 16\times 256$ & $16\times 16\times 1024$ \\ \midrule
    Flatten & & 73.1 & 76.1  \\
    Mean pooling & & 64.9 & 73.4 \\
    Gaussian pooling~\citep{mairal2016end} & & 77.5 & 82.0 \\ \midrule
    Ours & \multirow{2}{*}{9} & 75.6 & 79.3 \\
    Ours (with PE) & & 78.0 & 82.2 \\ 
    Ours & \multirow{2}{*}{64} & 77.9 & 80.1 \\
    Ours (with PE) & & 81.4 & 83.2 \\  
    Ours & \multirow{2}{*}{144} & 78.4 & 80.7 \\
    Ours (with PE) & & 81.8 & 83.4 \\ 
    \bottomrule
    \end{tabular}
\end{table}

\paragraph{Sentiment analysis on Stanford Sentiment Treebank.}
SST-2~\citep{Socher2013} belongs to the NLP GLUE benchmark~\citep{Wang2019} and consists in predicting whether a movie review is positive. The dataset contains $70,042$ reviews. The test predictions need to be submitted on the GLUE leaderboard, so that we remove a portion of the training set for validation purpose and report accuracies on the actual validation set used as a test set. Our model is one layer of our embedding with $\varphi$ a Gaussian kernel mapping with $64$ Nyström filters in the supervised setting, and a linear mapping in the unsupervised setting. The features used in our model and all baselines are word vectors with dimension $768$ provided by the HuggingFace implementation~\citep{wolf2019} of the pre-trained transformer model BERT~\citep{devlin2018}. State-of-the-art accuracies are usually obtained after supervised fine-tuning of pre-trained transformers. 
Training a linear model on pre-trained features after simple pooling (\textit{e.g}, mean) also yields good results. [CLS], which denotes the BERT embedding used for classification, is also a common baseline. The results are shown in Table~\ref{tab:sst-2}.

\begin{table}
    \small
    \centering
    \caption{Classification accuracies for SST-2 reported on standard validation set, averaged from 10 different runs ($q$ references $\times$ $p$ supports). 
    }
    \begin{tabular}{l|c|c} \toprule
    {Method} & {Unsupervised} & {Supervised} \\
    \midrule
    {[CLS] embedding~\citep{devlin2018}} & {84.6$\pm$0.3} & {90.3$\pm$0.1} \\
    {Mean Pooling of BERT features~\citep{devlin2018}} & {85.3$\pm$0.4} & {\textbf{90.8$\pm$0.1}} \\
    {Approximate Rep the Set~\citep{Skianis2020}} & {Not available.} & {86.8$\pm$0.9} \\
    {Rep the Set~\citep{Skianis2020}} & {Not available.} & {87.1$\pm$0.5} \\
    {Set Transformer~\citep{Lee2019}} & {Not available.} & {87.9$\pm$0.8} \\
    \midrule
    {Ours (dot-product instead of OT)} & {85.7$\pm$0.9} & {86.9$\pm$1.1} \\
    {Ours (Unsupervised: $1 \times 300$. Supervised: $4 \times 30$)} & {\textbf{86.8$\pm$0.3}} & {88.1$\pm$0.8} \\
    \bottomrule
    \end{tabular}
    \label{tab:sst-2}
    \vs
\end{table}

\paragraph{Additional experiments: Protein fold classification on SCOP 1.75.}
We follow the protocol described by~\citet{Hou2017} for this important task in bioinformatics. The dataset contains $19,245$ sequences from $1,195$ different classes of fold (hence less than 20 labels in average per class). The sequence lengths vary from tens to thousands. Each element of a sequence is a $45$-dimensional vector. The objective is to classify the sequences to fold classes, which corresponds to a multiclass classification problem. The features fed to the linear classifier are the output of our embedding with $\varphi$ the Gaussian kernel mapping on $k$-mers (subsequences of length $k$) with $k$ fixed to be $10$, which is known to perform well in this task~\citep{Dexiong2019a}. The number of anchor points for Nystr\"om method is fixed to 1024 and 512 respectively for unsupervised and supervised setting. In the unsupervised setting, we compare our method to state-of-the-art unsupervised method for this task: CKN~\citep{Dexiong2019a}, which performs a global mean pooling in contrast to the global adaptive pooling performed by our embedding. In the supervised setting, we compare the same model to the following supervised models: CKN, Recurrent Kernel Networks (RKN)~\citep{Dexiong2019b}, a CNN with $10$ convolutional layers named DeepSF~\citep{Hou2017}, Rep the Set~\citep{Skianis2020} and Set Transformer~\citep{Lee2019}, using the public implementations by their authors. Rep the Set and Set Transformer are used on the top of a convolutional layer of the same filter size as CKN to extract $k$-mer features. Their model hyper-parameters, weight decay and learning rate are tuned in the same way as for our models (see Appendix for details). The default architecture of Set Transformer did not perform well due to overfitting. We thus used a shallower architecture with one Induced Set Attention Block (ISAB), one Pooling by Multihead Attention (PMA) and one linear layer, similar to the one-layer architectures of CKN and our model. The results are shown in Table~\ref{tab:scop}.

\begin{table}[ht]
    \centering
    \caption{Classification accuracy (top 1/5/10) on test set for SCOP 1.75 for different unsupervised and supervised baselines, averaged from 10 different runs  ($q$ references $\times$ $p$ supports).
    }
    \resizebox{\textwidth}{!}{
    \begin{tabular}{l|c|c} \toprule
    {Method} & {Unsupervised} & {Supervised} \\
    \midrule
    {DeepSF~\citep{Hou2017}} & {Not available.} & {73.0/90.3/94.5} \\
    {CKN~\citep{Dexiong2019a}} & {81.8$\pm$0.8/92.8$\pm$0.2/95.0$\pm$0.2} & {84.1$\pm$0.1/94.3$\pm$0.2/96.4$\pm$0.1} \\
    {RKN~\citep{Dexiong2019b}} & {Not available.} & {85.3$\pm$0.3/95.0$\pm$0.2/96.5$\pm$0.1} \\
    {Set Transformer~\citep{Lee2019}} & {Not available.} & {79.2$\pm$4.6/91.5$\pm$1.4/94.3$\pm$0.6} \\
    {Approximate Rep the Set~\citep{Skianis2020}} & {Not available.} & {84.5$\pm$0.6/94.0$\pm$0.4/95.7$\pm$0.4} \\
    \midrule
    {Ours (dot-product instead of OT)} & {78.2$\pm$1.9/93.1$\pm$0.7/96.0$\pm$0.4} & {87.5$\pm$0.3/95.5$\pm$0.2/96.9$\pm$0.1} \\
    {Ours (Unsup.: $1 \times 100$ / Sup.: $5\times10$)} & {\textbf{85.8$\pm$0.2/95.3$\pm$0.1/96.8$\pm$0.1}} & {\textbf{88.7$\pm$0.3/95.9$\pm$0.2/97.3$\pm$0.1}} \\
    \bottomrule
    \end{tabular}
    }
    \label{tab:scop}
\end{table}
\vs
\paragraph{Additional experiments: Detection of chromatin profiles.}
Predicting the chromatin features such as transcription factor (TF) binding from raw genomic sequences has been studied extensively in recent years. CNNs with max pooling operations have been shown effective for this task. Here, we consider DeepSEA dataset~\citep{zhou2015predicting} consisting in simultaneously predicting 919 chromatin profiles, which can be formulated as a multi-label classification task. DeepSEA contains $4,935,024$ DNA sequences of length 1000 and each of them is associated with $919$ different labels (chromatin profiles). Each sequence is represented as a $1000\times 4$ binary matrix through one-hot encoding and the objective is to predict which profiles a given sequence possesses. As this problem is very imbalanced for each profile, learning an unsupervised model could require an extremely large number of parameters. We thus only consider our supervised embedding as an adaptive pooling layer and inject it into a deep neural network, between one convolutional layer and one fully connected layer, as detailed in Appendix~\ref{subsection:app_deepsea}. In our embedding, $\varphi$ is chosen to be identity and the positional encoding described in Section~\ref{section:embedding} is used.
We compare our model to a state-of-the-art CNN with 3 convolutional layers and two fully-connected layers~\citep{zhou2015predicting}. The results are shown in Table~\ref{tab:deepsea}. 

\begin{table}
\vs
    \small
    \centering
    \caption{Results for prediction of chromatin profiles on the DeepSEA dataset. The metrics are area under ROC (auROC) and area under PR curve (auPRC), averaged over 919 chromatin profiles. Due to the huge size of the dataset, we only provide results based on a single run. 
    }
    \begin{tabular}{l|c|c} \toprule
    {Method} & {auROC} & {auPRC} \\
    \midrule
    {DeepSEA~\citep{zhou2015predicting}} & {0.933} & {0.342} \\
    {Ours with position encoding (Sinusoidal~\citep{vaswani2017}/Ours)} & {0.917/\textbf{0.936}} & {0.311/\textbf{0.360}} \\
    \bottomrule
    \end{tabular}
    \label{tab:deepsea}
\end{table}

\subsection{Results and discussion}
In protein fold classification, our embedding outperforms all baselines in both unsupervised and supervised settings. Surprisingly, our unsupervised model also achieves better results than most supervised baselines. In contrast, Set Transformer does not perform well, possibly because its implementation was not designed for sets with varying sizes, and tasks with few annotations. 
In detection of chromatin profiles, our model (our embedding within a deep network) has fewer layers than state-of-the-art CNNs while outperforming them, which advocates for the use of attention-based models for such applications. Our results also suggest that positional information is important (CIFAR-10 results and Appendix~\ref{subsection:app_deepsea}), and our Gaussian position encoding outperforms the sinusoidal one introduced in~\citet{vaswani2017}. Note that in contrast to a typical transformer, which would have stored a $1000\times1000$ attention matrix, our attention score with a reference of size $64$ is only $1000\times64$, which illustrates the discussion in Section~\ref{section:relationship}. In NLP, an \textit{a priori} less favorable setting since sequences are shorter and there are more data, our supervised embedding gets close to a strong state-of-the-art, \textit{i.e.} a fully-trained transformer. We observed our method to be much faster than RepSet, as fast as Set Transformer, yet slower than ApproxRepSet (\ref{subsection:app_scop}). 
Using the OT plan as similarity score yields better accuracies than the dot-product between the input sets and the references (see Table~\ref{tab:scop}; \ref{tab:sst-2}). 

\paragraph{Choice of parameters.} This paragraph sums up the impact of hyper-parameter choices. Experiments justifying our claims can be found in Appendix~\ref{section:add_experiments}.
\begin{itemize}
    \item[-] Number of references $q$: for biological sequences, a single reference was found to be enough in the unsupervised case, see Table~\ref{tab:app_scop_unsup}. In the supervised setting, Table~\ref{tab:app_scop_sup_more}  suggests that using $q=5$ provides slightly better results but $q=1$ remains a good baseline, and that the sensitivity to number of references is moderate.
    \item[-] Number of supports $p$ in a reference: Table~\ref{tab:app_scop_unsup} and Table~\ref{tab:app_scop_sup_more} suggest that the sensitivity of the model to the number of supports is also moderate.
    \item[-] Nyström anchors: an anchor can be seen as a neuron in a feed-forward neural network (see expression of $\psi$ in~\ref{subsubsection:nyst}). In unsupervised settings, the more anchors, the better the approximation of the kernel matrix. Then, the performance saturates, see Table~\ref{tab:app_scop_unsup_more}. In supervised settings, the optimal number of anchors points is much smaller, as also observed by~\citet{Dexiong2019a}, Fig~6.
    \item[-] Bandwidth $\sigma$ in Gaussian kernel: $\sigma$ was chosen as in~\citet{Dexiong2019b} and we did not try to optimize it in this work, as it seemed to already provide good results. Nevertheless, slightly better results can be obtained when tuning this parameter, for instance in SST-2. 
\end{itemize}

\paragraph{OTKE and self-supervised methods.}
Our approach should not be positioned against self-supervision and instead brings complementary features: the OTKE may be plugged in state-of-the-art models pre-trained on large unannotated corpus. For instance, on SCOP 1.75, we use ESM-1~\citep{rives2019biological}, pre-trained on 250 millions protein sequences, with mean pooling followed by a linear classifier. As we do not have the computational ressources to fine-tune ESM1-t34, we only train a linear layer on top of the extracted features. Using the same model, we replace the mean pooling by our (unsupervised) OTKE layer, and also only train the linear layer. This results in accuracy improvements as showed in Table~\ref{tab:scop_esm}. While training huge self-supervised learning models on large datasets is very effective, ESM1-t34 admits more than 2500 times more parameters than our single-layer OTKE model (260k parameters versus 670M) and our single-layer OTKE outperforms smaller versions of ESM1 (43M parameters). 
Finally, self-supervised pre-training of a deep model including OTKE on large data sets would be interesting for fair comparison.
\begin{table}[h]
\small
    \centering
    \caption{Classification accuracy (top 1/5/10) results of our unsupervised embedding for SCOP 1.75 with pre-trained ESM models~\citep{rives2019biological}.}
    \label{tab:scop_esm}
    \begin{tabular}{lccc}\toprule
        {Model} & {Nb parameters} & {Mean Pooling} & {Unsupervised OTKE} \\ \midrule
          {ESM1-t6-43M-UR50S} & 43M & 84.01/93.17/95.07 & 85.91/93.72/95.30 \\
          {ESM1-t34-670M-UR50S} & 670M & 94.95/97.32/97.91 & 95.22/97.32/98.03 \\
          \bottomrule
    \end{tabular}\label{table:selfsup}
\end{table}
\paragraph{Multi-layer extension.} Extending the OTKE to a multi-layer embedding is a natural yet not straightforward research direction: 
it is not clear how to find a right definition of intermediate feature aggregation in a multi-layer OTKE model. Note that for DeepSEA, our model with single-layer OTKE already outperforms a multi-layer CNN, which suggests that a multi-layer OTKE is not always needed.

\section{Conclusion} We showed the effectiveness of OTKE for pooling in sequences in the context of bioinformatics and natural language processing: our model is more expressive than mean, max or sum and competitive with attention mechanisms in terms of sample complexity and model size. In the next chapter, we will explain how pooling is also an important component of a family of neural networks for graphs (GNNs). This pooling is often limited to the closest neighbors of the graph nodes, which leads to issues when processing graphs with complex interactions between distant nodes. In the philosophy of this chapter, we will see graphs as sequences and aggregate node features across the whole graph, an inductive bias different from that of GNNs. While OTKE used basic position encoding scheme and focused on the pooling mechanism, next chapter will do the opposite, relying on simple attention mechanisms for feature aggregation and introducing new mechanisms, or inductive biases, for encoding the position of the nodes, and more globally the structure of the graph.

\newpage




\vspace*{0.3cm}
\begin{center}
   {\huge \textbf{Appendix}}
\end{center}
\vspace*{0.5cm}
Appendix~\ref{section:add_proofs} contains the proofs skipped in the main part of the chapter; Appendix~\ref{section:add_experiments} provides additional experimental results as well as details on our protocol for reproducibility.


\section{Proofs}
\label{section:add_proofs}
\subsection{Proof of Lemma~\ref{lemma:kot_approx}}
\label{subsection:ot_approx}
\begin{proof}
First, since $\sum_{j=1}^{n'} p{\mathbf{P}(\mathbf{x'}, \mathbf{z}})_{jk} = 1$ for any $k$, we have
\begin{align*}
    W_2(\mathbf{x}, \mathbf{z})^2 & = \sum_{i=1}^n \sum_{k=1}^p {\mathbf{P}(\mathbf{x}, \mathbf{z}})_{ik} d_{\kappa}^2(\mathbf{x}_i, \mathbf{z}_k) \\
    & = \sum_{i=1}^n \sum_{k=1}^p \sum_{j=1}^{n'} p{\mathbf{P}(\mathbf{x'}, \mathbf{z}})_{jk} {\mathbf{P}(\mathbf{x}, \mathbf{z}})_{ik} d_{\kappa}^2(\mathbf{x}_i, \mathbf{z}_k) \\
    &=\| \mathbf{u} \|_2^2,
\end{align*}
with $\mathbf{u}$ a vector in $\mathbb{R}^{nn'p}$ whose entries are $\sqrt{p{\mathbf{P}(\mathbf{x'}, \mathbf{z}})_{jk} {\mathbf{P}(\mathbf{x}, \mathbf{z}})_{ik}} d_{\kappa}(\mathbf{x}_i, \mathbf{z}_k)$ for $i=1,\dots,n$, $j=1,\dots, n'$ and $k=1,\dots,p$. We can also rewrite $W_{2}^{\mathbf{z}}(\mathbf{x}, \mathbf{x'})$ as an $\ell_2$-norm of a vector $\mathbf{v}$ in $\mathbb{R}^{nn'p}$ whose entries are $\sqrt{p{\mathbf{P}(\mathbf{x'}, \mathbf{z}})_{jk} {\mathbf{P}(\mathbf{x}, \mathbf{z}})_{ik}} d_{\kappa}(\mathbf{x}_i, \mathbf{x}_j')$. Then by Minkowski inequality for the $\ell_2$-norm, we have
\begin{equation*}
\begin{aligned}
    |W_2(\mathbf{x}, \mathbf{z})-W_{2}^{\mathbf{z}}(\mathbf{x}, \mathbf{x'})|&=|\|\mathbf{u}\|_2-\|\mathbf{v}\|_2| \\
    &\leq \|\mathbf{u}-\mathbf{v}\|_2 \\
    &=\left( \sum_{i=1}^n \sum_{k=1}^p \sum_{j=1}^{n'} p{\mathbf{P}(\mathbf{x'}, \mathbf{z}})_{jk} {\mathbf{P}(\mathbf{x}, \mathbf{z}})_{ik} (d_{\kappa}(\mathbf{x}_i, \mathbf{z}_k) - d_{\kappa}(\mathbf{x}_i,\mathbf{x'}_j))^2 \right)^{\nicefrac{1}{2}} \\
    & \leq \left( \sum_{i=1}^n \sum_{k=1}^p \sum_{j=1}^{n'} p{\mathbf{P}(\mathbf{x'}, \mathbf{z}})_{jk} {\mathbf{P}(\mathbf{x}, \mathbf{z}})_{ik} d_{\kappa}^2(\mathbf{x'}_j, \mathbf{z}_k) \right)^{\nicefrac{1}{2}} \\
    &= \left( \sum_{k=1}^p \sum_{j=1}^{n'} {\mathbf{P}(\mathbf{x'}, \mathbf{z}})_{jk} d_{\kappa}^2(\mathbf{x'}_j, \mathbf{z}_k) \right)^{\nicefrac{1}{2}} \\
    &=W_2(\mathbf{x'}, \mathbf{z}),
\end{aligned}
\end{equation*}
where the second inequality is the triangle inequality for the distance $d_{\kappa}$.
Finally, we have
\begin{align*}
    & |W_2(\mathbf{x}, \mathbf{x'}) - W_{2}^{\mathbf{z}}(\mathbf{x}, \mathbf{x'})| \\ \leq & |W_2(\mathbf{x}, \mathbf{x'}) - W_2(\mathbf{x}, \mathbf{z})| + |W_2(\mathbf{x}, \mathbf{z}) -W_{2}^{\mathbf{z}}(\mathbf{x}, \mathbf{x'})| \\
    \leq & W_2(\mathbf{x'}, \mathbf{z})+W_2(\mathbf{x'}, \mathbf{z}) \\
    = & 2W_2(\mathbf{x'}, \mathbf{z}),
\end{align*}
where the second inequality is the triangle inequality for the 2-Wasserstein distance. By symmetry, we also have $|W_2(\mathbf{x}, \mathbf{x'}) - W_{2}^{\mathbf{z}}(\mathbf{x}, \mathbf{x'})| \leq 2W_2(\mathbf{x}, \mathbf{z})$, which concludes the proof.
\end{proof}

\subsection{Relationship between \texorpdfstring{$W_2$}. and \texorpdfstring{$W_2^{\mathbf{z}}$}. for multiple references}
\label{subsection:multi_ref}

Using the relation proved in Appendix~\ref{subsection:ot_approx}, we can obtain a bound on the error term between $W_2$ and $W_2^{\mathbf{z}}$ for a data set of $m$ samples ($\mathbf{x}^1, \dots, \mathbf{x}^m$) and $q$ references ($\mathbf{z}^1, \dots, \mathbf{z}^q$)
\begin{equation}
\label{eq:approx_error}
    \mathcal{E}^2 
    := \frac{1}{m^2} \sum_{i,j=1}^m |W_2(\mathbf{x}^i,\mathbf{x}^j) - W_{2}^{\mathbf{z}^1,\dots,\mathbf{z}^q}(\mathbf{x}^i, \mathbf{x}^j)|^2
    \leq \frac{4}{mq} \sum_{i=1}^m\sum_{j=1}^q W_2^2(\mathbf{x}^i, \mathbf{z}^j).
\end{equation}
When $q=1$, the right-hand term in the inequality is the objective to minimize in the Wasserstein barycenter problem~\citep{Cuturi2013a}, which further explains why we considered it: Once $W_{2}^{\mathbf{z}}$ is close to the Wasserstein distance $W_{2}$, $K_{\mathbf{z}}$ will also be close to $K_{\text{OT}}$. 
We extend here the bound in~\eqref{eq:approx_error_one_ref} in the case of one reference to the multiple-reference case. The approximate 2-Wasserstein distance $W_{2}^{\mathbf{z}}(\mathbf{x}, \mathbf{x'})$ thus becomes
\begin{equation*}
    W_{2}^{\mathbf{z}^1,\dots, \mathbf{z}^q}(\mathbf{x}, \mathbf{x'}):=\left\langle \frac{1}{q}\sum_{j=1}^q \mathbf{P}_{\mathbf{z}^j}(\mathbf{x}, \mathbf{x'}), d_{\kappa}^2(\mathbf{x}, \mathbf{x'}) \right\rangle^{\nicefrac{1}{2}}= \left(\frac{1}{q}\sum_{j=1}^q W_{2}^{\mathbf{z}^j}(\mathbf{x}, \mathbf{x'})^2\right) ^{\nicefrac{1}{2}}.
\end{equation*}
Then by Minkowski inequality for the $\ell_2$-norm we have
\begin{equation*}
\begin{aligned}
    |W_2(\mathbf{x}, \mathbf{x'}) - W_{2}^{\mathbf{z}^1,\dots, \mathbf{z}^q}(\mathbf{x}, \mathbf{x'})| &= \left| \left(\frac{1}{q}\sum_{j=1}^q W_{2}(\mathbf{x}, \mathbf{x'})^2\right) ^{\nicefrac{1}{2}}-\left(\frac{1}{q}\sum_{j=1}^q W_{2}^{\mathbf{z}^j}(\mathbf{x}, \mathbf{x'})^2\right) ^{\nicefrac{1}{2}} \right| \\
    &\leq \left(\frac{1}{q}\sum_{j=1}^q (W_{2}(\mathbf{x}, \mathbf{x'})-W_{2}^{\mathbf{z}^j}(\mathbf{x}, \mathbf{x'}))^2 \right) ^{\nicefrac{1}{2}},
\end{aligned}
\end{equation*}
and by~\eqref{eq:approx_error_one_ref} we have
\begin{equation*}
    |W_2(\mathbf{x}, \mathbf{x'}) - W_{2}^{\mathbf{z}^1,\dots, \mathbf{z}^q}(\mathbf{x}, \mathbf{x'})|\leq \left(\frac{4}{q}\sum_{j=1}^q \min(W_{2}(\mathbf{x}, \mathbf{z}^j), W_{2}(\mathbf{x'}, \mathbf{z}^j))^2 \right) ^{\nicefrac{1}{2}}.
\end{equation*}
Finally the approximation error in terms of Frobenius is bounded by
\begin{equation*}
    \mathcal{E}^2 := \frac{1}{m^2} \sum_{i,j=1}^m |W_2(\mathbf{x}^i,\mathbf{x}^j) - W_{2}^{\mathbf{z}^1,\dots,\mathbf{z}^q}(\mathbf{x}^i, \mathbf{x}^j)|^2
    \leq \frac{4}{mq}\sum_{i=1}^m\sum_{j=1}^q W_2^2(\mathbf{x}^i, \mathbf{z}^j).
\end{equation*}
In particular, when $q=1$ that is the case of single reference, we have
\begin{equation*}
    \mathcal{E}^2 \leq \frac{4}{m}\sum_{i=1}^m W_2^2(\mathbf{x}^i, \mathbf{z}), 
\end{equation*}
where the right term equals to the objective of the Wasserstein barycenter problem, which justifies the choice of $\mathbf{z}$ when learning without supervision.

\section{Additional Experiments and Setup Details}
\label{section:add_experiments}
This section contains additional results for the experiments of the main section; details on our setup, in particular hyper-parameter tuning for our methods and the baselines.

\subsection{Protein fold recognition}\label{subsection:app_scop}

\paragraph{Dataset description.}
Our protein fold recognition experiments consider the Structural Classification Of Proteins (SCOP) version 1.75 and 2.06. We follow the data preprocessing protocols in~\citet{Hou2017}, which yields a training and validation set composed of 14699 and 2013 sequences from SCOP 1.75, and a test set of 2533 sequences from SCOP 2.06. The resulting protein sequences belong to 1195 different folds, thus the problem is formulated as a multi-classification task. The input sequence is represented as a 45-dimensional vector at each amino acid. The vector consists of a 20-dimensional one-hot encoding of the sequence, a 20-dimensional position-specific scoring matrix (PSSM) representing the profile of amino acids, a 3-class secondary structure represented by a one-hot vector and a 2-class solvent accessibility. The lengths of the sequences are varying from tens to thousands.

\paragraph{Models setting and hyperparameters.}
We consider here the one-layer models followed by a global mean pooling for the baseline methods CKN~\citep{Dexiong2019a} and RKN~\citep{Dexiong2019b}. We build our model on top of the one-layer CKN model, where $\kappa$ can be seen as a Gaussian kernel on the k-mers in sequences.
The only difference between our model and CKN is thus the pooling operation, which is given by our embedding introduced in Section~\ref{section:embedding}. The bandwidth parameter of the Gaussian kernel $\kappa$ on k-mers is fixed to 0.6 for unsupervised models and 0.5 for supervised models, the same as used in CKN which were selected by the accuracy on the validation set. The filter size $k$ is fixed to 10 and different numbers of anchor points in Nystr\"om for $\kappa$ are considered in the experiments. The other hyperparameters for our embedding are the entropic regularization parameter $\varepsilon$, the number of supports in a reference $p$, the number of references $q$, the number of iterations for Sinkhorn's algorithm and the regularization parameter $\lambda$ in the linear classifier. The search grid for $\varepsilon$ and $\lambda$ is shown in Table~\ref{tab:hyper_scop} and they are selected by the accuracy on validation set. $\varepsilon$  plays an important role in the performance and is observed to be stable for the same dataset. For this dataset, it is selected to be 0.5 for all the unsupervised and supervised models. The effect of other hyperparameters will be discussed below.

For the baseline methods, the accuracies of PSI-BLAST and DeepSF are taken from~\citet{Hou2017}. The hyperparameters for CKN and RKN can be found in~\citet{Dexiong2019b}. For Rep the Set~\citep{Skianis2020} and Set Transformer~\citep{Lee2019}, we use the public implementations by the authors. These two models are used on the top of a convolutional layer of the same filter size as CKN to extract $k$-mer features. As the exact version of Rep the Set does not provide any implementation for back-propagation to a bottom layer of it, we consider the approximate version of Rep the Set only, which also scales better to our dataset. The default architecture of Set Transformer did not perform well due to overfitting. We therefore used a shallower architecture with one ISAB, one PMA and one linear layer, similar to the one-layer architectures of CKN and our model. We tuned their model hyperparameters, weight decay and learning rate. The search grids for these hyperparameters are shown in Table~\ref{tab:hyper_scop_baselines}.

\begin{table}
\small
    \centering
    \caption{Hyperparameter search grid for SCOP 1.75}
    \label{tab:hyper_scop}
    \begin{tabular}{lc}\toprule
         Hyperparameter & Search range  \\ \midrule
         $\varepsilon$ for Sinkhorn & $[1.0; 0.5; 0.1; 0.05; 0.01]$ \\
         $\lambda$ for classifier (unsupervised setting) & $1/2^{\text{range}(5,20)}$ \\
         $\lambda$ for classifier (supervised setting) & [1e-6;1e-5;1e-4;1e-3] \\ \bottomrule
    \end{tabular}
\end{table}

\begin{table}
    \small
    \centering
    \caption{Hyperparameter search grid for SCOP 1.75 baselines.}
    \label{tab:hyper_scop_baselines}
    \begin{tabular}{lc}\toprule
         Model and Hyperparameter & Search range \\
         \midrule
         {ApproxRepSet: Hidden Sets $\times$ Cardinality} & $[20;30;50;100] \times [10;20;50]$\\
         {ApproxRepSet: Learning Rate} & $[0.0001;0.0005;0.001]$\\
         {ApproxRepSet: Weight Decay} & [1e-5;1e-4;1e-3;1e-2] \\
         {Set Transformer: Heads $\times$ Dim Hidden} & $[1;4;8] \times [64;128;256]$ \\
         {Set Transformer: Learning Rate} & $[0.0001;0.0005;0.001]$ \\
         {Set Transformer: Weight Decay} & [1e-5;1e-4;1e-3;1e-2] \\
         \bottomrule
    \end{tabular}
\end{table}

\paragraph{Unsupervised embedding.}
The kernel embedding $\varphi$, which is infinite dimensional for the Gaussian kernel, is approximated with the Nystr\"om method using K-means on 300000 k-mers extracted from the same training set as in~\cite{Dexiong2019b}. The reference measures are learned by using either K-means or Wasserstein to update centroids in 2-Wasserstein K-means on 3000 subsampled sequences for RAM-saving reason. We evaluate our model on top of features extracted from CKNs of different dimensions, representing the number of anchor points used to approximate $\kappa$. The number of iterations for Sinkhorn is fixed to 100 to ensure the convergence. The results for different combinations of $q$ and $p$ are provided in Table~\ref{tab:app_scop_unsup}. Increasing the number of supports $p$ can improve the performance and then saturate it when $p$ is too large. On the other hand, increasing the number of references while keeping the embedding dimension (\textit{i.e.} $p \times q$) constant is not significantly helpful in this unsupervised setting. We also notice that Wasserstein Barycenter for learning the references does not outperform K-means, while the latter is faster in terms of computation.

\begin{table}
\small
    \centering
    \caption{Classification accuracy (top 1/5/10) results of our unsupervised embedding for SCOP 1.75. We show the results for different combinations of (number of references $q$ $\times$ number of supports $p$). The reference measures $\mathbf{z}$ are learned with either K-means or Wasserstein barycenter for updating centroids.}
    \label{tab:app_scop_unsup}
    \scalebox{0.94}{
    \begin{tabular}{llccccc}\toprule
         \multirow{2}{*}{Nb. filters} & \multirow{2}{*}{Method} & \multirow{2}{*}{$q$} & \multicolumn{4}{c}{Embedding size ($q\times p$)}  \\ \cmidrule{4-7}
         & & & 10 & 50 & 100 & 200 \\ \midrule
         \multirow{6}{*}{128} & \multirow{3}{*}{K-means} & 1 & 76.5/91.5/94.4 & 77.5/91.7/94.5 & 79.4/92.4/94.9 & 78.7/92.1/95.1 \\
         & & 5 & 72.8/89.9/93.7 & 77.8/91.7/94.6 & 78.6/91.9/94.6 & 78.1/92.1/94.7 \\ 
         & & 10 & 62.7/85.8/91.1 & 76.5/91.0/94.2 & 78.1/92.2/94.9 & 78.6/92.2/94.7 \\ \cmidrule{2-7}
         & \multirow{3}{*}{Wass. bary.} & 1 & 64.0/85.9/91.5 & 71.6/88.9/93.2 & 77.2/91.4/94.2 & 77.5/91.9/94.8 \\
         & & 5 & 70.5/89.1/93.0 & 76.6/91.3/94.4 & 78.4/91.7/94.3 & 77.1/91.9/94.7 \\ 
         & & 10 & 63.0/85.7/91.0 & 75.9/91.4/94.3 & 77.5/91.9/94.6 & 77.7/92.0/94.7 \\ \midrule
         \multirow{3}{*}{1024} & \multirow{3}{*}{K-means} & 1 & 84.4/95.0/96.6 & 84.6/95.0/97.0 & 85.7/95.3/96.7 & 85.4/95.2/96.7 \\
         & & 5 & 81.1/94.0/96.2 & 84.9/94.8/96.8 & 84.7/94.4/96.7 & 85.2/95.0/96.7 \\ 
         & & 10  & 79.8/93.5/95.9 & 83.1/94.6/96.6 & 84.4/94.7/96.7 & 84.8/94.9/96.7 \\ \bottomrule
    \end{tabular}
    }
\end{table}

\begin{table}
\small
    \centering
    \caption{Classification accuracy (top 1/5/10) results of our unsupervised embedding for SCOP 1.75. We show the results for different number of Nyström anchors. The number of references and supports are fixed to 1 and 100.}
    \label{tab:app_scop_unsup_more}
    \begin{tabular}{lc}\toprule
        {Number of anchors} & {Accuracies} \\ \midrule
          1024 & 85.8/95.3/96.8 \\
          2048 & 86.6/95.9/97.2 \\
          3072 & 87.8/96.1/97.4 \\
          \bottomrule
    \end{tabular}
\end{table}

\paragraph{Supervised embedding.}
Our supervised embedding is initialized with the unsupervised method and then trained in an alternating fashion which was also used for CKN: we use an Adam algorithm to update anchor points in Nystr\"om and reference measures $\mathbf{z}$, and the L-BFGS algorithm to optimize the classifier. The learning rate for Adam is initialized with 0.01 and halved as long as there is no decrease of the validation loss for 5 successive epochs. In practice, we notice that using a small number of Sinkhorn iterations can achieve similar performance to a large number of iteration, while being much faster to compute. We thus fix it to 10 throughout the experiments. The accuracy results are obtained by averaging on 10 runs with different seeds following the setting in~\cite{Dexiong2019b}. The results are shown in Table~\ref{tab:app_scop_sup} with error bars. The effect of the number of supports $q$ is similar to the unsupervised case, while increasing the number of references can indeed improve performance.

\begin{table}
    \small
    \centering
    \caption{Classification accuracy (top 1/5/10) of supervised models for SCOP 1.75. The accuracies obtained by averaging 10 different runs. We show the results of using either one reference with 50 supports or 5 references with 10 supports. Here DeepSF is a 10-layer CNN model.}\label{tab:app_scop_sup}
    \resizebox{\textwidth}{!}{
    \begin{tabular}{lccc} \toprule
         Method & Runtime & \multicolumn{2}{c}{Top 1/5/10 accuracy on SCOP 2.06}  \\ 
         \midrule
         PSI-BLAST~\citep{Hou2017} & - & \multicolumn{2}{c}{84.53/86.48/87.34} \\
         DeepSF~\citep{Hou2017} & - &  \multicolumn{2}{c}{73.00/90.25/94.51} \\
         Set Transformer~\citep{Lee2019} & 3.3h & \multicolumn{2}{c}{79.15$\pm$4.61/91.54$\pm$1.40/94.33$\pm$0.63} \\
         ApproxRepSet~\citep{Skianis2020} & 2h & \multicolumn{2}{c}{84.51$\pm$0.58/94.03$\pm$0.44/95.73$\pm$0.37} \\
         \midrule
         Number of filters & & 128 & 512 \\ \midrule
         CKN~\citep{Dexiong2019a} & 1.5h & 76.30$\pm$0.70/92.17$\pm$0.16/95.27$\pm$0.17 & 84.11$\pm$0.11/94.29$\pm$0.20/96.36$\pm$0.13 \\
         RKN~\citep{Dexiong2019b} & - & 77.82$\pm$0.35/92.89$\pm$0.19/95.51$\pm$0.20 & 85.29$\pm$0.27/94.95$\pm$0.15/96.54$\pm$0.12 \\ \midrule
         Ours \\
         $\Phi_{\mathbf{z}}$ (1 $\times$ 50) & 3.5h & 82.83$\pm$0.41/93.89$\pm$0.33/96.23$\pm$0.12 & 88.40$\pm$0.22/95.76$\pm$0.13/97.10$\pm$0.15 \\
         $\Phi_{\mathbf{z}}$ (5 $\times$ 10) & 4h & \textbf{84.68$\pm$0.50/94.68$\pm$0.29/96.49$\pm$0.18} & \textbf{88.66$\pm$0.25/95.90$\pm$0.15/97.33$\pm$0.14} \\ \bottomrule
    \end{tabular}
    }
\end{table}

\begin{table}
\small
    \centering
    \caption{Classification accuracy (top 1/5/10) results of our supervised embedding for SCOP 1.75. We show the results for different combinations of (number of references $q$ $\times$ number of supports $p$). The reference measures $\mathbf{z}$ are learned with K-means.}
    \label{tab:app_scop_sup_more}
    \scalebox{0.94}{
    \begin{tabular}{lcccc}\toprule
        {Embedding size ($q \times p$)} & 10 & 50 & 100 & 200 \\ \midrule
          $q=1$ & 88.3/95.5/97.0 & 88.4/95.8/97.2 & 87.1/94.9/96.7 & 87.7/94.9/96.3 \\
          $q=2$ & 87.8/95.8/97.0 & 89.6/96.2/97.5 & 86.5/94.9/96.6 & 87.6/94.9/96.3 \\
          $q=5$ & 87.0/95.1/96.7 & 88.8/96.0/97.2 & 87.4/95.4/97.0 & 87.4/94.7/96.2 \\
          $q=10$ & 84.5/93.6/95.6 & 89.8/96.0/97.2 & 88.0/95.7/97.0 & 85.6/94.4/96.1 \\
          \bottomrule
    \end{tabular}
    }
\end{table}

\subsection{Detection of chromatin profiles}\label{subsection:app_deepsea}
\paragraph{Dataset description.}
Predicting the functional effects of non-coding variants from only genomic sequences is a central task in human genetics. A fundamental step for this task is to simultaneously predict large-scale chromatin features from DNA sequences~\citep{zhou2015predicting}. We consider here the DeepSEA dataset, which consists in simultaneously predicting 919 chromatin profiles including 690 transcription factor (TF) binding profiles for 160 different TFs, 125 DNase I sensitivity profiles and 104 histone-mark profiles. In total, there are 4.4 million, 8000 and 455024 samples for training, validation and test. Each sample consists of a 1000-bp DNA sequence from the human GRCh37 reference. Each sequence is represented as a $1000\times 4$ binary matrix using one-hot encoding on DNA characters. The dataset is available at \url{http://deepsea.princeton.edu/media/code/deepsea_train_bundle.v0.9.tar.gz}. Note that the labels for each profile are very imbalanced in this task with few positive samples. For this reason, learning unsupervised models could be intractable as they may require an extremely large number of parameters if junk or redundant sequences cannot be filtered out.

\paragraph{Model architecture and hyperparameters.}
For the above reason and fair comparison, we use here our supervised embedding as a module in Deep NNs. The architecture of our model is shown in Table~\ref{tab:app_deepsea_arch}. We use an Adam optimizer with initial learning rate equal to 0.01 and halved at epoch 1, 4, 8 for 15 epochs in total. The number of iterations for Sinkhorn is fixed to 30. The whole training process takes about 30 hours on a single GTX2080TI GPU. The dropout rate is selected to be 0.4 from the grid $[0.1;0.2;0.3;0.4;0.5]$ and the weight decay is 1e-06, the same as~\cite{zhou2015predicting}. The $\sigma_{\text{pos}}$ for position encoding is selected to be 0.1, by the validation accuracy on the grid $[0.05;0.1;0.2;0.3;0.4;0.5]$. The checkpoint with the best validation accuracy is used to evaluate on the test set. Area under ROC (auROC) and area under precision curve (auPRC), averaged over 919 chromatin profiles, are used to measure the performance. The hidden size $d$ is chosen to be either 1024 or 1536.
\begin{table}
\small
    \centering
    \caption{Model architecture for DeepSEA dataset.}
    \label{tab:app_deepsea_arch}
    \begin{tabular}{l}\toprule
         Model architecture  \\ \midrule
         Conv1d(in channels=4, out channels=$d$, kernel size=16) + ReLU \\
         (Ours) EmbeddingLayer(in channels=$d$, supports=64, references=1, $\varepsilon=1.0$, PE=True, $\sigma_{\text{pos}}=0.1$) \\
         Linear(in channels=$d$, out channels=$d$) + ReLU \\
         Dropout(0.4) \\
         Linear(in channels=$d\times 64$, out channels=919) + ReLU \\
         Linear(in channels=919, out channels=919) \\
         \bottomrule
    \end{tabular}
\end{table}

\paragraph{Results and importance of position encoding.}
We compare our model to the state-of-the-art CNN model DeepSEA~\citep{zhou2015predicting} with 3 convolutional layers, whose best hyper-parameters can be found in the corresponding paper. Our model outperforms DeepSEA, while requiring fewer layers.
The positional information is known to be important in this task. To show the efficacy of our position encoding, we compare it to the sinusoidal encoding used in the original transformer~\citep{vaswani2017}. We observe that our encoding with properly tuned $\sigma_{\text{pos}}$ requires fewer layers, while being interpretable from a kernel point of view. We also find that larger hidden size $d$ performs better, as shown in Table~\ref{tab:app_deepsea}. ROC and PR curves for all the chromatin profiles and stratified by transcription factors, DNase I-hypersensitive sites and histone-marks can also be found in Figure~\ref{fig:app_curves_deepsea}.
\begin{table}
\small
    \centering
    \caption{Results for prediction of chromatin profiles on the DeepSEA dataset. The metrics are area under ROC (auROC) and area under PR curve (auPRC), averaged over 919 chromatin profiles.  The accuracies are averaged from 10 different runs. Armed with the positional encoding (PE) described in Section~\ref{section:embedding}, our embedding outperforms the state-of-the-art model and another model of our embedding with the PE proposed in~\citet{vaswani2017}.}
    \label{tab:app_deepsea}
    \resizebox{\textwidth}{!}{
    \begin{tabular}{lcccc} \toprule
         Method & DeepSEA & Ours & Ours ($d=1024$) & Ours ($d=1536$) \\
         Position encoding & - & Sinusoidal \citep{vaswani2017} & Ours & Ours \\ \midrule
         auROC & 0.933 & 0.917 & 0.935 & \textbf{0.936} \\
         auPRC & 0.342 & 0.311 & 0.354 & \textbf{0.360} \\ \bottomrule
    \end{tabular}
    }
\end{table}
\begin{figure}
    \centering
    \includegraphics[width=0.3\textwidth]{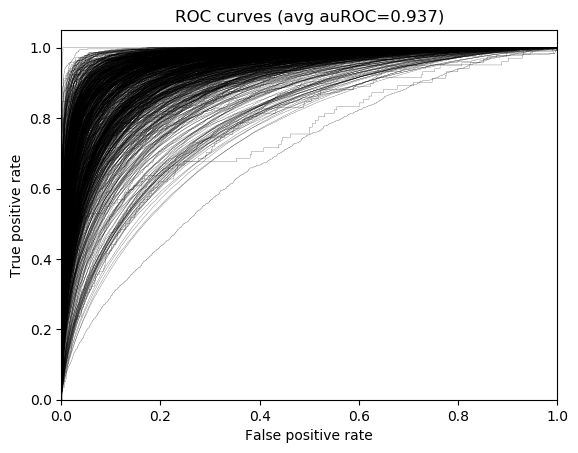}
    \includegraphics[width=0.3\textwidth]{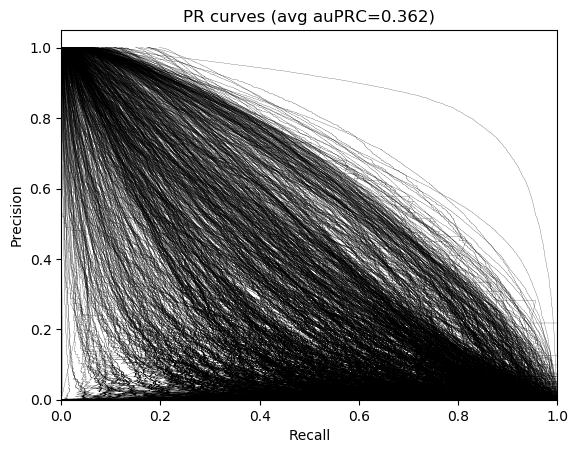} \\
    \includegraphics[width=0.3\textwidth]{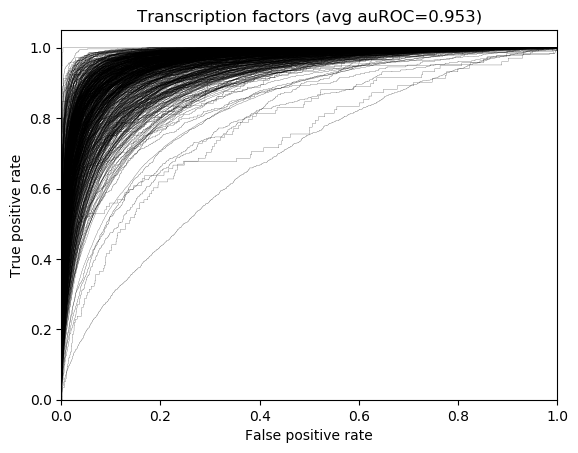}
    \includegraphics[width=0.3\textwidth]{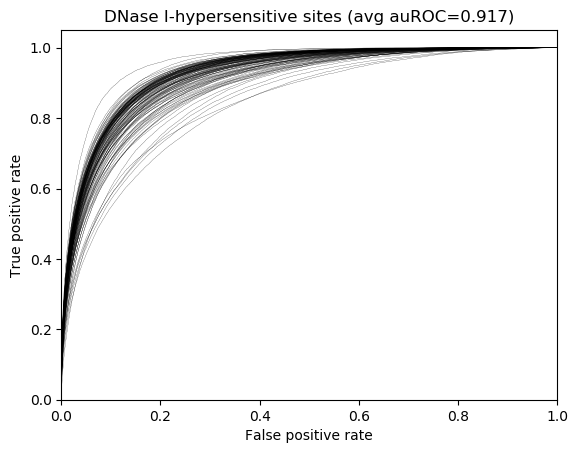}
    \includegraphics[width=0.3\textwidth]{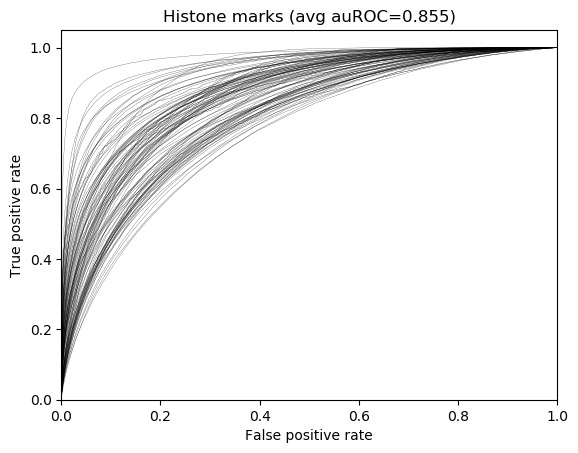}
    \includegraphics[width=0.3\textwidth]{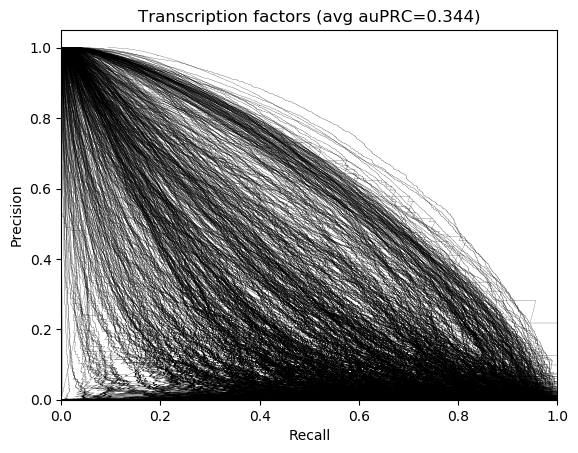}
    \includegraphics[width=0.3\textwidth]{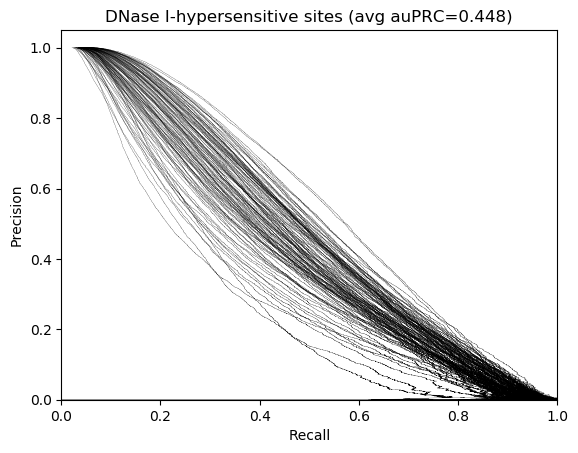}
    \includegraphics[width=0.3\textwidth]{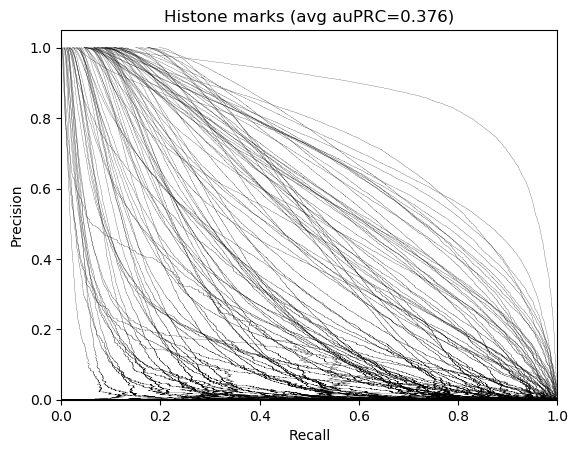}
    \caption{ROC and PR curves for all the chromatin profiles (first row) and stratified by transcription factors (left column), DNase I-hypersensitive sites (middle column) and histone-marks (right column). The profiles with positive samples fewer than 50 on the test set are not taken into account.}
    \label{fig:app_curves_deepsea}
\end{figure}

\subsection{SST-2}

\paragraph{Dataset description.} The data set contains 67,349 training samples and 872 validation samples and can be found at \url{https://gluebenchmark.com/tasks}. The test set contains 1,821 samples for which the predictions need to be submitted on the GLUE leaderboard, with limited number of submissions. As a consequence, our training and validation set are extracted from the original training set ($80\%$ of the original training set is used for our training set and the remaining 20\% is used for our validation set), and we report accuracies on the standard validation set, used as a test set. The reviews are padded with zeros when their length is shorter than the chosen sequence length (we choose $30$ and $66$, the latter being the maximum review length in the data set) and the BERT implementation requires to add special tokens [CLS] and [SEP] at the beginning and the end of each review. 

\paragraph{Model architecture and hyperparameters.} In most transformers such as BERT, the embedding associated to the token [CLS] is used for classification and can be seen in some sense as an embedding of the review adapted to the task. The features we used are the word features provided by the BERT base-uncased version, available at~\url{https://huggingface.co/transformers/pretrained_models.html}. For this version, the dimension of the word features is $768$. 
Our model is one layer of our embedding, with $\varphi$ the Gaussian kernel mapping with varying number of Nyström filters in the supervised setting, and the Linear kernel in the unsupervised setting. We do not add positional encoding as it is already integrated in BERT features. In the unsupervised setting, the output features are used to train a large-scale linear classifier, a Pytorch linear layer. We choose the best hyper-parameters based on the accuracy of a validation set. In the supervised case, the parameters of the previous model, $\mathbf{w}$ and $\mathbf{z}$, are optimized end-to-end. In this case, we tune the learning rate. In both case, we tune the entropic regularization parameter of optimal transport and the bandwidth of the Gaussian kernel. The parameters in the search grid are summed up in Table~\ref{tab:hyper_sst-2}. The best entropic regularization and Gaussian kernel bandwidth are typically and respectively $3.0$ and $0.5$ for this data set. The supervised training process takes between half an hour for smaller models (typically $128$ filters in $\mathbf{w}$ and $3$ supports in $\mathbf{z}$) and a few hours for larger models ($256$ filters and $100$ supports) on a single GTX2080TI GPU. The hyper-parameters of the baselines were similarly tuned, see~\ref{tab:hyper_sst-2_baselines}. Mean Pooling and [CLS] embedding did not require any tuning except for the regularization $\lambda$ of the classifier, which followed the same grid as in Table~\ref{tab:hyper_sst-2}.

\paragraph{Results and discussion.} As explained in Section~\ref{section:experiments}, our unsupervised embedding improves the BERT pre-trained features while still using a simple linear model as shown in Table~\ref{tab:sst-2_refs_supp}, and its supervised counterpart enables to get even closer to the state-of-the art (for the BERT base-uncased model) accuracy, which is usually obtained after fine-tuning of the BERT model on the whole data set. This can be seen in Tables~\ref{tab:sst-2_sup};~\ref{tab:sst-2_baselines}. We also add a baseline consisting of one layer of classical self-attention, which did not do well hence was not reported in the main text.

\begin{table}
\small
    \centering
    \caption{Accuracies on standard validation set for SST-2 with our unsupervised features depending on the number of references and supports. The references were computed using K-means on samples for multiple references and K-means on patches for multiple supports. The size of the input BERT features is ($\text{length} \times \text{dimension}$). The accuracies are averaged from $10$ different runs. ($q$ references $\times$ $p$ supports)}
    \begin{tabular}{lcccc} \toprule
    {BERT Input Feature Size} & \multicolumn{2}{c}{(30 $\times$ 768)} & \multicolumn{2}{c}{(66 $\times$ 768)} \\ \midrule
    {Features} & {Pre-trained} & {Fine-tuned} & {Pre-trained} & {Fine-tuned} \\ \midrule
    \text{[CLS]} & {84.6$\pm$0.3} & {90.3$\pm$0.1} & {86.0$\pm$0.2} & {\textbf{92.8$\pm$0.1}} \\
    \text{Flatten} & {84.9$\pm$0.4} & {\textbf{91.0$\pm$0.1}} & {85.2$\pm$0.3} & {92.5$\pm$0.1} \\
    \text{Mean pooling} & {85.3$\pm$0.3} & {90.8$\pm$0.1} & {85.4$\pm$0.3} & {92.6$\pm$0.2} \\
    \midrule
    \text{$\Phi_{\mathbf{z}}$  (1 $\times$ 3)} & {85.5$\pm$0.1} & {90.9$\pm$0.1} & {86.5$\pm$0.1} & {92.6$\pm$0.1} \\
    \text{$\Phi_{\mathbf{z}}$  (1 $\times$ 10)} & {85.1$\pm$0.4} & {90.9$\pm$0.1} & {85.9$\pm$0.3} & {92.6$\pm$0.1} \\
    \text{$\Phi_{\mathbf{z}}$  (1 $\times$ 30)} & {86.3$\pm$0.3} & {90.8$\pm$0.1} & {86.6$\pm$0.5} & {92.6$\pm$0.1} \\
    \text{$\Phi_{\mathbf{z}}$  (1 $\times$ 100)} & {85.7$\pm$0.7} & {90.9$\pm$0.1} & {86.6$\pm$0.1} & {92.7$\pm$0.1} \\
    \text{$\Phi_{\mathbf{z}}$  (1 $\times$ 300)} & {\textbf{86.8$\pm$0.3}} & {90.9$\pm$0.1} & {\textbf{87.2$\pm$0.1}} & {92.7$\pm$0.1} \\
    \bottomrule
    \end{tabular}
    \label{tab:sst-2_refs_supp}
\end{table}

\begin{table}
\small
    \centering
    \caption{Hyperparameter search grid for SST-2.}
    \label{tab:hyper_sst-2}
    \begin{tabular}{lc}\toprule
         Hyperparameter & Search range  \\ \midrule
         Entropic regularization $\varepsilon$ & $[0.5; 1.0; 3.0; 10.0]$ \\
         $\lambda$ for classifier (unsupervised setting) & $10^{\text{range}(-10,1)}$ \\
         Gaussian kernel bandwidth & $[0.4; 0.5; 0.6; 0.7; 0.8]$ \\
         Learning rate (supervised setting) & $[0.1;0.01;0.001]$ \\
         \bottomrule
    \end{tabular}
\end{table}

\begin{table}
\small
    \centering
    \caption{Hyperparameter search grid for SST-2 baselines.}
    \label{tab:hyper_sst-2_baselines}
    \begin{tabular}{lc}\toprule
         Model and Hyperparameter & Search range \\
         \midrule
         {RepSet and ApproxRepSet: Hidden Sets $\times$ Cardinality} & $[4; 20; 30; 50; 100] \times [3; 10; 20; 30; 50]$\\
         {ApproxRepSet: Learning Rate} & $[0.0001; 0.001; 0.01]$\\
         {Set Transformer: Heads $\times$ Dim Hidden} & $[1; 4] \times [8; 16; 64; 128]$ \\
         {Set Transformer: Learning Rate} & $[0.001; 0.01]$ \\
         \bottomrule
    \end{tabular}
\end{table}

\begin{table}
    \small
    \centering
    \caption{Classification accuracy on standard validation set of supervised models for SST-2, with pre-trained BERT ($30 \times 768$) features. The accuracies of our embedding were averaged from 3 different runs before being run 10 times for the best results for comparison with baselines, cf. Section~\ref{section:experiments}. 10 Sinkhorn iterations were used. ($q$ references $\times$ $p$ supports).}
    \label{tab:sst-2_sup}
    \begin{tabular}{lccc} \toprule
         Method & \multicolumn{3}{c}{Accuracy on SST-2}  \\
         \midrule
         Number of Nyström filters & 32 & 64 & 128 \\ \midrule
         $\Phi_{\mathbf{z}}$ (1 $\times$ 3) & {88.38} & {88.38} & {88.18}  \\
         $\Phi_{\mathbf{z}}$ (1 $\times$ 10) & {88.11} & {88.15} & {87.61} \\
         $\Phi_{\mathbf{z}}$ (1 $\times$ 30) & {88.30} & {88.30} & {88.26} \\
         $\Phi_{\mathbf{z}}$ (4 $\times$ 3) & {88.07} & {88.26} & {88.30} \\
         $\Phi_{\mathbf{z}}$ (4 $\times$ 10) & {87.6} & {87.84} & {88.11} \\
         $\Phi_{\mathbf{z}}$ (4 $\times$ 30) & {88.18} & {$\mathbf{88.68}$} & {88.07} \\ \bottomrule 
    \end{tabular}
\end{table}

\begin{table}
    \small
    \centering
    \caption{Classification accuracy on standard validation set of all baselines for SST-2, with pre-trained BERT ($30 \times 768$) features, averaged from 10 different runs.}
    \label{tab:sst-2_baselines}
    \begin{tabular}{lccc} \toprule
         Method & \multicolumn{3}{c}{Accuracy on SST-2}  \\ 
         \midrule
         {[CLS] embedding~\citep{devlin2018}} & \multicolumn{3}{c}{90.3$\pm$0.1} \\
        {Mean Pooling of BERT features~\citep{devlin2018}} & \multicolumn{3}{c}{$\mathbf{90.8\pm0.1}$} \\
        {One Self-Attention Layer~\citep{vaswani2017}} & \multicolumn{3}{c}{83.7$\pm$0.1} \\
        {Approximate Rep the Set~\citep{Skianis2020}} & \multicolumn{3}{c}{86.8$\pm$0.9} \\
        {Rep the Set~\citep{Skianis2020}} & \multicolumn{3}{c}{87.1$\pm$0.5} \\
        {Set Transformer~\citep{Lee2019}} & \multicolumn{3}{c}{87.9$\pm$0.8} \\
        {$\Phi_{\mathbf{z}}$ (1 $\times$ 30) (dot-product instead of OT)} & \multicolumn{3}{c}{86.9$\pm$1.1} \\
        \bottomrule 
    \end{tabular}
\end{table}

\cleardoublepage

\chapter{Encoding Graph Structure in Transformers with Kernels on Graphs}
\label{chapt:3_graphit}
We show that viewing graphs as sets of node features and incorporating structural and positional information into a transformer architecture is able to outperform representations learned with classical graph neural networks (GNNs). Our model, GraphiT, encodes such information by (i) leveraging relative positional encoding strategies in self-attention scores based on positive definite kernels on graphs, and (ii) enumerating and encoding local sub-structures such as paths of short length. We thoroughly evaluate these two ideas on many classification and regression tasks, demonstrating the effectiveness of each of them independently, as well as their combination. In addition to performing well on standard benchmarks, our model also admits natural visualization mechanisms for interpreting graph motifs explaining the predictions, making it a potentially strong candidate for scientific applications where interpretation is important.\footnote{Code available at \url{https://github.com/inria-thoth/GraphiT}.}

\paragraph{This chapter is based on the following material:} 

\begin{center}
\begin{tcolorbox}[width=\linewidth, sharp corners=all, colback=white!95!black]
    \begin{itemize}
        \item[-] G. Mialon*, D. Chen*, M. Selosse*, J. Mairal. "GraphiT: Encoding Graph Structure in Transformers" (preprint arXiv:2106.05667).
    \end{itemize}
\end{tcolorbox}
\end{center}

\section{Introduction} 

Graph-structured data are present in numerous scientific applications and are the subject of growing interest. Examples of such data are as varied as proteins in computational biology~\citep{senior2020improved}, which may be seen as a sequence of amino acids, but also as a graph representing their tertiary structure, molecules in chemo-informatics~\citep{duvenaud2015convolutional}, shapes in computer vision and computer graphics~\citep{verma2018feastnet}, electronic health records, or communities in social networks. In the case of molecules for examples, the atoms can be seens as nodes and chemical bonds as edges as in Figure~\ref{fig:theobromin}. Designing graph representations for machine learning is a particularly active area of research, even though not new~\citep{borgwardt2020graph}, with a strong effort currently focused on graph neural networks~\citep{bronstein17,Chen2020,kipf2017semisupervised,scarselli2008graph,Velickovic2018,xu2019powerful}. A major difficulty is to find graph representations that are computationally tractable, adaptable to a given task, and capable of distinguishing graphs with different topological structures and local characteristics.

\begin{figure}
    \centering
    \includegraphics[scale=.25]{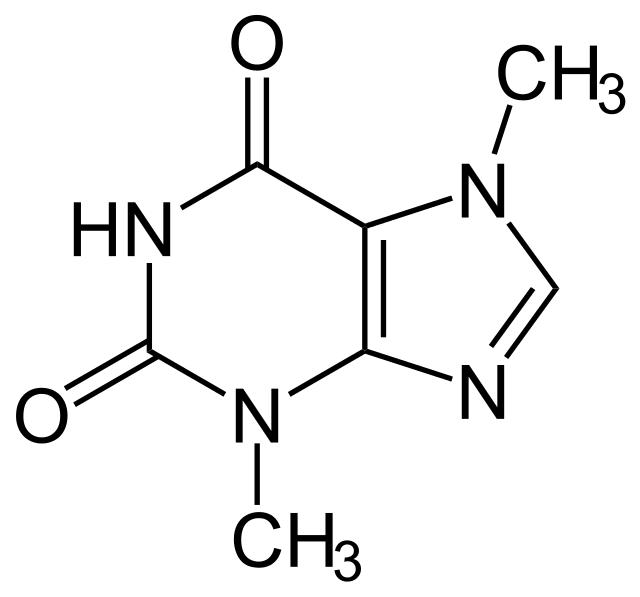}
    \caption{A molecule of Theobromin, hypothetically responsible for the positive effect of chocolate on the mood (from Wikipedia).}
    \label{fig:theobromin}
\end{figure}

In this work, we are interested in the transformer, which has become the standard architecture for natural language processing~\citep{vaswani2017}, and is gaining ground in computer vision~\citep{dosovitskiy2021an} and computational biology~\citep{rives2019biological}. The ability of a transformer to aggregate information across long contexts for sequence data makes it an interesting challenger for GNNs that successively aggregate local information from neighbors in a multilayer fashion. In contrast, a single self-attention layer of the transformer can potentially allow all the nodes of an input graph to communicate. 
The price to pay is that this core component is invariant to permutations of the input nodes, hence does not take the topological structure of the graph into account and looses crucial information. 

For sequence data, this issue is addressed by encoding positional information of each token and giving it to the transformer architecture. For graphs, the problem is more challenging as there is no single way to encode node positions. For this reason, there have been few attempts to use vanilla transformers for graph representation. To the best of our knowledge, the closest work to ours seems to be the graph transformer architecture of~\cite{dwivedi2021generalization}, who 
propose an elegant positional encoding strategy based on the eigenvectors of the graph Laplacian~\citep{belkin2003}. However, they focus on applying attention to neighboring nodes only, as in GNNs, and their results suggest that letting all nodes communicate is not a competitive approach.


Our work provides another perspective and reaches a slightly different conclusion; we show that even though local communication is indeed often more effective, the transformer with global communication can also achieve good results. For that, we introduce a set of techniques to encode the local graph structure within our model, GraphiT (encoding \textbf{graph} structure \textbf{i}n \textbf{t}ransformers). More precisely, GraphiT relies on two ingredients that may be combined, or used independently.
First, we propose relative positional encoding strategies for weighting attention scores by using positive definite kernels, a viewpoint introduced for sequences in~\cite{tsai2019transformer}. This concept is particularly appealing for graphs since it allows to leverage the rich literature on kernels on graphs~\cite{kondor04,smola2003graphkernel}, which are powerful tools for encoding the similarity between nodes. The second idea consists in computing features encoding the local structure in the graph. To achieve this, we leverage the principle of graph convolutional kernel networks (GCKN) of~\cite{Chen2020}, which consists in enumerating and encoding small sub-structure (for instance, paths or subtree patterns), which may then be used as an input to the transformer model.

We demonstrate the effectiveness of our approach on several classification and regression benchmarks, showing that GraphiT with global or local attention layer can outperform GNNs in various tasks, and also show that basic visualization mechanisms allow us to automatically discover discriminative graph motifs, which is potentially useful for scientific applications where interpretation is important.

\paragraph{Contributions.} In summary, our contribution is three-folds:
\begin{itemize}
    \item[-] We introduce two mechanism within the transformer architecture to integrate structural information from graphs.
    \item[-] We demonstrate the effectiveness of this approach, which is able to outperform GNNs in various tasks.
    \item[-] We show that basic visualization mechanisms allows to recover meaningful graph motifs.
\end{itemize}

\section{Related work}\label{sec:related}

\paragraph{Graph kernels.} 
A classical way to represent graphs for machine learning tasks consists in defining a high-dimensional embedding for graphs, which may then be used to perform prediction with a linear models (\emph{e.g.}, support vector machines). Graph kernels typically provide such embeddings by counting the number of occurrences of local substructures that are present within a graph~\citep{borgwardt2020graph}. The goal is to choose substructures leading to expressive representations sufficiently discriminative, while allowing fast algorithms to compute the inner-products between these embeddings. For instance, walks have been used for such a purpose~\citep{gartner2003graph}, as well as shortest paths~\citep{borgwardt2005shortest}, subtrees~\citep{harchaoui2007image,mahe2009graph,shervashidze2011weisfeiler}, or graphlets~\citep{shervashidze2009efficient}. Our work uses short paths, but other substructures could be used in principle. Note that graph kernels used in the context of comparing graphs, is a line of work different from the kernels on graphs that we will introduce in Section~\ref{sec:preliminaries} for computing embeddings of nodes.

\vs
\paragraph{Graph neural networks.}
Originally introduced in~\cite{scarselli2008graph}, GNNs have been derived as an extension of convolutions for graph-structured data: they use a message passing paradigm in which vectors (messages) are exchanged (passed) between neighboring nodes whose representations are updated using neural networks. Many strategies have been proposed to aggregate features of neighboring nodes (see, \textit{e.g},~\cite{ bronstein17,duvenaud2015convolutional}). The graph attention network (GAT)~\citep{Velickovic2018} is the first model to use an attention mechanism for aggregating local information. Recently, hybrid approaches between graph neural networks and graph kernels were proposed in~\cite{Chen2020,du2019graph}. Diffusion processes on graphs that are related to the diffusion kernel we consider in our model were also used within GNNs in~\cite{klicpera2019diffusion}.


\vs
\paragraph{Transformers for graph-structured data.} 

Prior to~\cite{dwivedi2021generalization}, there were some attempts to use transformers in the context of graph-structured data. The authors of~\cite{li2019} propose to apply attention to all nodes, yet without position encoding. In \cite{zhang2020graphbert}, a transformer architecture called Graph-BERT is fed with sampled subgraphs of a single graph in the context of node classification and graph clustering. They also propose to encode positions by aggregating different encoding schemes. However, these encoding schemes are either impossible to use in our settings as they require having sampled subgraphs of regular structures as input, or less competitive than Laplacian eigenvectors as observed in~\cite{dwivedi2021generalization}. The transformer model introduced in \cite{yun2019} needs to first transform an {heterogeneous} graph into a new graph structure via meta-paths, which does not directly operate on node features. To the best of our knowledge, our work is the first to demonstrate that vanilla transformers with appropriate node position encoding can compete with GNNs in graph prediction tasks.

\section{Preliminaries about Kernels on Graphs}
\label{sec:preliminaries}


\paragraph{Spectral graph analysis.}
 The Laplacian of a graph with $n$ nodes is defined as $L = D - A$, where~$D$ is a $n \times n$ diagonal matrix that carries the degrees of each node on the diagonal and $A$ is the adjacency matrix. Interestingly, $L$ is a positive semi-definite matrix such that for all vector $u$ in $\Real^n$, $u^\top L u = \sum_{i\sim j} (u[i]-u[j])^2$, which can be interpreted as the amount of ``oscillation'' of the vector~$u$, hence its ``smoothness', when seen as a function on the nodes of the graph.
 
 The Laplacian is often used via its eigenvalue decomposition $L= \sum_{i} \lambda_i u_i u_i^\top$, where the eigenvalue $\lambda_i = u_i^\top L u_i$ characterizes the amount of oscillation of the corresponding eigenvector~$u_i$. For this reason, this decomposition is traditionally viewed as the discrete equivalent to the sine/cosine Fourier basis in $\mathbb{R}^n$ and associated frequencies. Note that very often the normalized Laplacian $I - D^{-\frac{1}{2}} A D^{-\frac{1}{2}}$ is used instead of $L$, which does not change the above interpretation.
 
Interestingly, it is then possible to define a whole family of positive definite kernels on the graph~\citep{smola2003graphkernel} by applying a regularization function $r$ to the spectrum of $L$. We therefore get a rich class of kernels
\begin{equation}
    K_r = \sum_{i=1}^m {r(\lambda_i)} u_i u_i^\top,
\end{equation}
associated with the following norm $\|f\|_r^2 = \sum_{i=1}^m {({f}_i^\top u_i)^2}/{r(\lambda_i)}$ from a reproducing kernel Hilbert space (RKHS),
where $r : \mathbb{R} \mapsto \mathbb{R}_*^+ $ is a non-increasing function such that smoother functions on the graph would have smaller norms in the RKHS. We now introduce two examples. 


\vs
\paragraph{Diffusion Kernel~\cite{kondor04}.} 
It corresponds to the case $r(\lambda_i) = e^{-\beta \lambda_i}$, which gives:
\begin{equation}
    K_{D} = \sum_{i=1}^m e^{- \beta \lambda_i} u_i u_i^\top = e^{- \beta L} = \lim_{p \to +\infty} \left(  I - \frac{\beta}{p}L\right)^p.\label{eq:diffusion}
\end{equation}
The diffusion kernel can be seen as a discrete equivalent of the Gaussian kernel, a solution of the heat equation in the continuous setting, hence its name. Intuitively, the diffusion kernel between two nodes can be interpreted as the quantity of some substance that would accumulate at the first node after a given amount of time (controlled by $\beta$) if we injected the substance at the other node and let it diffuse through the graph. An illustration of the diffusion kernel is presented in Figure~\ref{fig:diffusion}. It is related to the random walk kernel that will be presented below.


\begin{figure}
\centering
\includegraphics[scale=.4]{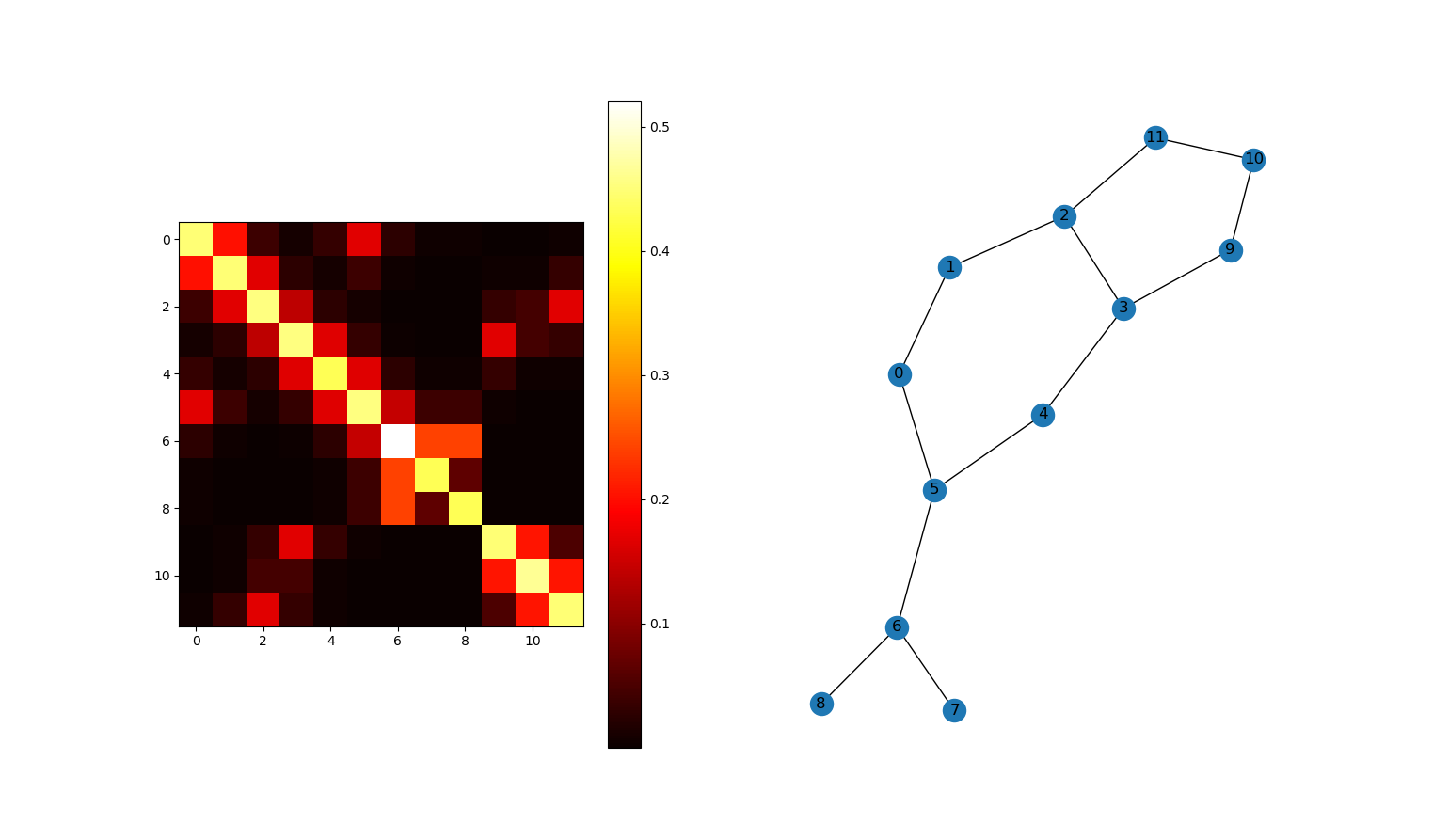}
\caption{Diffusion kernel ($\beta = 1$) for a molecular graph from the MUTAG dataset used in our experiments. Neighboring nodes have higher values. Note that nodes that are structurally similar, such as $7$ and $8$, have similar entries.}
\label{fig:diffusion}
\end{figure}

\paragraph{p-step random walk kernel.} 
By taking $r(\lambda_i) = (1 - \gamma \lambda_i)^{p}$ , we obtain a kernel that admits an interpretation in terms of p steps of a particular random walk on the graph:
\begin{equation}
    K_{pRW} = (I  - \gamma L)^{p}.
\end{equation}
It is related to the diffusion kernel by choosing $\gamma=\beta/p$ and taking the limit with $p \to +\infty$, according to~(\ref{eq:diffusion}). Unlike the diffusion kernel which yields a dense matrix, the random walk kernel is sparse and has limited connectivity, meaning that two nodes are connected in $K_{pRW}$ only if there exists a path of length~$p$ between them.
As these kernels on graphs reflect the structural similarity of the nodes independently of their features,
it is natural to use the resulting Gram matrix to encode such a structure within the transformer model, as detailed next.

\section{Encoding Graph Structure in Transformers}

In this section, we detail the architecture and structure encoding strategies behind GraphiT. In particular, we build upon the findings of~\cite{tsai2019transformer} for our architectural choices and propose new strategies for encoding structural information in the transformer architecture.

\subsection{Transformer Architectures for Graphs}

We process graphs with a vanilla transformer encoder architecture~\citep{vaswani2017} for solving classification and regression tasks, by seeing graphs as sets of node features. 
We first present the transformer without encoding the graph structure, before introducing mechanisms for encoding nodes positions in Section~\ref{subsec:position}, and then topological structures in Section~\ref{subsec:structure}.
Specifically, a transformer is composed of a sequence of layers, which process an input set of $d_{\text{in}}$ features represented by a matrix ${X}$ in $\Real^{n \times d_{\text{in}}}$, and compute another set in $\Real^{n \times d_{\text{out}}}$. A critical component is the attention mechanism:
\begin{equation}
    \text{Attention}(Q, V) = \text{softmax}\left( \frac{QQ^\top}{\sqrt{d_{out}}} \right) V \in \Real^{n \times d_{\text{out}}},
    \label{eq:attention}
\end{equation}
with $Q^\top = W_Q {X}^\top$ is called the query matrix, $V^\top = W_V {X}^\top$ the value matrix, and $W_Q, W_V$ in $\Real^{d_{\text{out}} \times d_{\text{in}}}$ are projection matrices that are learned. Note that following the recommendation of~\cite{tsai2019transformer}, we use the same matrix for keys and queries. This allows us to define a symmetric and positive definite kernel on pairs of nodes and thus to combine with other kernels on graph. This also reduces the number of parameters in our models without hurting the performance in practice. During the forward pass, the feature map $X$ is updated in a residual fashion (with either $d_{\text{in}}=d_{\text{out}}$ or with an additional projection matrix when the dimensions do not match, omitted here for simplicity) as follows:
\begin{equation*}
    {X} = {X} + \text{Attention}(Q, V).
\end{equation*}

Note that transformers without position encoding and GNNs are tightly connected: a GNN can be seen as a transformer where aggregation is performed on neighboring nodes only, and a transformer as a GNN processing a fully-connected graph. However, even in this case, differences would remain between common GNNs and Transformers, as the latter use for example LayerNorm and skip-connections. In this work, we will either adopt the local aggregation strategy, or let all nodes communicate in order to test this inductive bias in the context of graphs, where capturing long range interactions between nodes may be useful for some tasks. 

\subsection{Encoding Node Positions}\label{subsec:position}

The output of the transformer is invariant to permutations in the input data. It is therefore crucial to provide the transformer with information about the data structure. In this section, we revisit previously proposed strategies for sequences and graphs and devise new methods for positional encoding. Note that a natural baseline is simply to adopt a local aggregation strategy similar to~\cite{Velickovic2018} or~\cite{dwivedi2021generalization} for example.

\subsubsection{Existing Strategies for Sequences and Graphs}

\paragraph{Absolute and relative positional encoding in transformers for sequences.} 
In NLP, positional information was initially encoded by adding a vector based on sinusoidal functions to each of the input token features~\cite{vaswani2017}.
This approach was coined as absolute position encoding and proved to be also useful for other data modalities~\citep{dosovitskiy2021an,dwivedi2021generalization}. In relative position encoding, which was proposed later~\citep{parikh-etal-2016-decomposable, shaw2018selfattention}, positional information is added to the attention logits and in the values. This information only depends on the relative distance between the two considered elements in the input sequence. 



\paragraph{Absolute positional encoding for graphs.}

Whereas positional encodings can be hand-crafted or learned relatively easily for sequences and images, which respectively admit a chain or grid structure, this task becomes non-trivial for graphs, whose structure may vary a lot inside a data set, besides the fact that the concept of node position is ill-defined. To address these issues, an absolute position encoding scheme was recently proposed in \cite{dwivedi2021generalization}, by leveraging the graph Laplacian (LapPE). More precisely, each node of each graph is assigned a vector containing the first $k$ coordinates of the node in the eigenbasis of the graph normalized Laplacian sorted in ascending order of the eigenvalues. Since the first eigenvector associated to the eigenvalue 0 is constant, the first coordinate is omitted.

As detailed in Section~\ref{sec:preliminaries}, these eigenvectors oscillate more and more and the corresponding coordinates are often interpreted as Fourier coefficients representing frequency components in increasing order. Note that eigenvectors of the Laplacian computed on different graphs could not be compared to each other in principle, and are also only defined up to a $\pm 1$ factor.
While this raises a conceptual issue for using them in an absolute positional encoding scheme, it is shown in~\cite{dwivedi2021generalization}---and confirmed in our experiments---that the issue is mitigated by the Fourier interpretation, and that
the coordinates used in LapPE are effective in practice for discriminating between nodes in the same way as the position encoding proposed in~\cite{vaswani2017} for sequences. 
Yet, because the eigenvectors are defined up to a $\pm$1 factor, the sign of the encodings needs to be randomly flipped during the training of the network. 
In the next section, we introduce a novel strategy in the spirit of relative positional encoding that does not suffer from the previous conceptual issue, and can also be combined with LapPE if needed.




\subsubsection{Relative Position Encoding Strategies by Using Kernels on Graphs}



\paragraph{Modulating the transformer attention.} 
To avoid the issue of transferability of the absolute positional encoding between graphs, we use information on the nodes structural similarity to bias the attention scores. More precisely, and in the fashion of the last chapter, or~\citet{tsai2019transformer} for sequences, we modulate the attention kernel using the Gram matrix of some kernel on graphs described in Section~\ref{sec:preliminaries} as follows:
\begin{equation}
    \text{PosAttention}(Q, V, K_r) = \text{normalize}\left(\text{exp}\left( \frac{QQ^\top}{\sqrt{d_{\text{out}}}}\right) \odot K_r \right) V \in \Real^{n \times d_{\text{out}}},
    \label{eq:pos_attention}
\end{equation}
with the same $Q$ and $V$ matrices, and $K_r$ a kernel on the graph. ``$\text{exp}$'' denotes the elementwise exponential of the matrix entries and $\text{normalize}$ means $\ell_1$-normalization on rows such that normalization(exp(u))=softmax(u). The reason for multiplying the exponential of the attention logits before $\ell_1$-normalizing the rows is that it corresponds to a classical kernel smoothing. Indeed, if we consider the PosAttention output for a node $i$:
\begin{equation*}
    \text{PosAttention}(Q, V, K_r)_i = \sum_{j=1}^n \frac{\text{exp}\left(\nicefrac{Q_i Q_j^\top}{\sqrt{d_{out}}} \right) \times K_r(i, j)}{\sum_{j'=1}^n \text{exp}\left(\nicefrac{Q_i Q_{j'}^\top}{\sqrt{d_{out}}} \right) \times K_r(i, j')} V_j \in \Real^{d_{\text{out}}},
    \label{eq:pos_attention_i}
\end{equation*}
we obtain a typical smoothing, \textit{i.e}, a linear combination of features with weights determined by a non-negative kernel, here $k(i, j) := \text{exp} (\nicefrac{Q_i Q_j^\top}{\sqrt{d_{out}}}) \times K_r(i, j)$, and summing to 1. In fact, $k$ can be considered as a new kernel between nodes $i$ and $j$ made by multiplying a kernel based on positions ($K_r$) and a kernel based on content (via $Q$). As observed in~\cite{tsai2019transformer} for sequences, modulating the attention logits with a kernel on positions is related to relative positional encoding~\citep{shaw2018selfattention}, where we bias the attention matrix with a term depending only on the relative difference in positions between the two elements in the input set. Moreover, during the forward pass, the feature map ${X}$ is updated as follows:
\begin{equation}
    {X} = {X} + D^{-\frac{1}{2}} \text{PosAttention}(Q, V, K_r),
    \label{eq:pos_forward}
\end{equation}
where $D$ is the matrix of node degrees. We found such a normalization with $D^{-1/2}$ to be beneficial in our experiments since it reduces the overwhelming influence of highly connected graph components.
Note that, as opposed to  absolute position encoding, we do not add positional information to the values and the model does not rely on the transferability of eigenvectors between different Laplacians.

 
 \paragraph{Choice of kernels and parameters.} Interestingly, the choice of the kernel enables to encode a priori knowledge directly within the model architecture, while the parameter has an influence on the attention span. For example, in the diffusion kernel, $\beta$ can be seen as the duration of the diffusion process. The smaller it is, the more focused the attention on the close neighbors. Conversely, a large~$\beta$ corresponds to an homogeneous attention span. As another example, it is clear that the choice of $p$ in the random walk kernel corresponds to visiting at best $p$-degree neighbors. In our experiments, the best kernel may vary across datasets, suggesting that long-range global interactions are of different importance depending on the task. For example, on the dataset PROTEINS (see Section~\ref{sec:experiments}), the vanilla transformer without using any structural information performs very well.





\subsection{Encoding Topological Structures} \label{subsec:structure}

Position encoding aims at adding positional only information to the feature vector of an input node or to the attentions scores. Substructures is a different type of information, carrying local positional information and content, which has been heavily used within graph kernels, see Section~\ref{sec:related}. 
In the context of graphs neural networks, this idea was exploited in the graph convolutional kernel network model (GCKN) of~\cite{Chen2020}, which is a hybrid approach between GNNs and graph kernels based on substructure enumeration (\textit{e.g.}, paths). Among different strategies we experimented, enriching nodes features with the output of a GCKN layer turned out to be a very effective strategy.


\paragraph{Graph convolutional kernel networks (GCKN).}
 GCKNs~\cite{Chen2020} is a multi-layer model that produces a sequence of graph feature maps akin to a GNN. The main difference is that each layer enumerates local sub-structures at each node (here, paths of length~$k$), encodes them using a kernel embedding, and aggregates the resulting sub-structure representations. This results in a feature map that carries more information about the graph structure than traditional neighborhood aggregation based GNNs, which is appealing for transformers since the vanilla version is blind to the graph structure.
 
 Formally, let us consider a graph $G$ with $n$ nodes, and let us denote by $\mathcal{P}_k(u)$ the set of paths shorter than or equal to $k$ that start with node $u$. With an abuse of notation, $p$ in $\mathcal{P}_k(u)$ will denote the concatenation of all node features encountered along the path. Then, a layer of GCKN defines a feature map 
 $X$ in $\Real^{n \times d}$ such that 
\begin{equation*}
    X(u)=\sum_{p \in \mathcal{P}_k(u)}  \psi(p),
\end{equation*}
where $X(u)$ is the column of $X$ corresponding to node $u$ and $\psi$ is a $d$-dimensional
embedding of the path features~$p$. More precisely, the path features in~\cite{Chen2020} are embedded to a RKHS by using a Gaussian kernel, and a finite-dimensional approximation is obtained by using the Nystr\"om method~\cite{williams2001using}. The embedding is parametrized and can be learned without or with supervision (see~\cite{Chen2020} for details). Moreover, path features of varying lengths up to a maximal length can be used with GCKN. In this work, we evaluate the strategy of encoding a node as the concatenation of its original features and those produced by one GCKN layer. This strategy has proven to be very successful in practice.



\section{Experiments}
\label{sec:experiments}

In this section, we evaluate instances of GraphiT as well as popular GNNs and position encoding baselines on various graph classification and regression tasks. We want to answer several questions:
\begin{itemize}[leftmargin=0.7cm,itemsep=0pt,parsep=0pt,topsep=0pt]
    \item[\textbf{Q1}:] Can vanilla transformers, when equipped with appropriate position encoding and/or structural information, outperform GNNs in graph classification and regression tasks?
    \item [\textbf{Q2}:] Is kernel-based relative positional encoding more effective than the absolute position encoding provided by the eigenvectors of the Laplacian (LapPE)?
    \item [\textbf{Q3}:] What is the most effective way to encode graph structure information within transformers? 
\end{itemize}
We also discuss our results and conduct ablation studies. Finally, we demonstrate the ability of attention scores to highlight meaningful graph features when using kernel-based positional encoding.

\subsection{Methodology}

\paragraph{Benchmark and baselines.} 
We benchmark our methods on various graph classification datasets with discrete node labels (MUTAG, PROTEINS, PTC, NCI1) and one regression dataset with discrete node labels (ZINC). These datasets can be obtained \textit{e.g} via the Pytorch Geometric toolbox~\citep{Fey2019}. 
We compare our models to the following GNN models: Molecular Fingerprint (MF)~\citep{duvenaud2015convolutional}, Graph Convolutional Networks (GCN)~\citep{kipf2017semisupervised}, Graph Attention Networks (GAT)~\citep{Velickovic2018}, Graph Isomorphism Networks (GIN)~\citep{xu2019powerful} and finally Graph Convolutional Kernel Networks (GCKN)~\citep{Chen2020}. In particular, GAT is an important baseline as it uses attention to aggregate neighboring node information. We compare GraphiT to the transformer architecture proposed in~\cite{dwivedi2021generalization} and also use their Laplacian absolute position encoding as a baseline for evaluating our graph kernel relative position encoding. 
All models are implemented in Pytorch and our code is available online at~\url{https://github.com/inria-thoth/GraphiT}.


\paragraph{Reporting scores.} 
For all datasets except ZINC, we sample ten times random train/val/test splits, of size $80/10/10$, respectively. 
For each split, we evaluate all methods by (i) training several models on the train fold with various hyperparameters; (ii) performing model selection on val, by averaging the validation accuracy of a model on its last 50 epochs; (iii) retraining the selected model on train+val; (iv) estimate the test score by averaging the test accuracy over the last 50 epochs. 
The results reported in our tables are then averaged over the ten splits.
This procedure is a compromise between a double-nested cross validation procedure that would be too computationally expensive and reporting 10-fold cross validation scores that would overestimate the test accuracy.
For ZINC, we use the same train/val/test splits as in~\cite{dwivedi2021generalization}, train GraphiT with 10 layers, 8 heads and 64 hidden dimensions as in~\cite{dwivedi2021generalization}, and report the average test mean absolute error on 4 independent runs.

\paragraph{Optimization procedure and hyperparameter search.} 
Our models are trained with the Adam optimizer by decreasing the learning rate by a factor of 2 each 50 epochs. For classification tasks, we train about the same number (81) of models with different hyperparameters for each GNN and transformer method, thus spending a similar engineering effort on each method.
For GNN models, we select the best type of global pooling, number of layers and hidden dimensions from three different values. Regarding transformers, we select the best number of heads instead of global pooling type for three different values. For all considered models, we also select the best learning rate and weight decay from three different values and the number of epochs is fixed to 300 to guarantee the convergence. For the ZINC dataset, we found that a standard warmup strategy suggested for transformers in~\cite{vaswani2017} leads to more stable convergence for larger models. The rest of the hyperparameters remains the same as used in~\cite{dwivedi2021generalization}.
More details and precise grids for hyperparameter search can be found in Appendix~\ref{sec:add_exp}.

\subsection{Results and Discussion}

\begin{table}[t]
    \small
    \centering
    \caption{Average mean classification accuracy/mean absolute error.
    }
    \resizebox{\textwidth}{!}{
    \begin{tabular}{l|c|c|c|c|c} \toprule
    {Method / Dataset} & {MUTAG} & {PROTEINS} & {PTC} & {NCI1} & {ZINC (no edge feat.)} \\
    \midrule
    {Size} & {188} & {1113} & {344} & {4110} & {12k} \\
    {Classes} & {2} & {2} & {2} & {2} & {Reg.} \\
    {Max. number of nodes} & {28} & {620} & {109} & {111} & {37} \\
    \midrule
    {MF~\citep{duvenaud2015convolutional}} & {81.5$\pm$11.0} & {71.9$\pm$5.2} & {57.3$\pm$6.9} & {80.6$\pm$2.5} & {0.387$\pm$0.019} \\
    {GCN~\citep{kipf2017semisupervised}} & {78.9$\pm$10.1} & {75.8$\pm$5.5} & {54.0$\pm$6.3} & {75.9$\pm$1.6}  & {0.367$\pm$0.011} \\
    {GAT~\citep{Velickovic2018}} & {80.3$\pm$8.5} & {74.8$\pm$4.1} & {55.0$\pm$6.0} & {76.8$\pm$2.1}  & {0.384$\pm$0.007} \\
    {GIN~\citep{xu2019powerful}} & {82.6$\pm$6.2} & {73.1$\pm$4.6} & {55.0$\pm$8.7} & {\textbf{81.7$\pm$1.7}}  & {0.387$\pm$0.015} \\
    {GCKN-subtree~\citep{Chen2020}} & {87.8$\pm$9.4} & {72.0$\pm$3.7} & {\textbf{62.1$\pm$6.4}} & {79.6$\pm$1.8}  & {0.474$\pm$0.001} \\ 
    \midrule
    \citep{dwivedi2021generalization} & 79.3$\pm$11.6 & 65.8$\pm$3.1 & 58.4$\pm$8.2 & 78.9$\pm$1.1 & 0.359$\pm$0.014 \\
    \citep{dwivedi2021generalization} + LapPE & 83.9$\pm$6.5 & 70.1$\pm$3.2 & 57.7$\pm$3.1 & 80.0$\pm$1.9 & 0.323$\pm$0.013 \\
    \midrule
    {Transformers (T)} & 82.2$\pm$6.3 & 75.6$\pm$4.9 & 58.1$\pm$10.5 & 70.0$\pm$4.5  & 0.696$\pm$0.007 \\
    {T + LapPE} & 85.8$\pm$5.9 & 74.6$\pm$2.7 & 55.6$\pm$5.0 & 74.6$\pm$1.9 & 0.507$\pm$0.003 \\
    {T + Adj PE} & 87.2$\pm$9.8 & 72.4$\pm$4.9 & 59.9$\pm$5.9 & 79.7$\pm$2.0 & 0.243$\pm$0.005\\
    {T + 2-step RW kernel} & 85.3$\pm$6.9 & 72.8$\pm$4.5 & 62.0$\pm$9.4 & 78.0$\pm$1.5 & 0.243$\pm$0.010\\
    {T + 3-step RW kernel} & 83.3$\pm$6.3 & \textbf{76.2$\pm$4.4} & 61.0$\pm$6.2 & 77.6$\pm$3.6 & 0.244$\pm$0.011\\
    {T + Diffusion kernel} & 82.7$\pm$7.6 & 74.6$\pm$4.2 & 59.1$\pm$7.4 & 78.9$\pm$1.6  & {0.255$\pm$0.010} \\
    {T + GCKN } & {84.4$\pm$7.8} & {69.5$\pm$3.8} & {61.5$\pm$5.8} & {78.1$\pm$5.1}  & 0.274$\pm$0.011 \\
    {T + GCKN + 2-step RW kernel} & {90.4$\pm$5.8} & {72.5$\pm$4.6} & {58.4$\pm$7.6} & {81.0$\pm$1.8} & {0.213$\pm$0.016}\\
    {T + GCKN + Adj PE} & {\textbf{90.5$\pm$7.0}} & {71.1$\pm$6.9} & {57.9$\pm$4.2} & {81.4$\pm$2.2} & {\textbf{0.211$\pm$0.010}}\\
    \bottomrule
    \end{tabular}
    }
    \label{tab:sup}
\end{table}

\begin{table}[t]
    \small
    \centering
    \caption{Ablation: comparison of different structural encoding schemes and their combinations.
    }
    \resizebox{\textwidth}{!}{
    \begin{tabular}{l|c|c|c|c|c|c} \toprule
         \multirow{2}{*}{Dataset} & \multirow{2}{3cm}{Structure encoding in node features} & \multicolumn{5}{c}{Relative positional encoding in attention scores}  \\
         & & T vanilla & T + Adj & T + 2-step & T + 3-step & T + Diffusion   \\ \midrule
         \multirow{3}{*}{MUTAG} & None & 82.2$\pm$6.3  & 87.2$\pm$9.8 &  85.3$\pm$6.9 & 83.3$\pm$6.3 & 82.7$\pm$7.6 \\
         & LapPE & 85.8$\pm$5.9 & 86.0$\pm$4.2 & 84.7$\pm$4.7 & 83.5$\pm$5.2 & 84.2$\pm$7.2 \\
         & GCKN & 84.4$\pm$7.8 & \textbf{90.5$\pm$7.0} & 90.4$\pm$5.8 &  90.0$\pm$6.3 & 90.0$\pm$6.8 \\ \midrule
         \multirow{3}{*}{PROTEINS} & None & 75.6$\pm$4.9 & 72.4$\pm$4.9 & 72.8$\pm$4.5 & \textbf{76.2$\pm$4.4} & 74.6$\pm$4.2\\
         & LapPE & 74.6$\pm$2.7 & 74.7$\pm$5.2 & 75.0$\pm$4.7 & 74.3$\pm$5.3 & 74.7$\pm$5.3\\
         & GCKN & 69.5$\pm$3.8 & 71.1$\pm$6.9 & 72.5$\pm$4.6 & 70.0$\pm$5.1 & 72.4$\pm$4.9 \\ \midrule
         \multirow{3}{*}{PTC} & None & 58.1$\pm$10.5 & 59.9$\pm$5.9 & \textbf{62.0$\pm$9.4} & 61.0$\pm$6.2 & 59.1$\pm$7.4 \\
         & LapPE & 55.6$\pm$5.0 & 57.1$\pm$3.8 & 58.8$\pm$6.6 & 57.1$\pm$5.3 & 57.3$\pm$7.8 \\
         & GCKN & 61.5$\pm$5.8 & 57.9$\pm$4.2 & 58.4$\pm$7.6 & 55.2$\pm$8.8 & 55.9$\pm$8.1 \\ \midrule
         \multirow{3}{*}{NCI1} & None & 70.0$\pm$4.5 & 79.7$\pm$2.0 & 78.0$\pm$1.5 & 77.6$\pm$3.6 & 78.9$\pm$1.6 \\
         & LapPE & 74.6$\pm$1.9 & 78.7$\pm$1.9 & 78.4$\pm$1.3 & 78.7$\pm$1.5 & 77.8$\pm$1.0 \\
         & GCKN & 78.1$\pm$5.1 & \textbf{81.4$\pm$2.2} & 81.0$\pm$1.8 & 81.0$\pm$1.8 & 81.0$\pm$2.0 \\ \midrule
         \multirow{3}{*}{ZINC} & None & 0.696$\pm$0.007 & 0.243$\pm$0.005 & 0.243$\pm$0.010 & 0.244$\pm$0.011 & 0.255$\pm$0.010 \\
         & LapPE & 0.507$\pm$0.003 & \textbf{0.202$\pm$0.011} & 0.227$\pm$0.030 & 0.210$\pm$0.003 & 0.221$\pm$0.038 \\
         & GCKN & 0.274$\pm$0.011 & 0.211$\pm$0.010 & 0.213$\pm$0.016 &  0.203$\pm$0.011 & 0.218$\pm$0.006 \\ \bottomrule
    \end{tabular}
    }
    \label{tab:structural_encoding}
\end{table}

\paragraph{Comparison of GraphiT and baselines methods.}
We show our results in Table~\ref{tab:sup}. For smaller datasets such as MUTAG, PROTEINS or PTC, our Transformer without positional encoding performs reasonably well compared to GNNs, whereas for NCI1 and ZINC, incorporating structural information into the model is key to good performance. On all datasets, GraphiT is able to perform as well as or better than the baseline GNNs. In particular on ZINC, GraphiT outperforms all previous baseline methods by a large margin. For this, it seems that the factor $D^{-1/2}$ in~\eqref{eq:pos_forward} is important, allowing to capture more information about the graph structure. The answer to \textbf{Q1} is therefore positive.

\vs
\paragraph{Comparison of relative position encoding schemes.}
Here, we compare our transformer used with different relative positional encoding strategies, including adjacency (1-step RW kernel with $\gamma=1.0$ corresponding to a normalized adjacency matrix $D^{-1/2}A D^{-1/2}$) which is symmetric but not positive semi-definite, 2 and 3-step RW kernel with $\gamma=0.5$ and a diffusion kernel with $\beta=1$.
Unlike the vanilla transformer that works poorly on big datasets including NCI1 and ZINC, keeping all nodes communicate through our diffusion kernel positional encoding can still achieve performances close to encoding methods relying on local communications such as adjacency or p-step kernel encoding. Interestingly, our adjacency encoding, which could be seen as a variant of the neighborhood aggregation of node features used in GAT, is shown to be very effective on many datasets. In general, sparse local positional encoding seems to be useful for our prediction tasks, which tempers the answer to \textbf{Q1}. 

\vs
\paragraph{Comparison of structure encoding schemes in node features.} 
%
In this paragraph, we compare different ways of injecting graph structures to the vanilla transformer, including Laplacian PE~\citep{dwivedi2021generalization} and unsupervised GCKN-path features~\citep{Chen2020}. We observe that
incorporating topological structures directly into the features of input nodes is a very useful strategy for vanilla transformers. This yields significant performance improvement on almost all datasets except PROTEINS, by either using the Laplacian PE or GCKN-path features. Among them, GCKN features are observed to outperform Laplacian PE by a large margin, except on MUTAG and PROTEINS (third column of Table~\ref{tab:structural_encoding}).
A possible reason for this exception is that PROTEINS seems to be very specific, such that prediction models do not really benefit from encoding positional information. In particular, GCKN brings a pronounced performance boost on ZINC, suggesting the importance of encoding substructures like paths for prediction tasks on molecules.

\vs
\paragraph{Combining relative position encoding and structure encoding in node features.}
Table~\ref{tab:structural_encoding} reports the ablation study of the transformer used with or without structure encoding and coupled or not with a relative positional encoding. The results show that relative PE outperforms the topological Laplacian PE, suggesting a positive answer to \textbf{Q2}. However, using both simultaneously improves the results considerably, especially on ZINC. In fact, combining relative position encoding and a structure encoding scheme globally improves the performances. In particular, using the GCKN-path layer features works remarkably well for all datasets except PROTEINS. More precisely, we see that the combination of GCKN with the adjacent matrix PE yields the best results among the other combinations for MUTAG and NCI1. In addition, the GCKN coupled with the 3-step RW kernel achieves the second best performance for ZINC. The answer to \textbf{Q3} is therefore combining a structure encoding scheme in node features, such as GCKN, with relative positional encoding.
\vs
\paragraph{Discussion.}
Combining transformer and GCKN features results in substantial improvement over the simple sum or mean global pooling used in original GCKN models on ZINC dataset as shown in Table~\ref{tab:sup}, which suggests that transformers can be considered as a very effective method for aggregating local substructure features on large datasets at the cost of using much more parameters. This point of view has also been explored in the previous chapter, which introduces a different form of attention for sequence and text classification.
A potential limitation of GraphiT is its application to large graphs, 
as the complexity of the self-attention mechanism scales quadratically with the size of the input sequence. However, and as mentionned in Chapter~\ref{chapt:2_otke}, a recent line of work coined as efficient transformers alleviated these issues both in terms of memory and computational cost and we refer the reader to the following survey~\citep{tay2020efficient}.

\subsection{Visualizing attention scores for the Mutagenicity dataset.}
\label{subsec:visu}

We now show that the attention yields visualization mechanisms for detecting important graph motifs.

\vs
\paragraph{Mutagenicity.} In chemistry, a mutagen is a compound that causes genetic mutations. This important property is related to the molecular substructures of the compound. The Mutagenicity dataset~\citep{KKMMN2016} contains 4337 molecules to be classified as mutagen or not, and aims at better understanding which substructures cause mutagenicity. We train GraphiT with diffusion position encoding on this dataset with the aim to study whether important substructures can be detected. To this end, we feed the model with molecules of the dataset and collect the attention scores of each layer averaged by heads, as can be seen in Figure~\ref{fig:attentions} for the molecules of Figure~\ref{fig:molecules}. 

\begin{figure}
\begin{subfigure}{1.\textwidth}
  \centering
  \includegraphics[scale=0.4]{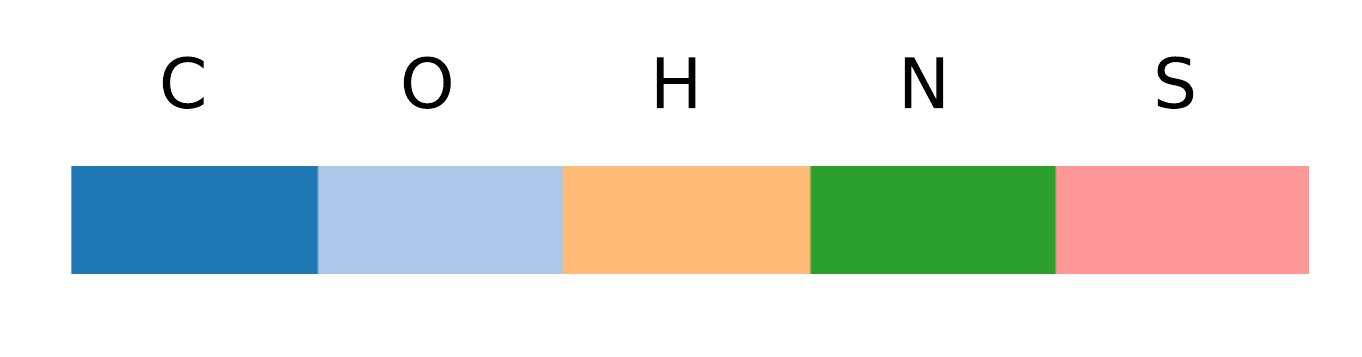}
\end{subfigure}%

\vspace{1cm}

\centering
\begin{subfigure}{1.\textwidth}
  \centering
  \includegraphics[scale=0.9]{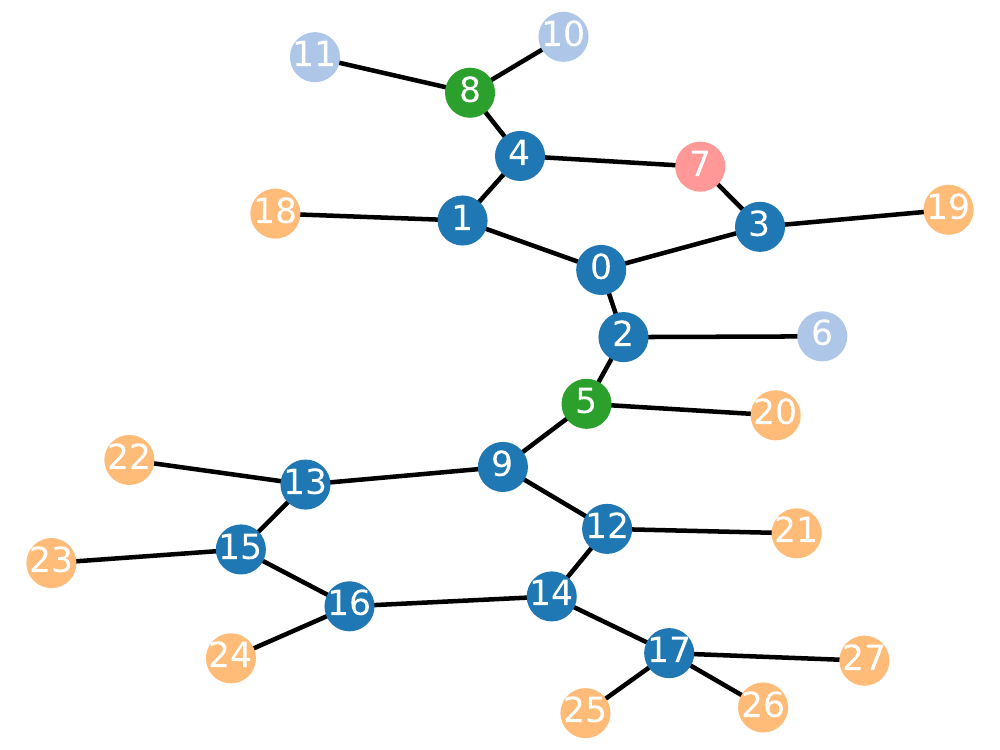}
  \caption{Nitrothiopheneamide-methylbenzene}
  \label{fig:NO2}
\end{subfigure}%

\vspace{1cm}

\begin{subfigure}{1.\textwidth}
  \centering
  \includegraphics[scale=0.9]{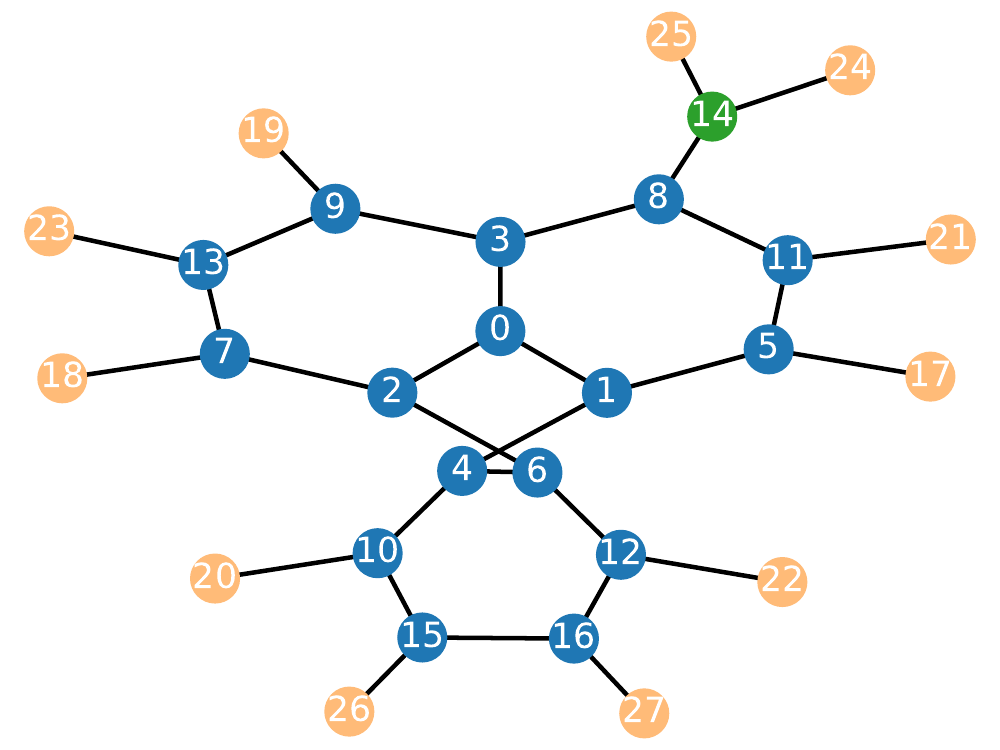}
  \caption{Aminofluoranthene}
  \label{fig:NH2}
\end{subfigure}
\caption{Examples of molecules from Mutagenicity correctly classified as mutagenetic by our model.}
\label{fig:molecules}
\end{figure}

\begin{figure}
\centering
\begin{subfigure}{.4\textwidth}
    \includegraphics[width=1\linewidth]{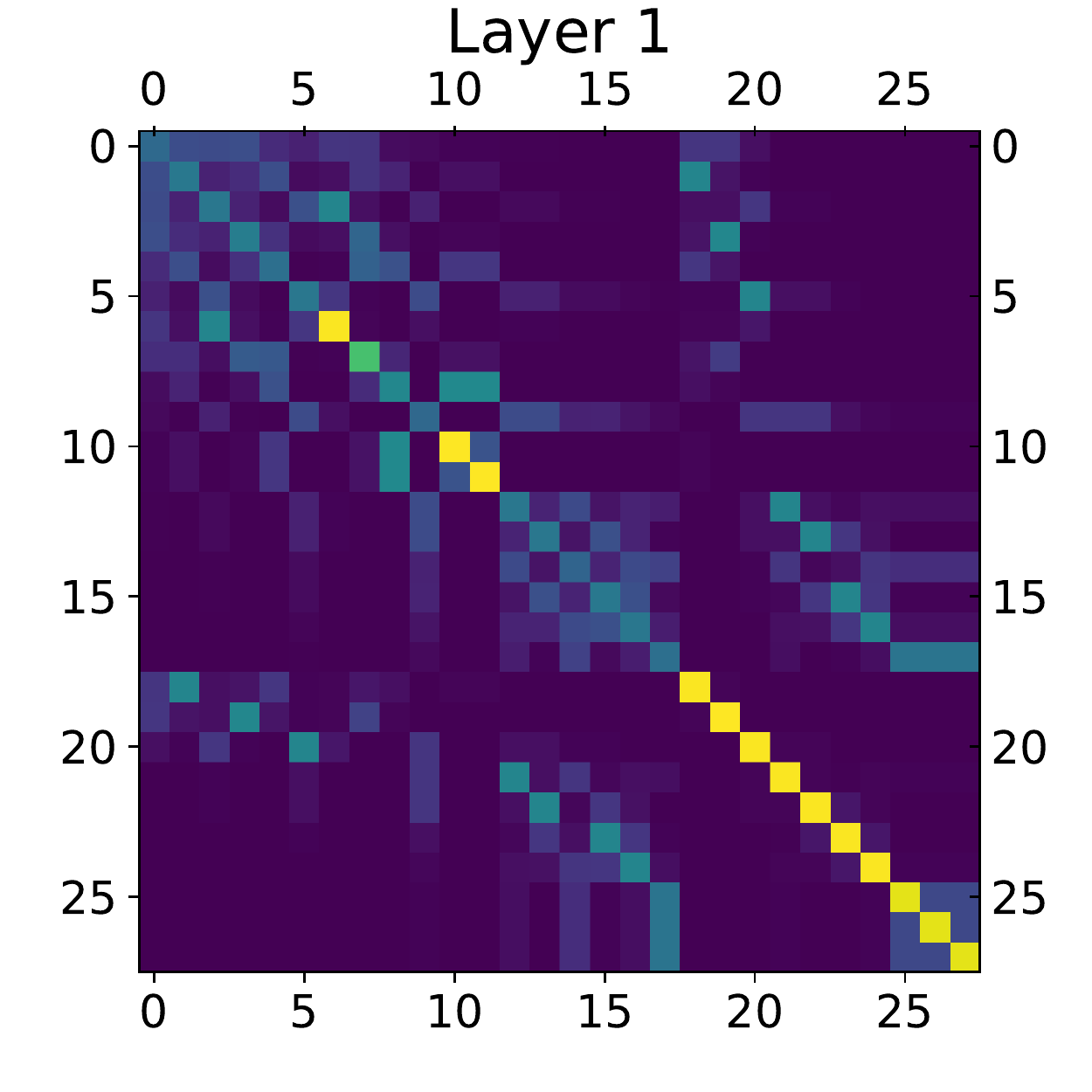}
\end{subfigure}
\begin{subfigure}{.4\textwidth}
    \includegraphics[width=1\linewidth]{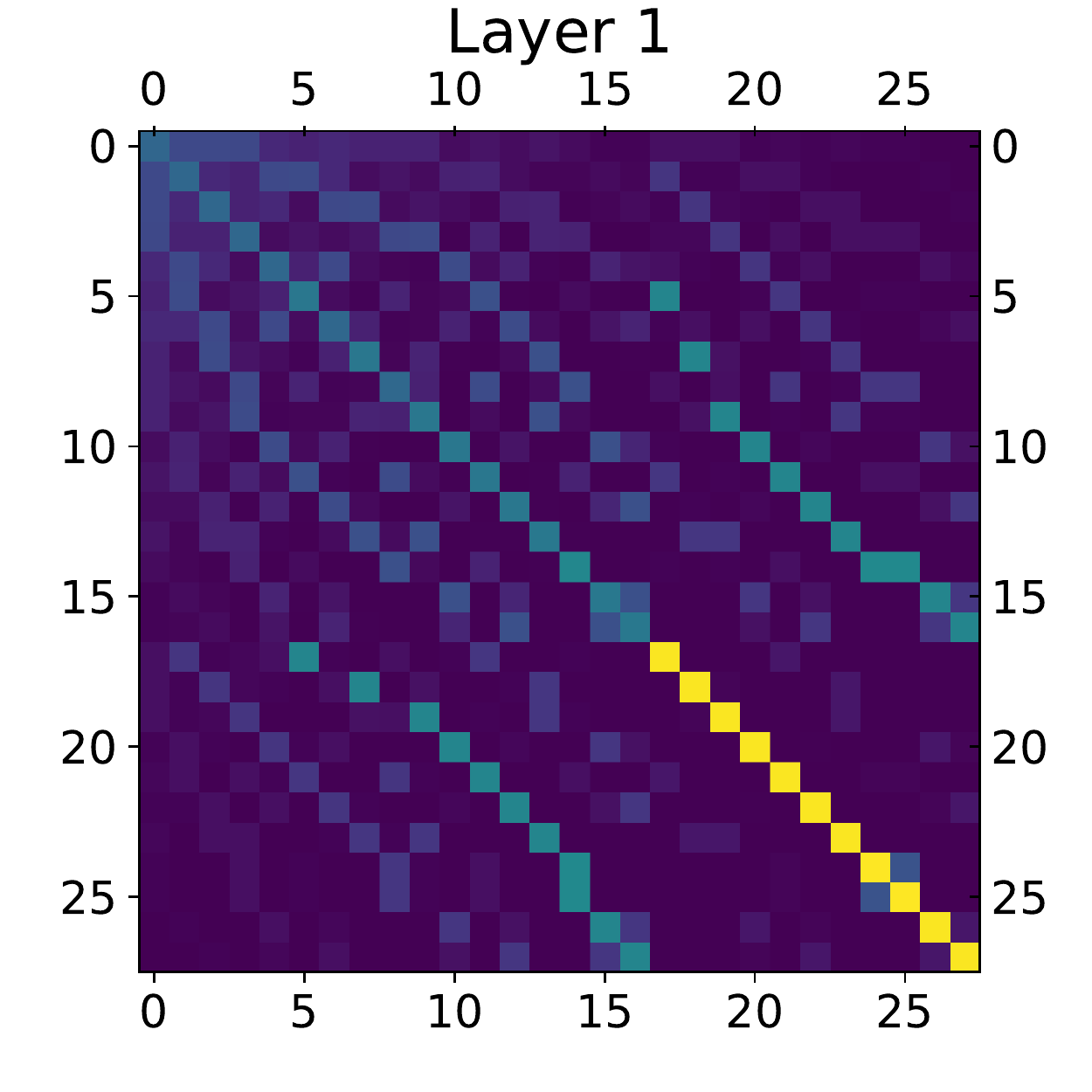}
\end{subfigure}
\begin{subfigure}{.4\textwidth}
    \includegraphics[width=1\linewidth]{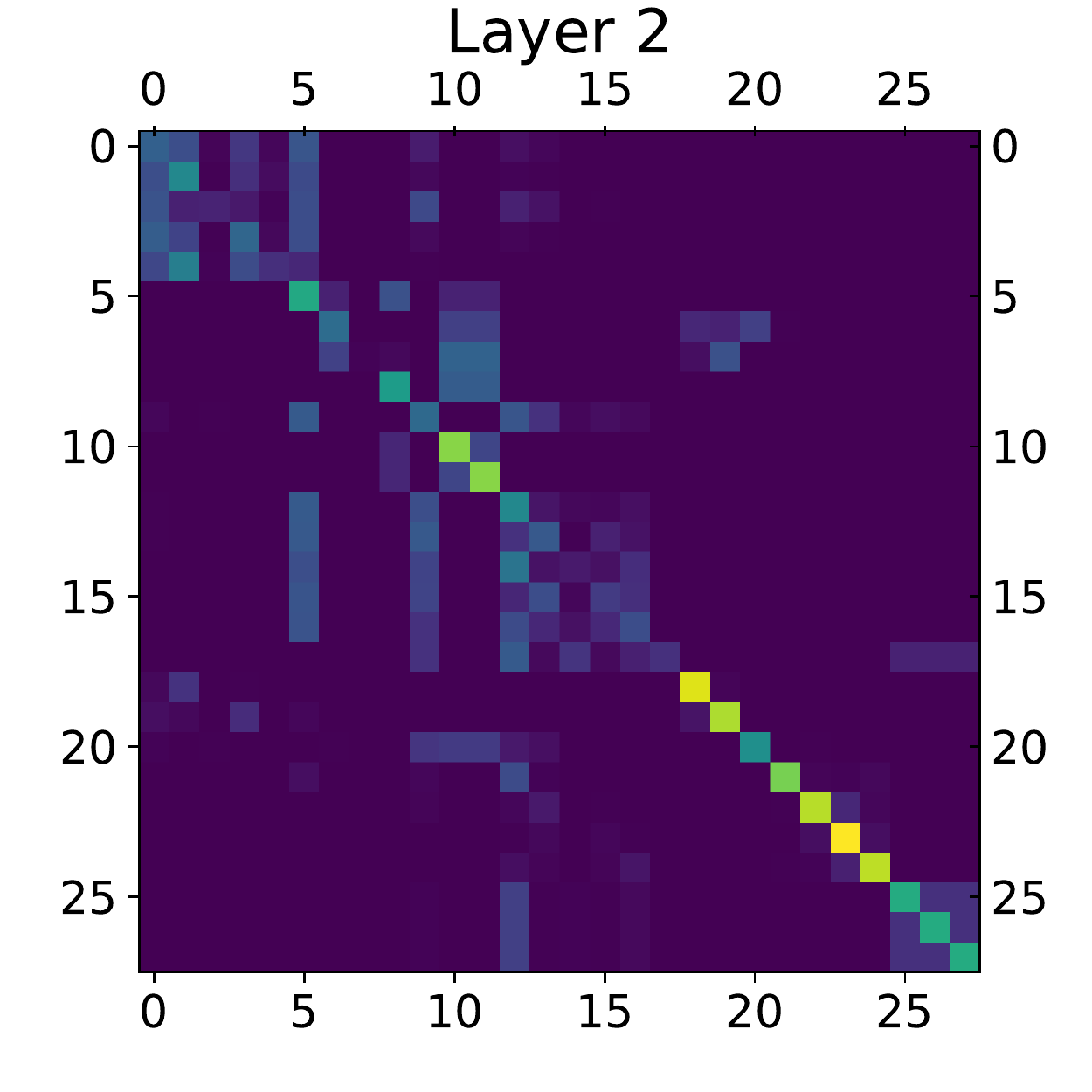}
\end{subfigure}
\begin{subfigure}{.4\textwidth}
    \includegraphics[width=1\linewidth]{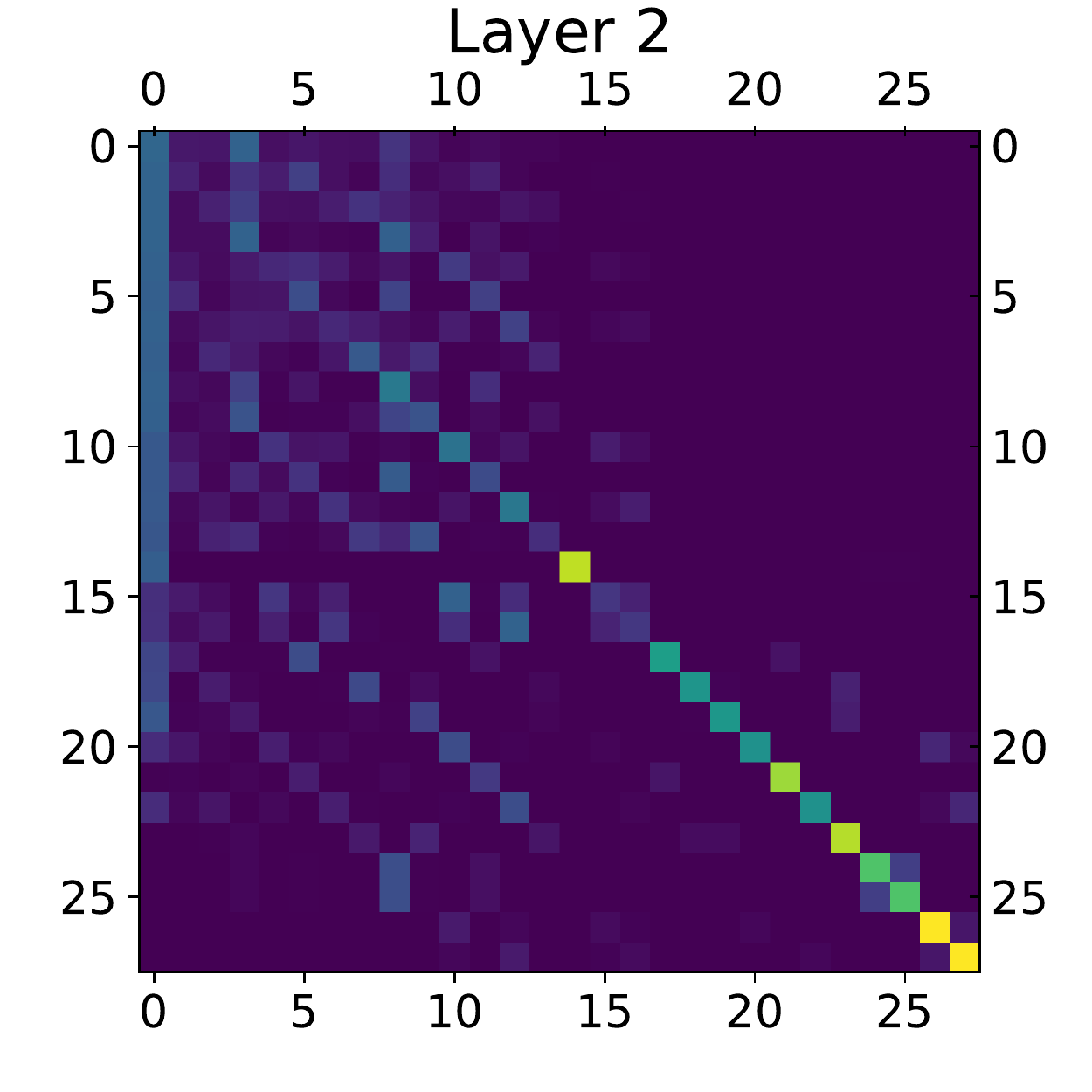}
\end{subfigure}
\begin{subfigure}{.4\textwidth}
    \includegraphics[width=1\linewidth]{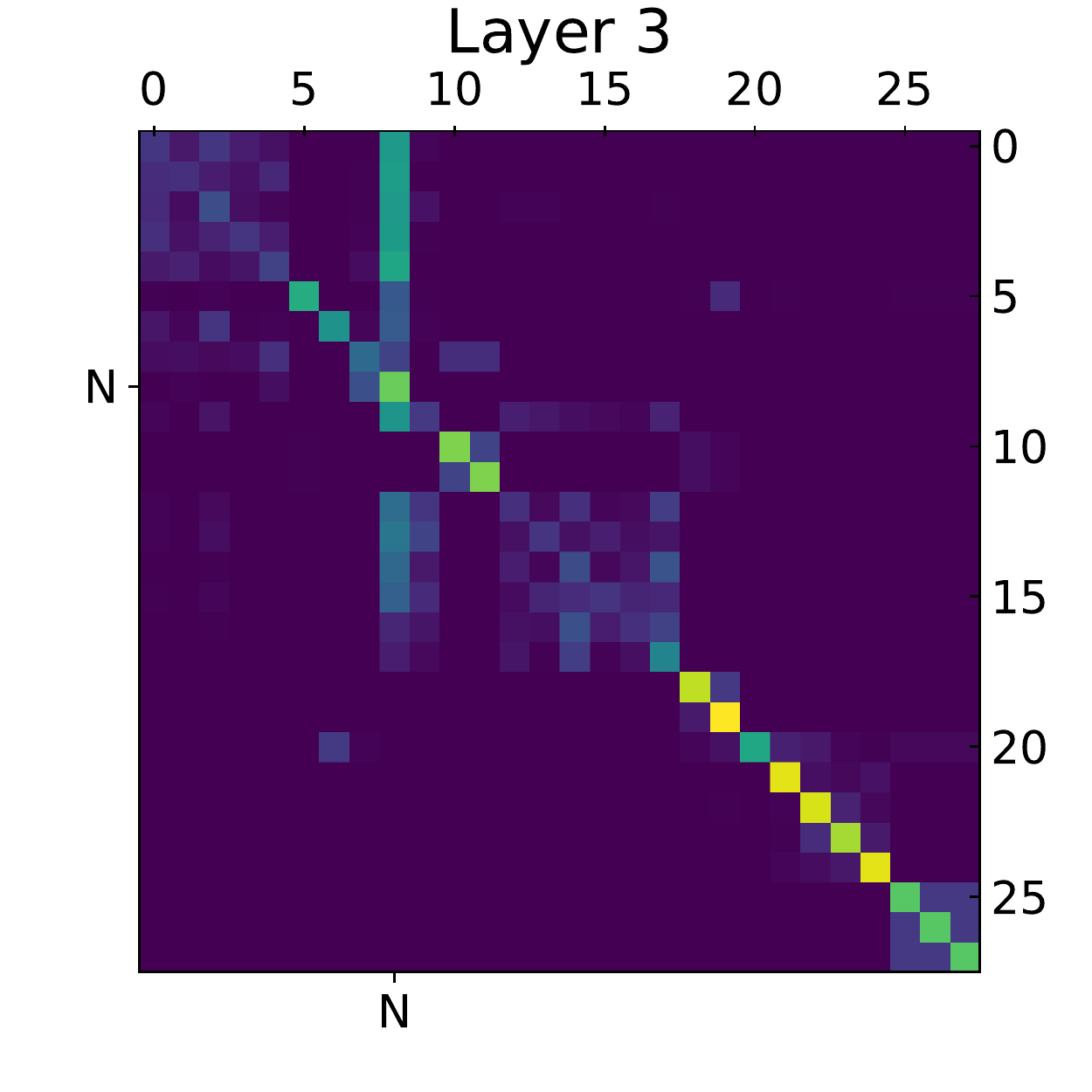}
\end{subfigure}
\begin{subfigure}{.4\textwidth}
    \includegraphics[width=1\linewidth]{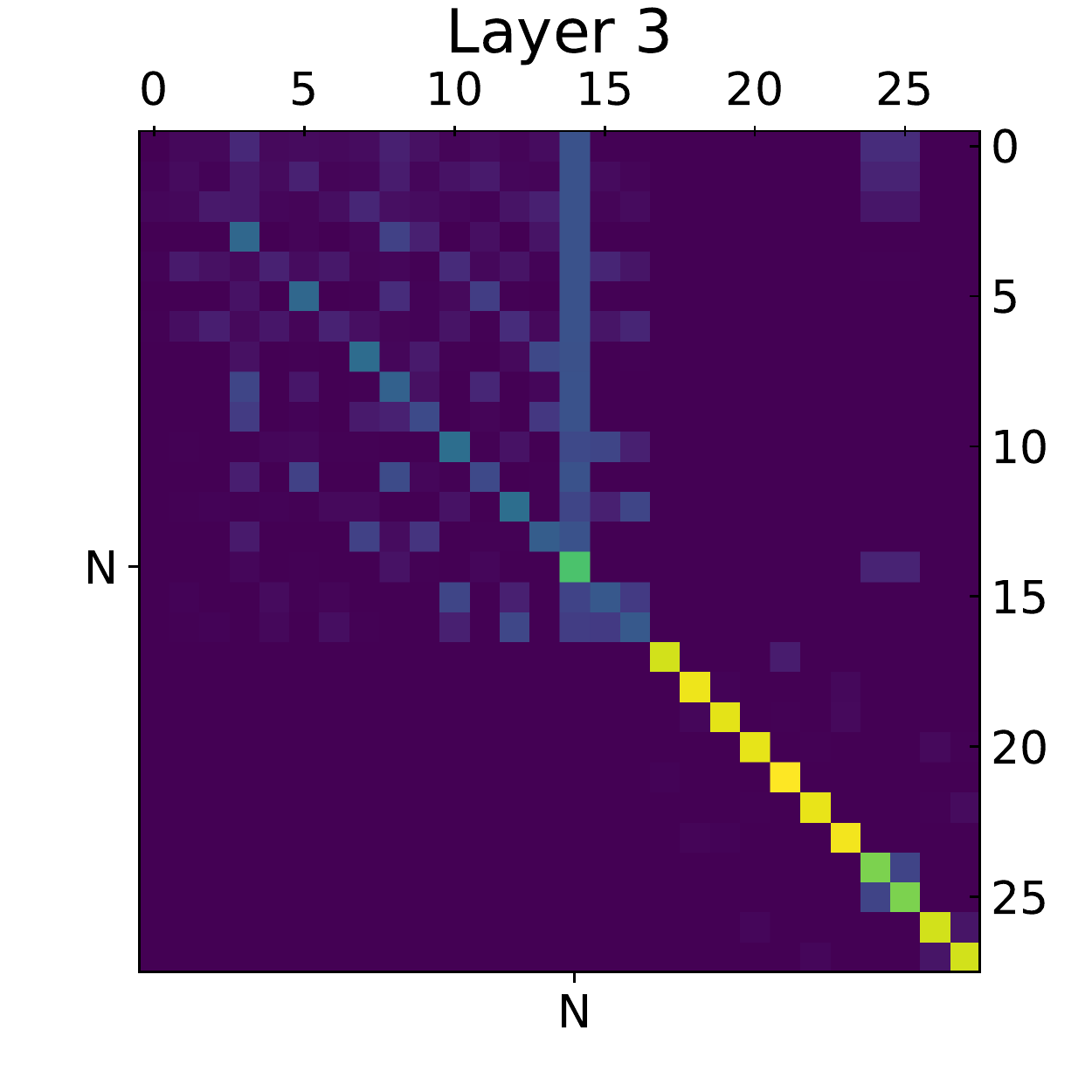}
\end{subfigure}

\caption{Attention scores averaged by heads for each layer of our trained model for the compounds in Figures~\ref{fig:NO2} (\textit{Left}) and~\ref{fig:NH2} (\textit{Right)}. \textit{Top Left}: diffusion kernel for~\ref{fig:NO2}. 
\textit{Bottom Left}: node $8$ (N of NO$_2$) is salient. 
\textit{Top Right}: diffusion kernel for~\ref{fig:NH2}. 
\textit{Bottom Right}: node $14$ (N of NH$_2$) is salient.}
\label{fig:attentions}
\end{figure}

\vs
\paragraph{Attention scores interpretation.} While the scores in the first layer are close to the diffusion kernel, the following layers get sparser. Since the attention matrix is multiplied on the right by the values, the coefficients of the node aggregation step for the n-th node is given by the n-th \textit{row} of the attention matrix. As a consequence, salient \textit{columns} suggest the presence of important nodes. After verification, for many samples of Mutagenicity fed to our model, salient atoms indeed represent important groups in chemistry, many of them being known for causing mutagenicity.

\paragraph{Nitrothiopheneamide-methylbenzene.} In compound~\ref{fig:NO2}, the N atom of the nitro group (NO$_2$) is salient. 
\ref{fig:NO2} was correctly classified as a mutagen by our model and the nitro group is indeed known for its mutagenetic properties~\citep{chung1996}. Note how information flows from O atoms to N in the first two layers and then, at the last layer, how every element of the carbon skeleton look at N, \textit{i.e} the NO$_2$ group. 

\paragraph{Aminofluoranthene.} We can apply the same reasoning for the amino group (NH$_2$) in compound~\ref{fig:NH2}~\citep{berry1985}. We were also able to identify long-range intramolecular hydrogen bounds such as between H and Cl in other samples.

\paragraph{1,2-Dibromo-3-Chloropropane (DBCP).} 
DBCP in Figure~\ref{fig:dbcp} was used as a soil fumigant in various countries. It was banned from use in the United Stated in 1979 after evidences that DBCP causes diseases, male sterility, or blindness, which can be instances of mutagenicity. In Figure~\ref{fig:attentions_dbcp}, attention focuses on the carbon skeleton and on the two Bromine (Br) atoms, the latter being indeed known for being associated with mutagenicity~\citep{LAG199473}.

\begin{figure}
\centering
\begin{subfigure}{1.0\textwidth}
\centering
\includegraphics[scale=0.4]{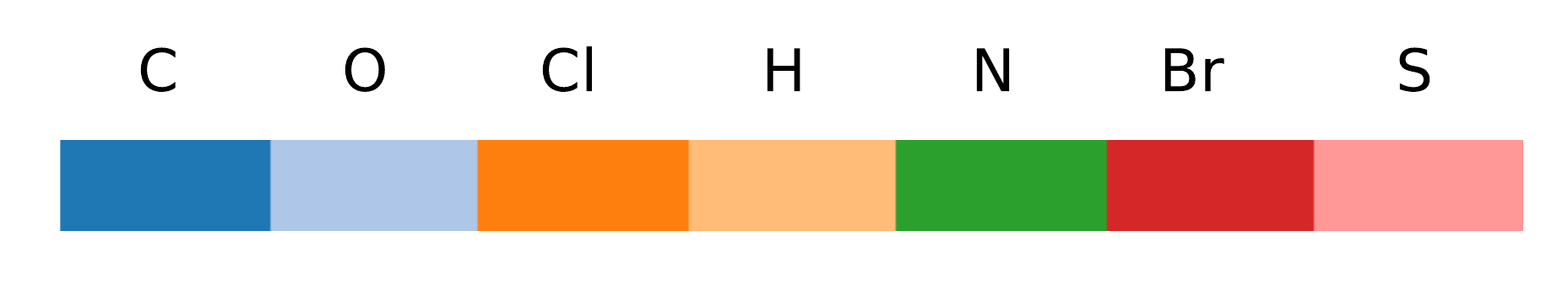}
\end{subfigure}

\vspace{1cm}

\begin{subfigure}{1.0\textwidth}
\centering
\includegraphics[scale=0.9]{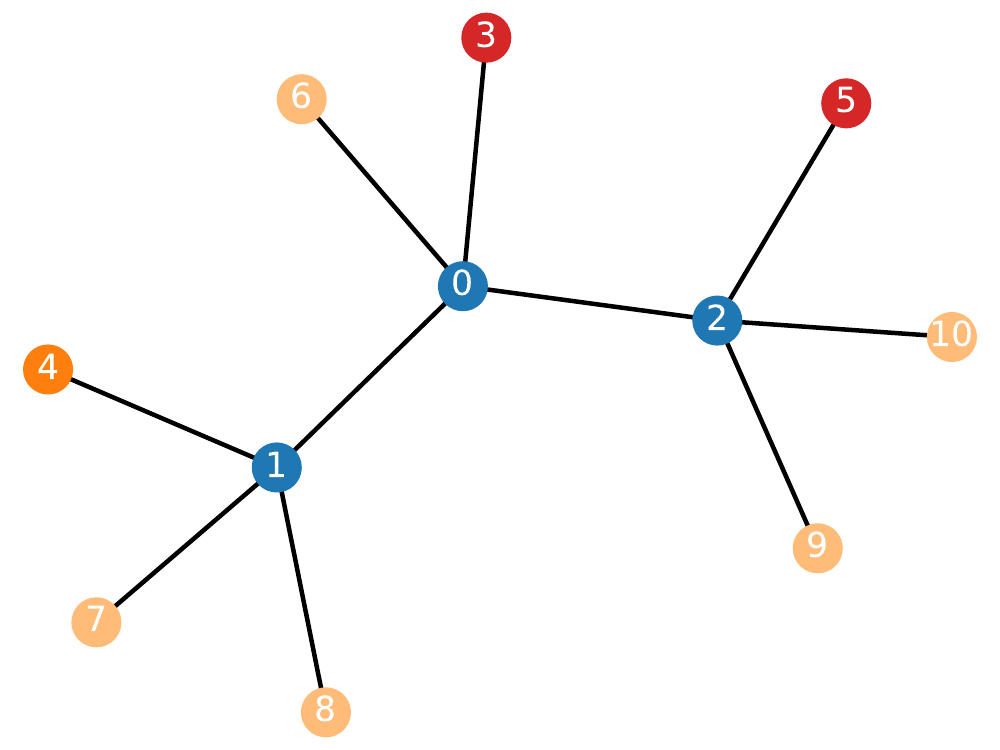}
\caption{1,2-Dibromo-3-Chloropropane.}
\label{fig:dbcp}
\end{subfigure}

\vspace{1cm}

\begin{subfigure}{1.0\textwidth}
\centering
\includegraphics[scale=0.9]{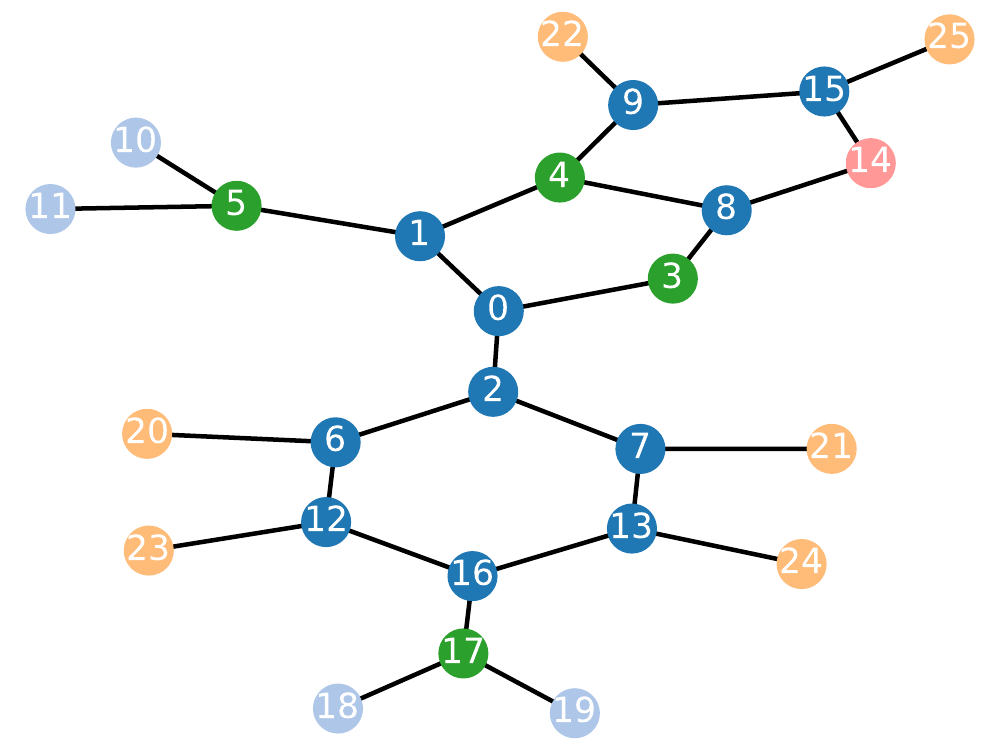}
\caption{Nitrobenzene-nitroimidazothiazole.}
\label{fig:2NO2}
\end{subfigure}
\end{figure}

\begin{figure}
\centering
\begin{subfigure}{0.4\textwidth}
    \includegraphics[width=1\linewidth]{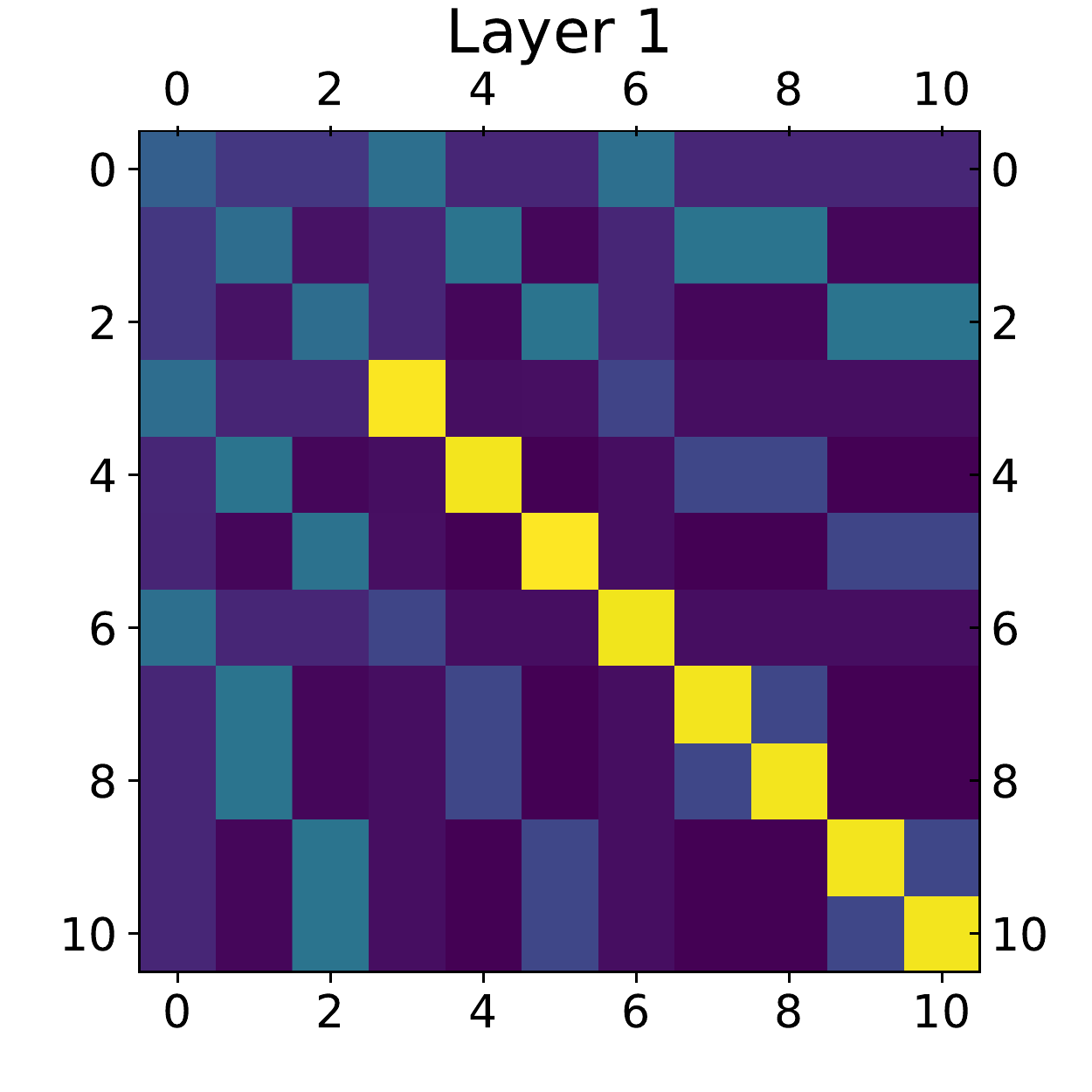}
\end{subfigure}
\begin{subfigure}{0.4\textwidth}
    \includegraphics[width=1\linewidth]{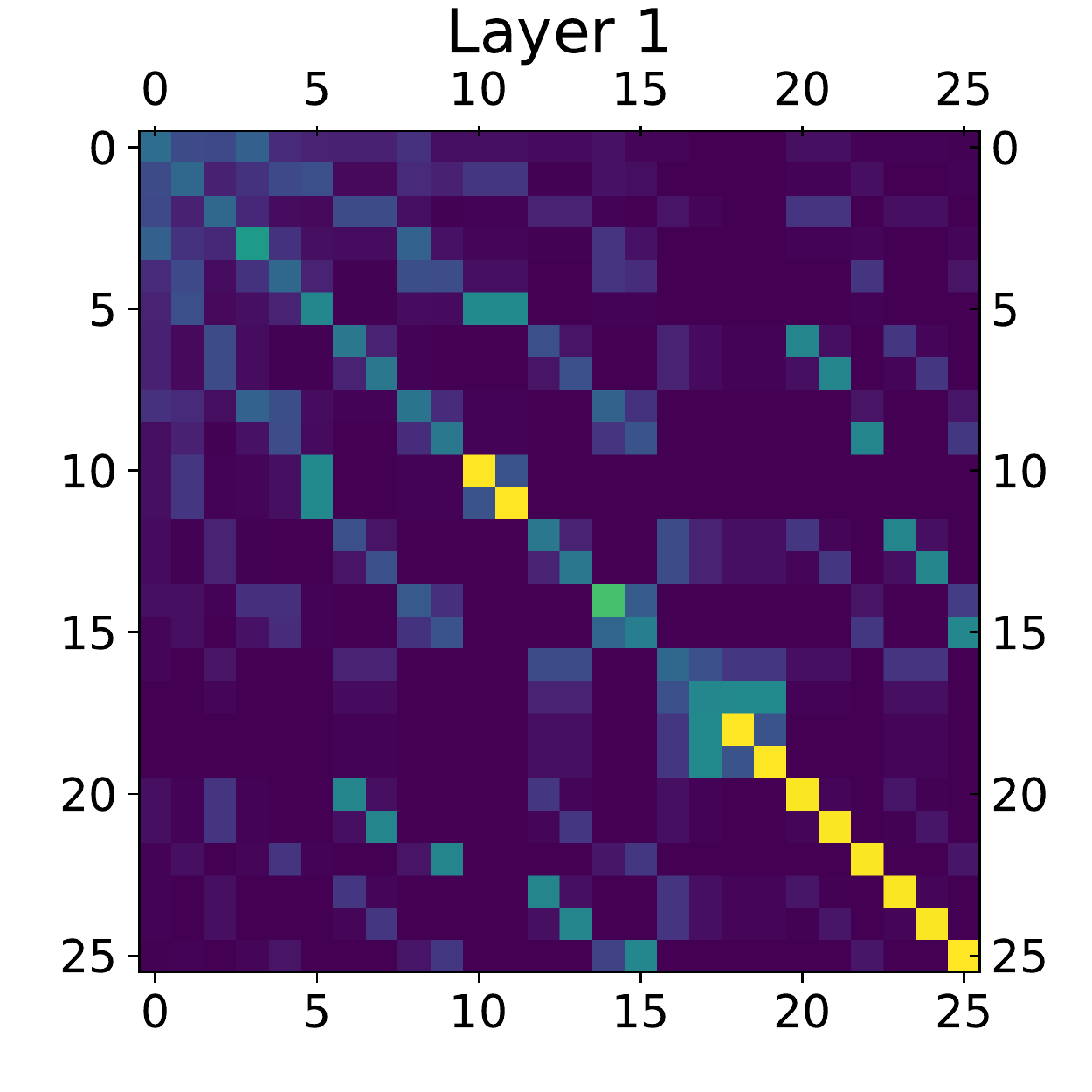}
\end{subfigure}
\begin{subfigure}{0.4\textwidth}
    \includegraphics[width=1\linewidth]{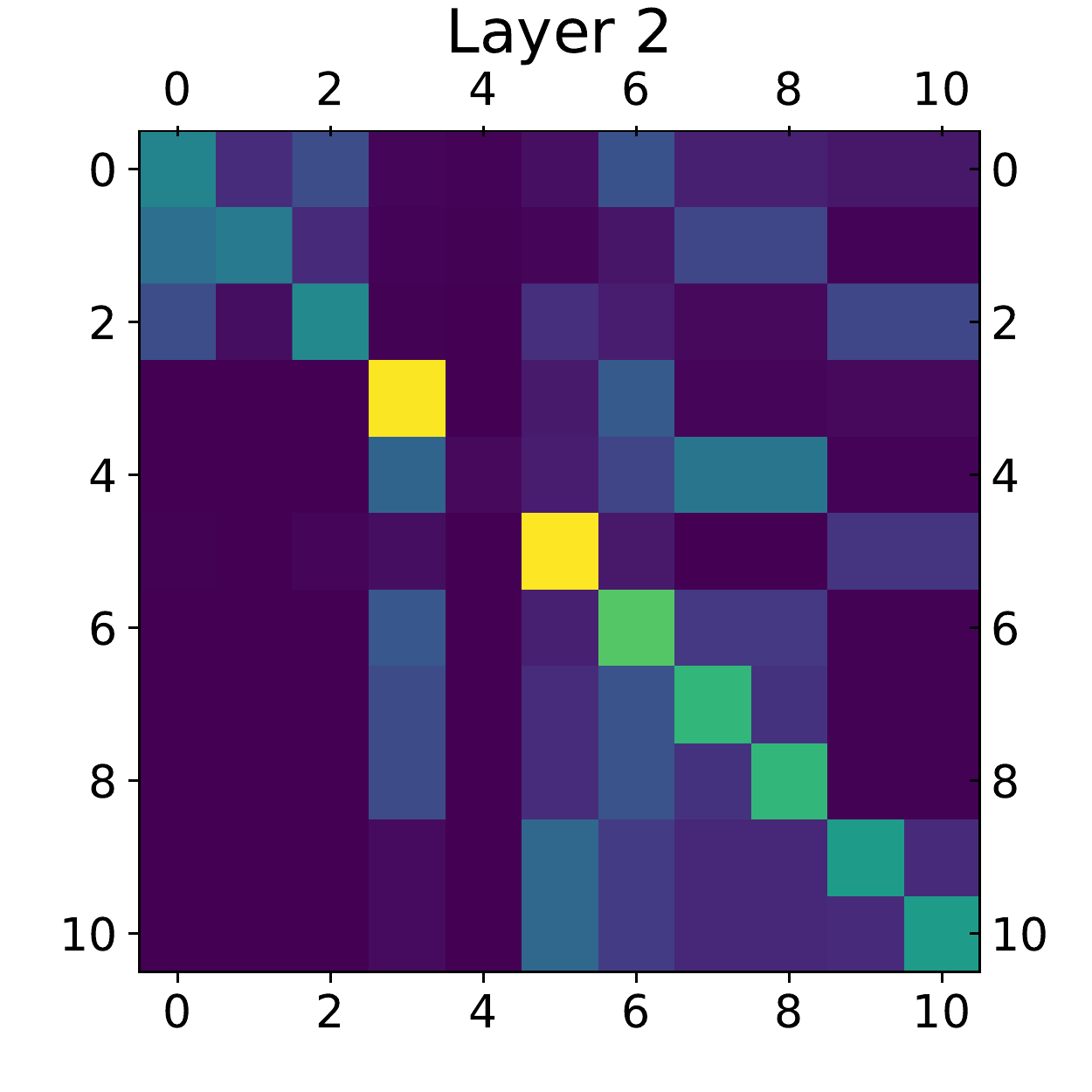}
\end{subfigure}
\begin{subfigure}{0.4\textwidth}
    \includegraphics[width = 1\linewidth]{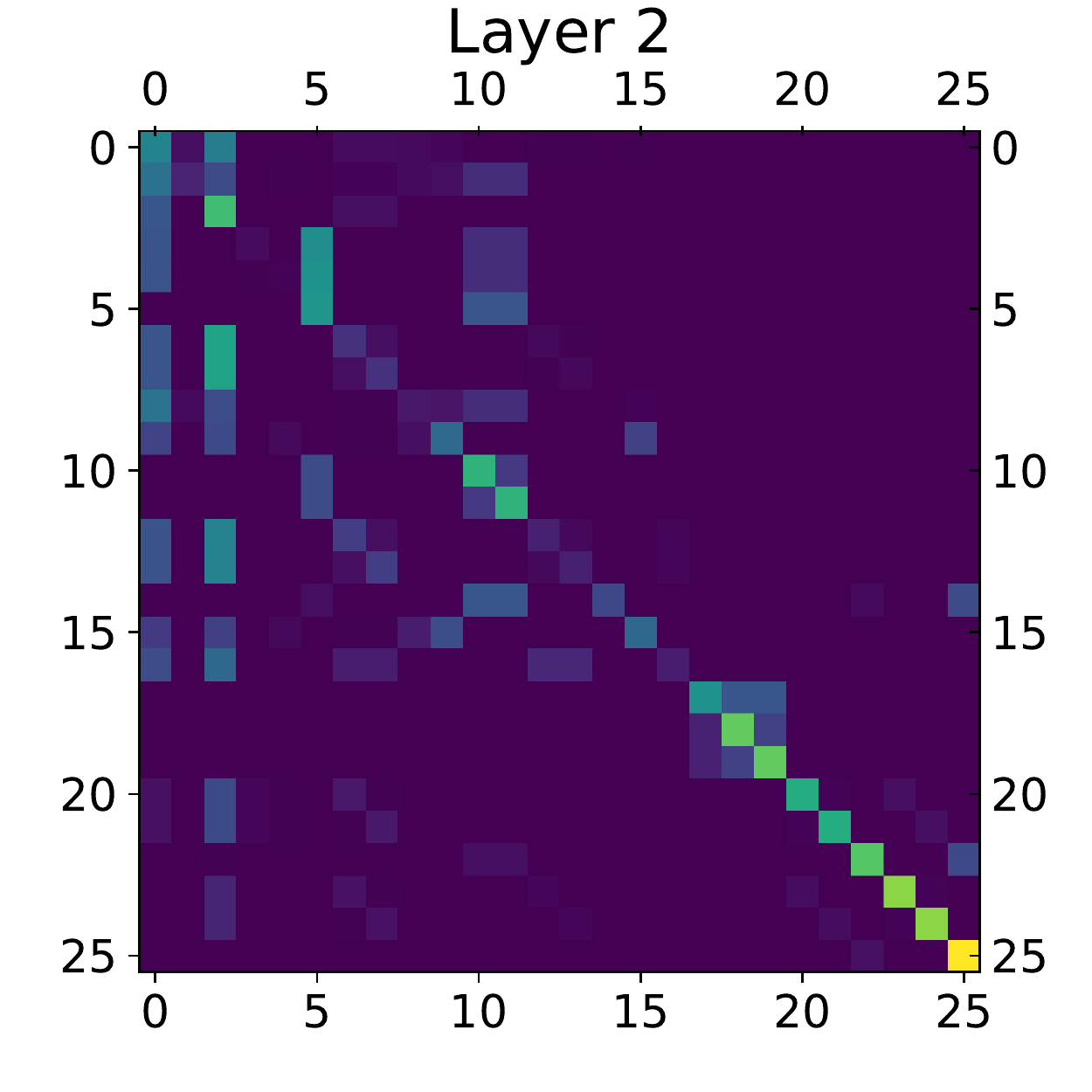}
\end{subfigure}
\begin{subfigure}{0.4\textwidth}
    \includegraphics[width = 1\linewidth]{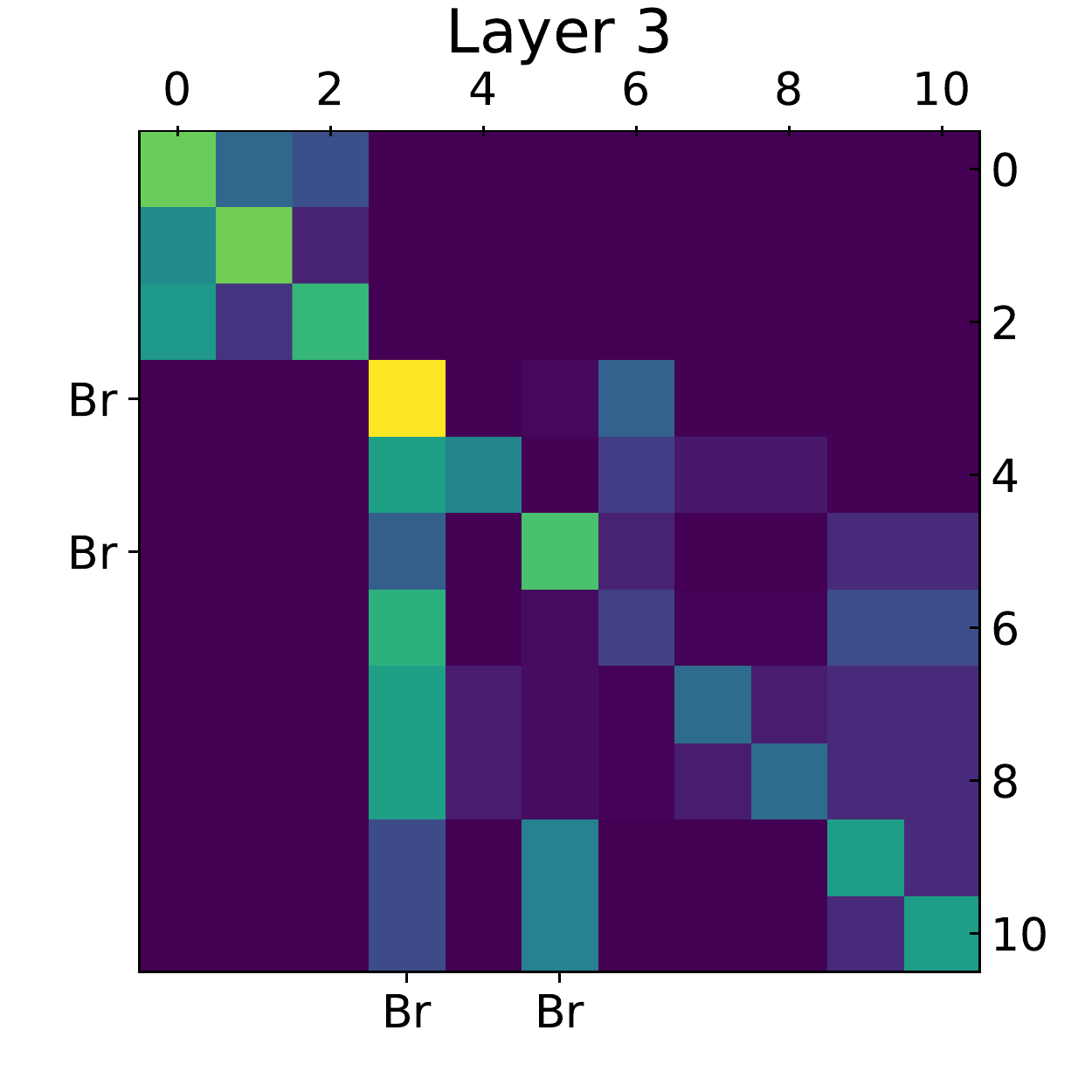}
\end{subfigure}
\begin{subfigure}{0.4\textwidth}
    \includegraphics[width = 1\linewidth]{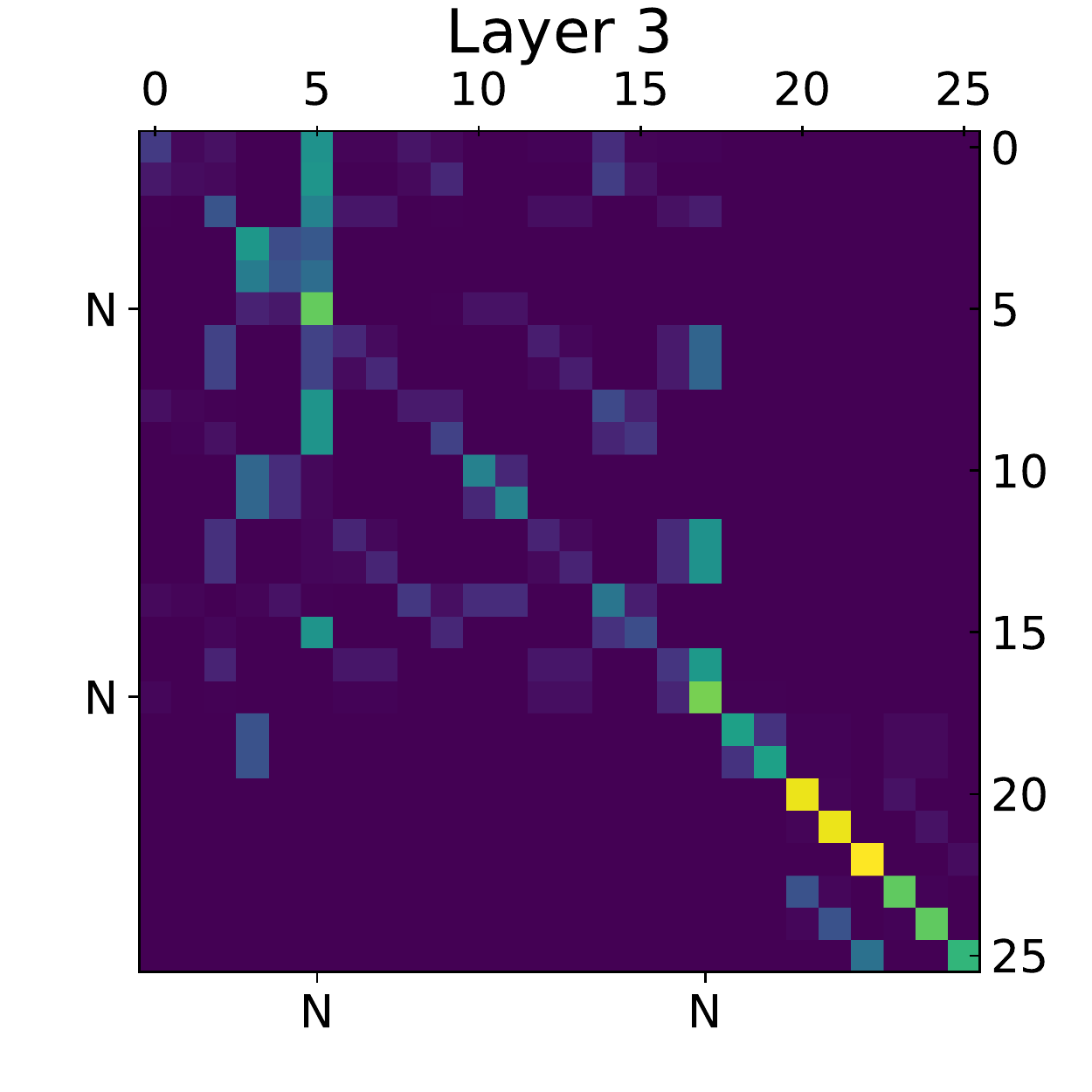}
\end{subfigure}

\caption{Attention scores averaged by heads for each layer of our trained model for the compounds in Figure~\ref{fig:dbcp} and Figure~\ref{fig:2NO2}. \textit{Top Left}: diffusion kernel for~\ref{fig:dbcp}. \textit{Bottom Left}: node $3$ and $5$ (Br) are salient. \textit{Top Right}: diffusion kernel for~\ref{fig:2NO2}. \textit{Bottom Right}: node $5$ and $17$ (N) are salient.}
\label{fig:attentions_dbcp}
\end{figure}

\paragraph{Nitrobenzene-nitroimidazothiazole.} This compound is shown in Figure~\ref{fig:2NO2}. As for compound~\ref{fig:NO2} in Section~\ref{subsec:visu}, our model puts emphasis on two nitro groups which are indeed known for inducing mutagenicity.

\paragraph{Triethylenemelamine.} Triethylenemelamine in Figure~\ref{fig:tmel} is a compound exhibiting mutagenic properties, and is used to induce cancer in experimental animal models. Our model focuses on the three nitrogen atoms of the aziridine groups which are themselves mutagenic compounds.

\begin{figure}
\centering
\begin{subfigure}{1.0\textwidth}
\centering
\includegraphics[scale=0.4]{chapters/graphit/figures/colorbar_appendix.pdf}
\end{subfigure}

\vspace{1cm}

\begin{subfigure}{1.0\textwidth}
\centering
\includegraphics[scale=0.9]{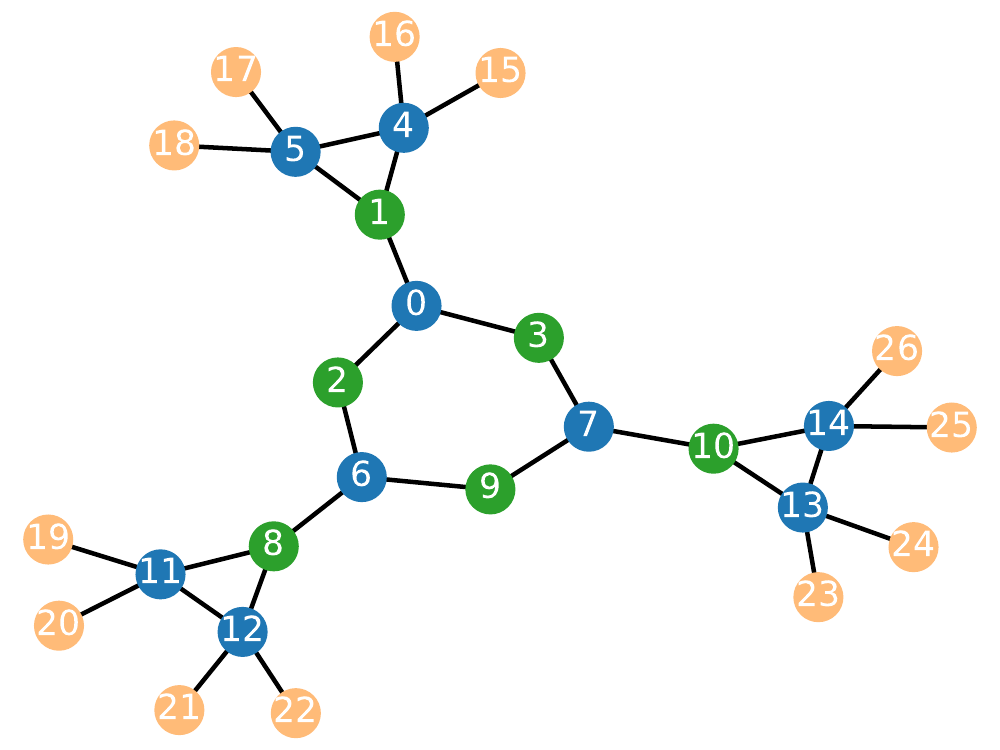}
\end{subfigure}
\caption{Triethylenemelamine.}
\label{fig:tmel}
\end{figure}
\begin{figure}
\centering
\begin{subfigure}{0.4\textwidth}
    \includegraphics[width = 1\linewidth]{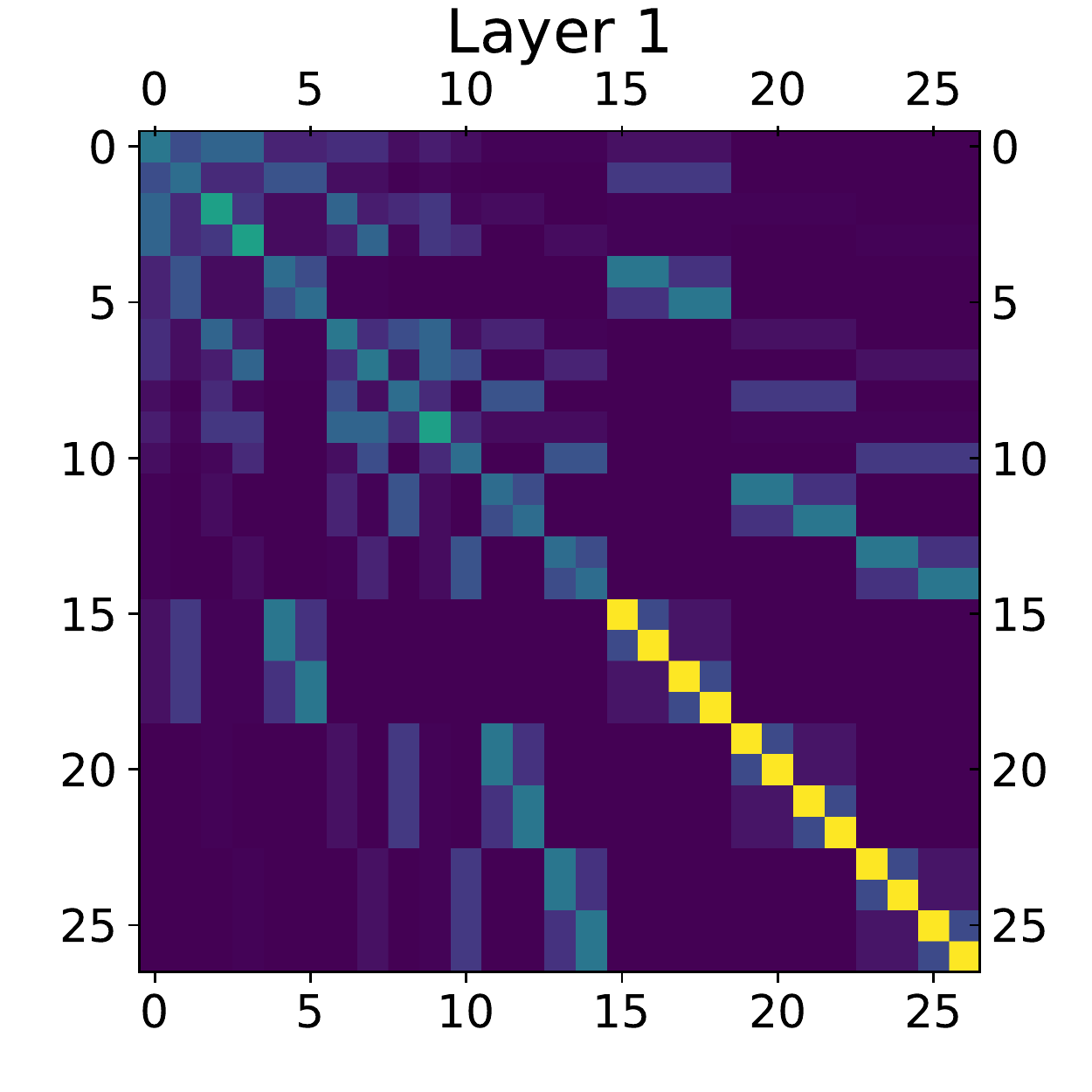}
\end{subfigure}
\begin{subfigure}{0.4\textwidth}
    \includegraphics[width = 1\linewidth]{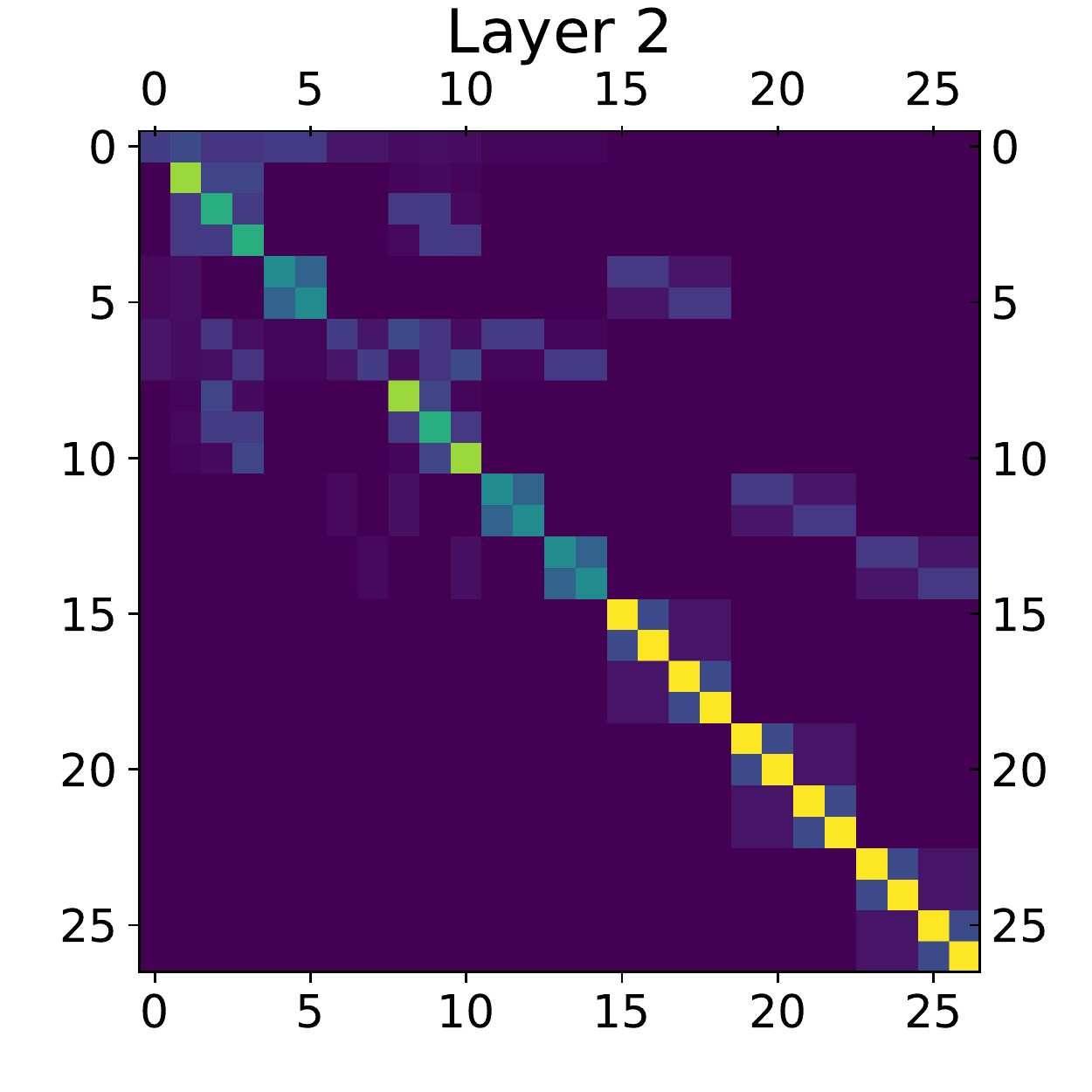}
\end{subfigure}

\begin{subfigure}{0.4\textwidth}
    \includegraphics[width = 1\linewidth]{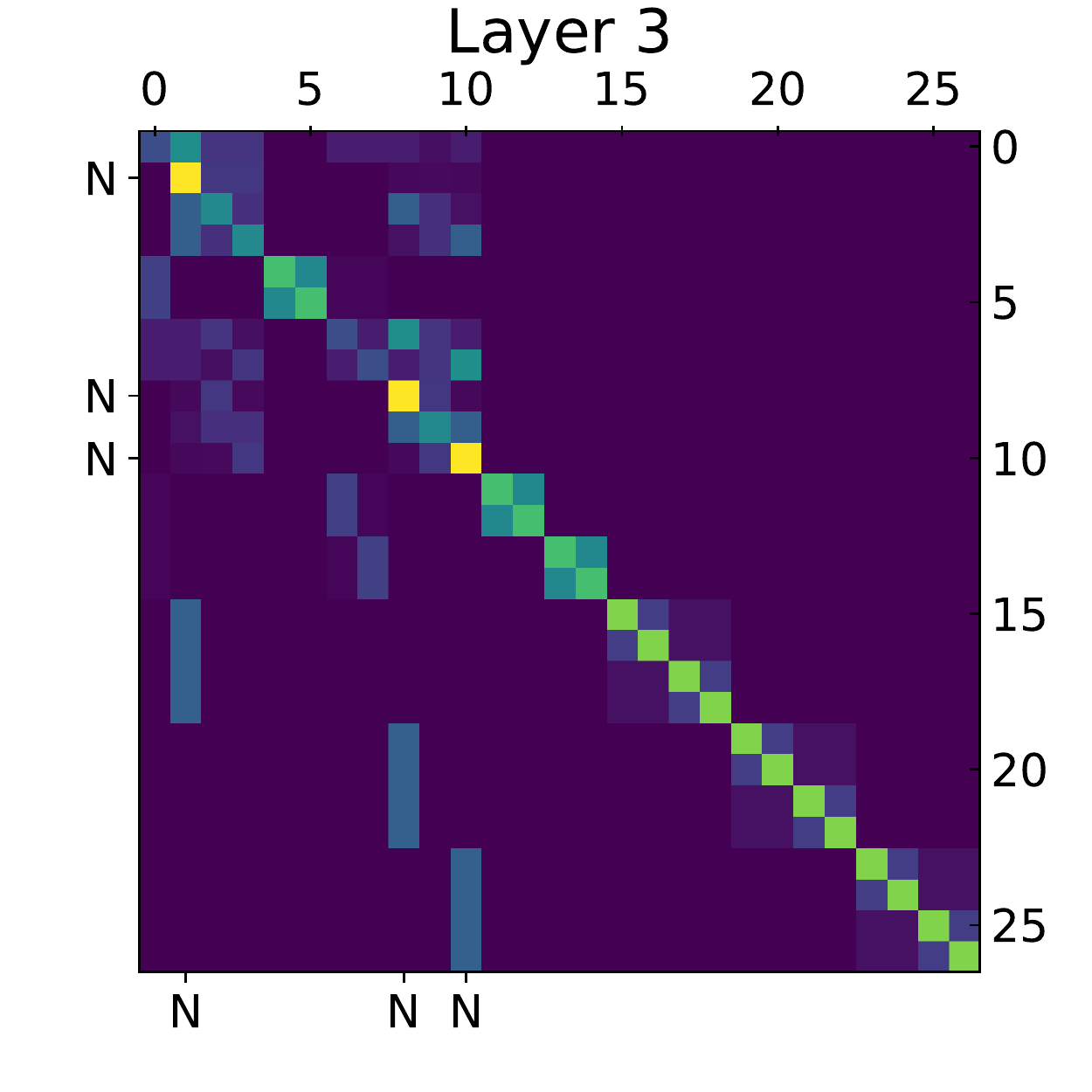}
\end{subfigure}

\caption{Attention scores averaged by heads for each layer of our trained model for the compound in Figure~\ref{fig:tmel}. \textit{Left}: diffusion kernel for~\ref{fig:tmel}.  \textit{Right}: node $1$, $8$, and $10$ (N) are salient.}
\label{fig:attentions_tmel}
\end{figure}

\clearpage
\section{Remark on Generalized Self-attention}

\cite{tsai2019transformer} shown that it is actually possible to rewrite self-attention as a kernel smoothing. Kernel smoothing is a solution to the problem of estimating the value of an unknown function for a new input, given the values of the function for some inputs. It consists in outputting a weighted average of the known values, with weights reflecting the similarity between the new sample and the known samples. If we consider the self-attention output for one member of the feature map:
\begin{align*}
        \text{Attention}(Q, K, V)_i & = \sum_{j=1}^n  \frac{ \exp\left( \frac{Q_iK^\top_j}{\sqrt{d_{out}}} \right)}{\sum_{j'=1}^n \exp \left( \frac{Q_iK^\top_{j'}}{\sqrt{d_{out}}}\right) } V_j \in \Real^{d_{\text{out}}} \\
        & = \sum_{j=1}^n  \frac{ k(i, j)}{\sum_{j'=1}^n k(i, j) } v(X_j) \in \Real^{d_{\text{out}}},
\end{align*}
with $Q_i = W_QX_i$, $K_j = W_KX_j$, $v(X_j) = W_VX_j$. When $k$ is a non-negative kernel function between a pair $(i, j)$, we indeed get a kernel smoothing. In the vanilla transformer~\citep{vaswani2017}, we have:
\begin{align*}
k(i, j) = \exp\left( \frac{Q_iK^\top_j}{\sqrt{d_{out}}} \right).
\end{align*} 
However, different choices for $k$ suggest different transformers architectures. In this chapter,
our similarity function was:
\begin{align*}
k(i, j) = \exp\left( \frac{Q_iK^\top_j}{\sqrt{d_{out}}} \right) \times K_r(i, j),
\end{align*}
with a first term based on the nodes contents similarity, and a second term based on the nodes structural similarity. \cite{tsai2019transformer} show that splitting the attention into such terms can be related to relative positional encoding~\cite{shaw2018selfattention}. The method introduced in the previous chapter can also be derived from this approach. Indeed, we replaced the attention weights (after the softmax) by the corresponding entries in the optimal transport plan $\mathbf{P}$ (which are normalized by design). More generally, inductive bias can be injected into the transformer architecture via an appropriate choice of the similarity function $k(i, j)$.

\section{Conclusion}
\label{sec:conclusion}
In this work, we show that using a transformer to aggregate local substructures with appropriate position encoding, GraphiT, is a very effective strategy for graph representation, and that attention scores allow simple model interpretation. Here, the role of inductive bias is twofold: first, the choice of the kernel for comparing the position of the nodes can depend on the properties we expect from the graphs. When long range interactions are expected, kernels on the graph with higher characteristic length should be more appropriate. Second, we also augment the node features with kernel methods for representing the extended neighborhood of a given node: this reflects the assumption that the meaning of a node depends on its neighborhood. On a different note, transformers for graphs also reflect a different inductive bias compared to GNNs: not only neighbors but all nodes should communicate.

One of the reasons for the success of transformers lies in their scaling properties: in language modeling for example, it has been shown that for an equal large amount of data, the transformer's loss follows a faster decreasing power law than Long-Short-Term-Memory networks when the size of the model increases~\citep{kaplan2020scaling}. 
We believe that an interesting future direction for this work would be to evaluate if GNNs and GraphiT also scale similarly in the context of large self-supervised pre-training, and can achieve a similar success as in natural language processing.





\newpage

\vspace*{0.3cm}
\begin{center}
   {\huge \textbf{Appendix}}
\end{center}
\vspace*{0.5cm}

\section{Experimental Details}
\label{sec:add_exp}
In this section, we provide implementation details and additional experimental results.

\subsection{General Details.}

\paragraph{Computing infrastructure.} Computations were done on a GPU cluster equipped with Tesla V100-16G and Tesla V100-32G. We have monitored precisely the entire computational cost of this research project (including preliminary experiments, designing early and final models, evaluating baselines, running final experiments), which was approximately 20k GPU hours. 

\paragraph{Position encoding and structure encoding parameters for all datasets.} 
$\gamma$ for p-step random walk kernels is fixed to 0.5 for both 2- and 3-step random walk kernels. $\beta$ for the diffusion kernel is fixed to 1.0. Regarding the structure encoding in node features, the dimension of Laplacian positional encoding is set to 8 for ZINC as suggested by~\cite{dwivedi2021generalization} and to 2 for graph classification datasets. The path size, the filter number and the bandwidth parameter of the unsupervised GCKN path features (used for structure encoding) are set to 8, 32 and 0.6 respectively for the ZINC dataset whereas the path size is set to 5 for graph classification datasets.

\paragraph{Other details.} For all instances of GraphiT, the hidden dimensions of the feed-forward hidden layer in each layer of the encoder are fixed to twice of the dimensions of the attention.

\subsection{Graph Classification Datasets}

Here, we provide experimental details for MUTAG, PROTEINS, PTC, NCI1 and Mutagenicity datasets.

\paragraph{Datasets.} These free datasets were collected by~\cite{KKMMN2016} for academic purpose. MUTAG consists in classifying molecules into mutagenetic or not. PROTEINS consists in classifying proteins into enzymes and non-enzymes. PTC consists in classifying molecules into carcinogenetic or not. NCI1 consists in classifying molecules into positive or negative to cell lung cancer. Mutagenicity consists in classifying compounds into mutagenetic or not.

\paragraph{Training splits.} For MUTAG, PROTEINS, PTC and NCI1, our outer splits correspond to the splits used in~\cite{xu2019powerful}. Our inner splits that divide their train splits into our train and validation splits are provided in our code. The error bars are computed as the standard deviation of the test accuracies across the 10 outer folds. On the other side, as the purpose of using Mutagenicity is model interpretation, we use simple train, validation and test splits respectively composed of $80\%$, $10\%$, and $10\%$ of the whole dataset. 

\paragraph{Hyperparameter choices.} For smaller datasets (MUTAG, PTC, PROTEINS), we use the search grids in Table~\ref{tab:gnn_small} for GNNs, the search grids in Table~\ref{tab:gckn} for GCKN and the search grids in Table~\ref{tab:transfo_small} for both GraphiT and the transformer model of~\cite{dwivedi2021generalization}. For bigger datasets (NCI1), we use the search grid in Table~\ref{tab:gnn_small} for GNNs, the search grids in Table~\ref{tab:gckn} for GCKN and the search grids in Table~\ref{tab:transfo_big} for both GraphiT and the transformer model of~\cite{dwivedi2021generalization}. The best models are selected based on the validation scores that are computed as the average of the validation scores over the last 50 epochs. Then, the selected model is retrained on train plus validation sets and the average of the test scores over the last 50 epochs is reported.
For Mutagenicity, we train a GraphiT model with 3 layers, 4 heads, 64 hidden dimensions. Initial learning rate was set to 1e-3, weight-decay was fixed to 1e-4 and structural information was encoded only via relative position encoding with the diffusion kernel.

\paragraph{Optimization.} We use the cross-entropy loss and Adam optimizer with batch size 128 for GNNs and 32 for both GraphiT and the transformer model of~\cite{dwivedi2021generalization}. For transformers, we do not observe significant improvement using warm-up strategies on these classification tasks. Thus, we simply follow a scheduler that halves the learning rate after every 50 epochs, as for GNNs. All models are trained for 300 epochs.

\begin{table}
    \small
    \centering
    \caption{Parameter grid for GNNs trained on MUTAG, PROTEINS, PTC, NCI1.}
    \begin{tabular}{l|c} \toprule
    {Parameter} & {Grid} \\
    \midrule
    {Layers} & [1, 2, 3] \\
    {Hidden dimension} & [64, 128, 256] \\
    {Global pooling} & [sum, max, mean]\\
    {Learning rate} & [0.1, 0.01, 0.001] \\
    {Weight decay} & [0.1, 0.01, 0.001] \\
    \bottomrule
    \end{tabular}
    \label{tab:gnn_small}
\end{table}

\begin{table}
    \small
    \centering
    \caption{Parameter grid for GCKN trained on MUTAG, PROTEINS, PTC, NCI1.}
    \begin{tabular}{l|c} \toprule
    {Parameter} & {Grid} \\
    \midrule
    {Path size} & [3, 5, 7] \\
    {Hidden dimension} & [16, 32, 64] \\
    {Sigma} & [0.5] \\
    {Global pooling} & [sum, max, mean]\\
    {Learning rate} & [0.001] \\
    {Weight decay} & [0.01, 0.001, 0.0001] \\
    \bottomrule
    \end{tabular}
    \label{tab:gckn}
\end{table}

\begin{table}
    \small
    \centering
    \caption{Parameter grid for transformers trained on MUTAG, PROTEINS, PTC.}
    \begin{tabular}{l|c} \toprule
    {Parameter} & {Grid} \\
    \midrule
    {Layers} & [1, 2, 3] \\
    {Hidden dimension} & [32, 64, 128] \\
    {Heads} & [1, 4, 8]\\
    {Learning rate} & [0.001] \\
    {Weight decay} & [0.01, 0.001, 0.0001] \\
    \bottomrule
    \end{tabular}
    \label{tab:transfo_small}
\end{table}

\begin{table}
    \small
    \centering
    \caption{Parameter grid for transformers trained on NCI1.}
    \begin{tabular}{l|c} \toprule
    {Parameter} & {Grid} \\
    \midrule
    {Layers} & [2, 3, 4] \\
    {Hidden dimension} & [64, 128, 256] \\
    {Heads} & [1, 4, 8]\\
    {Learning rate} & [0.001] \\
    {Weight decay} & [0.01, 0.001, 0.0001] \\
    \bottomrule
    \end{tabular}
    \label{tab:transfo_big}
\end{table}



\subsection{Graph Regression Dataset}
Here, we present experimental details and additional experimental results for ZINC dataset.

\paragraph{Dataset.} 

ZINC is a free dataset consisting of 250k compounds and the task is to predict the constrained solubility of the compounds, formulated as a graph property regression problem. This problem is crucial for designing generative models for molecules. Each molecular graph in ZINC has the type of heavy atoms as node features (represented by a binary vector using one-hot encoding) and the type of bonds between atoms as edge features. In order to focus on the exploitation of the topological structures of the graphs, we omitted the edge features in our experiments. They could possibly be incorporated into GraphiT through the approach of~\cite{dwivedi2021generalization} or considering different kernels on graph for each edge bond type, which is left for future work.

\paragraph{Training splits.} 

Following~\cite{dwivedi2021generalization}, we use a subset of the ZINC dataset composed of respectively 10k, 1k and 1k graphs for train, validation and test split. The error bars are computed as the standard deviation of test accuracies across 4 different runs.

\paragraph{Hyperparameter choice.}

In order to make fair comparisons with the most relevant work \cite{dwivedi2021generalization}, we use the same number of layers, number of heads and hidden dimensions, namely 10, 8 and 64. The number of model parameters for our transformers is only $2/3$ of that of~\cite{dwivedi2021generalization} as we use a symmetric variant for the attention score function in~\eqref{eq:pos_attention}. Regarding the GNNs, we use the values reported in~\cite{dwivedi2021generalization} for GIN, GAT and GCN. In addition, we use the same grid to train the MF model~\citep{duvenaud2015convolutional}, \textit{i.e.}, a learning rate starting at 0.001, two numbers of layers (4 and 16) and the hidden dimension equal to 145. Regarding the GCKN-subtree model, we use the same hidden dimension as GNNs, that is 145. We fix the bandwidth parameter to 0.5 and the path size is fixed to 10. We select the global pooling form max, mean, sum and weight decay from 0.001, 0.0001 and 1e-5, similar to the search grid used in~\cite{dwivedi2021generalization}.

\paragraph{Optimization.} 

Following~\cite{dwivedi2021generalization}, we use the L1 loss and the Adam optimization method with batch size of 128 for training. Regarding the scheduling of the learning rate, we observe that a standard warm-up strategy used in~\cite{vaswani2017} leads to more stable convergence for larger models (hidden dimension equal to 128). We therefore adopt this strategy throughout the experiments with a warm-up of 2000 iteration steps. 

\paragraph{Additional results.}
Besides the relatively small models presented in Table~\ref{tab:sup}, we also show in Table~\ref{tab:zinc} the performance of larger models with 128 hidden dimensions. Increasing the number of hidden dimensions generally results in boosting performance, especially for transformer variants combined with Laplacian positional encoding in node features. While sparse local positional encoding is shown to be more useful compared to the positional encoding with a diffusion kernel, we show here that a variant of diffusion kernel positional encoding outperforms all other sparse local positional encoding schemes. Since the skip-connection in transformers already assigns some weight to each node itself and the diagonal of our kernels on graphs also have important weight, we considered setting these diagonal to zero in order to remove the effect of the attention scores on self-loop. This modification leads to considerable performance improvement on longer range relative positional encoding schemes, especially for the transformer with diffusion kernel positional encoding combined with GCKN in node features, which results in best performance.

\begin{table}
    \centering
    \caption{Mean absolute error for regression problem ZINC. The results are computed from 4 different runs following~\citep{dwivedi2021generalization}.}
    \label{tab:zinc}
    \begin{tabular}{l|c|c|c}\toprule
        \multirow{2}{*}{Relative PE in attention score} & \multicolumn{3}{c}{Structural encoding in node features} \\
         & None & LapPE & GCKN (p=8,d=32) \\ \midrule
        \cite{dwivedi2021generalization} &  {0.3585$\pm$0.0144} & 0.3233$\pm$0.0126  & - \\ \midrule
        Transformers with d=64 \\ \midrule
        T & {0.6964$\pm$0.0067} & {0.5065$\pm$0.0025} & {0.2741$\pm$0.0112} \\
        {T + Adj PE} & 0.2432$\pm$0.0046 & \textbf{0.2020$\pm$0.0112} & {0.2106$\pm$0.0104} \\
        {T + 2-step RW kernel} & 0.2429$\pm$0.0096 & 0.2272$\pm$0.0303 & 0.2133$\pm$0.0161  \\
        {T + 3-step RW kernel} & 0.2435$\pm$0.0111 & 0.2099$\pm$0.0027 & 0.2028$\pm$0.0113  \\
        T + diffusion & 0.2548$\pm$0.0102 
        & 0.2209$\pm$0.0038 & 0.2180$\pm$0.0055 \\ \midrule
        Setting diagonal to zero, d=64 \\ \midrule
        T & 0.7041$\pm$0.0044 & 0.5123$\pm$0.0232 & 0.2735$\pm$0.0046 \\
        T + 2-step & 0.2427$\pm$0.0053 & 0.2108$\pm$0.0072 & 0.2176$\pm$0.0430 \\
        T + 3-step & 0.2451$\pm$0.0043 & 0.2054$\pm$0.0072 & 0.1986$\pm$0.0091 \\
        T + diffusion & 0.2468$\pm$0.0061 & 0.2027$\pm$0.0084 &  \textbf{0.1967$\pm$0.0023} \\
        \midrule
        Larger models with d=128 \\ \midrule
        T & 0.7044$\pm$0.0061 & 0.4965$\pm$0.0338 & 0.2776$\pm$0.0084 \\
        {T + Adj PE} & 0.2310$\pm$0.0072 & \textbf{0.1911$\pm$0.0094} & 0.2055$\pm$0.0062 \\
        {T + 2-step RW kernel} & 0.2759$\pm$0.0735 
        & 0.2005$\pm$0.0064 & 0.2136$\pm$0.0062 \\
        {T + 3-step RW kernel} & 0.2501$\pm$0.0328 
        & 0.2044$\pm$0.0058 & 0.2128$\pm$0.0069 \\
        T + diffusion & 0.2371$\pm$0.0040 
        & 0.2116$\pm$0.0103 & 0.2238$\pm$0.0068  \\ \midrule
        Setting diagonal to zero, d=128 \\ \midrule
        T & 0.7044$\pm$0.0061 & 0.4964$\pm$0.0340 & 0.2721$\pm$0.0099 \\
        T + 2-step & 0.2348$\pm$0.0010 & 0.2012$\pm$0.0038 & 0.2031$\pm$0.0083 \\
        T + 3-step & 0.2402$\pm$0.0056 & 0.2031$\pm$0.0076 & 0.2019$\pm$0.0084 \\
        T + diffusion & 0.2351$\pm$0.0121 & \textbf{0.1985$\pm$0.0032} & 0.2019$\pm$0.0018 \\ 
        \bottomrule
    \end{tabular}
\end{table}

\cleardoublepage

\chapter{Sample Screening in Empirical Risk Minimization}
\label{chapt:4_screening}
In this chapter, we design simple screening tests to automatically discard data samples in empirical risk minimization without losing optimization guarantees. We derive loss functions that produce dual objectives with a sparse solution. We also show how to regularize convex losses to ensure such a dual sparsity-inducing property, and propose a general method to design screening tests for classification or regression based on ellipsoidal approximations of the optimal set. In addition to producing computational gains, our approach also allows us to compress a dataset into a subset of representative points.

\paragraph{This chapter is based on the following material:} 

\begin{center}
\begin{tcolorbox}[width=\linewidth, sharp corners=all, colback=white!95!black]
    \begin{itemize}
        \item[-] G. Mialon, A. d'Aspremont, J. Mairal. "Screening Data Points in Empirical Risk Minimization via Ellipsoidal Regions and Safe Loss Functions" (AISTATS, 2020).
    \end{itemize}
\end{tcolorbox}
\end{center}

\section{Introduction}
\label{sec:introduction}
Let us consider a collection of $n$ pairs $(a_i,b_i)_{i=1,\ldots,n}$, where each vector~$a_i$ in~$\Real^p$ describes a data point and~$b_i$ is its label.
For regression, $b_i$ is real-valued, and we address the convex optimization problem
\begin{equation}
   \min_{x \in \Real^p, t \in \Real^n} f(t) + \lambda R(x) \st   t = A x - b, \label{eq:regression}\tag{${\mathcal P}_1$}
\end{equation}
where $A = [a_1,\ldots,a_n]^\top$ in $\Real^{n \times p}$ carries the feature vectors, and $b=[b_1,\ldots,b_n]$ carries the labels. The function $f$ is a convex loss and measures the fit between data points and the model, and $R$ is a convex regularization function.
For classification, the scalars~$b_i$ are binary labels in $\{-1,+1\}$, and we consider instead of~\eqref{eq:regression} margin-based loss functions, where our problem becomes
\begin{equation}
   \min_{x \in \Real^p, t \in \Real^n} f(t) + \lambda R(x) \st   t=\diag(b)A x, \label{eq:classification}\tag{${\mathcal P}_2$}
\end{equation}
The above problems cover a wide variety of formulations such as
Lasso~\citep{tibshirani1996regression} and its
variants~\citep{zou2005regularization}, logistic regression, support vector
machines \citep{friedman2001elements}, and many more. When $R$ is the $\ell_1$-norm, the solution is encouraged to be
sparse~\citep{bach2012optimization}, which can be exploited to speed-up 
optimization procedures. 

A recent line of work has focused on screening tests that seek to automatically discard variables 
before running an optimization algorithm.
For example,~\cite{safe} 
derive a screening rule from Karush-Kuhn-Tucker conditions,
noting that if a dual optimal variable satisfies a given inequality
constraint, the corresponding primal optimal variable must be zero. 
Checking this condition on a set that is known to contain the
optimal dual variable ensures that the corresponding  primal variable can be safely removed.
This prunes
out irrelevant features {\em before} solving the problem. This is called a \textit{safe} rule if it discards variables that are guaranteed to be useless; but it is possible to relax the ``safety'' of the rules~\citep{tibshirani2012strong} without losing too much accuracy in practice. The seminal approach by~\cite{safe} has led to a series of works proposing refined tests~\citep{ellipsoids,wang2013lasso} or dynamic rules~\citep{safer_rules} for the Lasso, where screening is performed as the optimization algorithm proceeds, significantly speeding up convergence. Other papers have proposed screening rules for sparse logistic regression~\citep{reglog_screening} or other linear models.

Whereas the goal of these previous methods is to remove \emph{variables}, our
goal is to design screening tests for \emph{data points} in order to
remove observations that do not contribute to the final model.  The
problem is important when there is a large amount of ``trivial''
observations that are useless for learning. This typically occurs in {\em tracking or 
anomaly detection} applications, where a classical heuristic seeks to mine the data
to find difficult examples~\citep{felzenszwalb2009object}. A few of such screening tests for data points have been proposed in the
literature. Some are problem-specific (\textit{e.g.}~\cite{ogawa2014safe} for SVM), others are making strong assumptions on the objective. For
instance, the most general rule of~\cite{double_screening} for classification
requires strong convexity and the ability to compute a duality gap in closed
form. The goal of our work is to provide a more generic approach for screening data
samples, both for regression and classification. Such screening tests may be designed for loss functions that induce
a sparse dual solution. We describe this class of loss functions and investigate a regularization mechanism that ensures that the loss enjoys such a property.

\paragraph{Contributions.} Our contributions can be summarized as follows:

\begin{itemize}
    \item[-] We revisit the Ellipsoid method~\citep{ellipsoids_survey} to design screening test for samples, when the objective is convex and its dual admits a sparse solution.
    \item[-]  We propose a new regularization mechanism to design regression or classification losses that 
   induce sparsity in the dual. This allows us to recover existing loss functions and to discover new ones with sparsity-inducing properties in the dual.  
    \item[-] Originally designed for linear models, we extend our screening rules to kernel methods. Unlike the existing literature, our method also works for non strongly convex objectives.
    \item[-] We demonstrate the benefits of our screening rules in various numerical experiments on large-scale classification problems and regression\footnotemark[1]. 
\end{itemize}

\footnotetext[1]{Our code is available at \url{https://github.com/GregoireMialon/screening\_samples}.}

\section{Preliminaries}
\label{sec:tools}
We now present the key concepts used in this chapter.

\subsection{Fenchel Conjugacy}

\begin{definition}[Fenchel conjugate]
   Let $f: \Real^p \to \Real \cup \{-\infty,+\infty\}$ be an extended real-valued function. The Fenchel conjugate of $f$ is defined by
\begin{equation*}
    f^*(y) = \underset{t \in \mathbb{R}^p}{\max} \langle t, y \rangle - f(t).
\end{equation*}
\end{definition}
The biconjugate of $f$ is naturally the conjugate of $f^*$ and is denoted by $f^{**}$. The Fenchel-Moreau theorem~\citep{Hiri96} states that if $f$ is proper, lower semi-continuous and convex, then it is equal to its biconjugate $f^{**}$. Finally, Fenchel-Young's inequality gives for all pair $(t,y)$
\begin{equation*}
    f(t) + f^*(y) \geq \langle t, y \rangle,
\end{equation*}
with an equality case iff $y \in \partial f(t)$.

Suppose now that for such a function~$f$, we add a convex term $\Omega$ to $f^*$ in the definition of the biconjugate. We get a modified biconjugate $f_{\mu}$, written
\begin{align*}
     f_{\mu}(t) & = \underset{y \in \mathbb{R}^p}{\text{max }} \langle y, t \rangle - f^*(y) - \mu \Omega (y) \\
     & = \underset{y \in \mathbb{R}^p}{\text{max}} \langle y, t \rangle + \underset{z \in \mathbb{R}^p}{\text{min}} \left\{- \langle z, y \rangle + f(z) \right\} - \mu \Omega (y).
\end{align*}
The inner objective function is continuous, concave in~$y$ and convex in $z$, such that we can switch min and max according to Von Neumann's minimax theorem to get
\begin{align*}
     f_{\mu}(t) & = \underset{z \in \mathbb{R}^p}{\text{min }} f(z) + \underset{y \in \mathbb{R}^p}{\text{max }} \left\{ \langle t - z, y \rangle - \mu \Omega (y) \right\}\\
     & = \underset{z \in \mathbb{R}^p}{\text{min }} f(z) + \mu \Omega^*\left(\frac{t - z}{\mu}\right).
\end{align*}

\begin{definition}[Infimum convolution]
$f_{\mu}$ is called the infimum convolution of $f$ and $\Omega^*$, which may be written as $f ~\square~ \Omega^*$. 
\end{definition}

Note that $f_{\mu}$ is convex as the minimum of a convex function in $(t, z)$. We recover the Moreau-Yosida smoothing~\citep{moreau,yosida} and  its generalization when $\Omega$ is respectively a quadratic term or a strongly-convex term~\citep{Nest03}.  

\subsection{Empirical Risk Minimization and Duality}\label{subsec:dual}

Let us consider the convex ERM problem
\begin{equation}
\label{eq:erm}
   \min_{x \in \Real^p} P(x) = \frac{1}{n} \sum_{i=1}^{n} f_i(a_i^\top x) + \lambda R(x),
\end{equation}
which covers both~(\ref{eq:regression}) and~(\ref{eq:classification}) by using the appropriate definition of function $f_i$.
We consider the dual problem (obtained from Lagrange duality)
\begin{equation}
\label{eq:dual}
    \max_{\nu \in \Real^n} D(\nu) = \frac{1}{n} \sum_{i=1}^n - f_i^*(\nu_i) - \lambda R^*\left(-\frac{A^\top \nu}{\lambda n}\right).
\end{equation}
We always have $P(x) \geq D(\nu)$. Since there exists a pair $(x,t)$ such that $Ax=t$ (Slater's conditions), we have $P(x^\star) = D(\nu^\star)$ and $x^\star = -\frac{A^\top \nu^\star}{\lambda n}$ at the optimum.

\subsection{Safe Loss Functions and Sparsity in the Dual of ERM Formulations} A key feature of our losses is to encourage sparsity of dual solutions, which typically emerge from loss functions with a flat region. We call such functions ``safe losses'' since they will allow us to design safe screening tests.

\begin{definition}[Safe loss function]
\label{def:margin_loss}
Let $f: \Real \to \Real$ be a continuous convex loss function such that $\inf_{t \in \mathbb{R}} f(t) = 0$. We say that $f$ is a safe loss if there exists a non-singleton and non-empty interval $\mathcal{I} \subset \mathbb{R}$ such that
\begin{equation*}
    t \in \mathcal{I} \implies f(t) = 0.
\end{equation*} 
\end{definition}

\begin{lemma}[Dual sparsity]\label{lemma:margin}
\label{lemma:margin_sparsity}
   Consider the problem~(\ref{eq:erm}) where $R$ is a convex penalty. Denoting by $x^\star$ and $\nu^\star$ the optimal primal and dual variables respectively, we have for all $i =1,\ldots, n$,
   $$ \nu^\star_i \in \partial f_i(a_i^\top x^\star).$$
\end{lemma}

The proof can be found in Appendix~\ref{sec:add_proofs}.

\begin{remark}[Safe loss and dual sparsity]
A consequence of this lemma is that for both classification and
regression, the sparsity of the dual solution is related to loss functions that have ``flat'' regions---that is, such that $0 \in \partial f_i'(t)$. This is the case for safe loss functions defined above.
\end{remark}

The relation between flat losses and sparse dual solutions is classical, see~\cite{steinwart2004sparseness,blondel19}.

\section{Safe Rules for Screening Samples}
\label{sec:safe_lasso}
In this section, we derive screening rules in the spirit of SAFE~\citep{safe} to select data points in regression or classification problems with safe losses. 

\subsection{Principle of SAFE Rules for Data Points}

We recall that our goal is to safely delete data points prior to optimization, that is, we want to train the model on a subset of the original dataset while still getting the same optimal solution as a model trained on the whole dataset. 
This amounts to identifying beforehand which dual variables are zero at the optimum. Indeed, as discussed in Section~\ref{subsec:dual}, the optimal primal variable $x^\star = - \frac{A^\top \nu^\star}{\lambda n}$ only relies on non-zero entries of $\nu^\star$. 
To that effect, we make the following assumption:
\begin{assumption}[Safe loss assumption]\label{assum:safe}
   We consider problem~(\ref{eq:erm}), where each $f_i$ is a safe loss function. Specifically, we assume that $f_i(a_i^\top x) = \phi( a_i^\top x - b_i)$ for regression, or $f_i(a_i^\top x) = \phi( b_i a_i^\top x)$ for classification, where $\phi$ satisfies Definition~\ref{def:margin_loss} on some interval $\Ical$. For simplicity, we assume that there exists $\mu > 0$ such that $\Ical = [-\mu,\mu]$ for regression losses and $\Ical = [\mu, +\infty)$ for classification, which covers most useful cases.
\end{assumption}
We may now state the basic safe rule for screening.
\begin{lemma}[SAFE rule]
   Under Assumption~\ref{assum:safe}, consider a subset $\Xcal$ containing the optimal solution~$x^\star$. 
   If, for a given data point $(a_i, b_i)$, $ a_i^\top x  - b_i \in \mathring{\mathcal{I}}$ for all $x$ in $\Xcal$, (resp. $b_i  a_i^\top x  \in \mathring{\mathcal{I}}$), where $\mathring{\mathcal{I}}$ is the interior of $\mathcal{I}$,
   then this data point can be discarded from the dataset.
\end{lemma}

\begin{proof}
   From the definition of safe loss functions, $f_i$ is differentiable at $a_i^\top x^\star$ with $\nu_i^\star = f_i'(a_i^\top x) = 0$.
\end{proof}

We see now how the safe screening rule can be interpreted in terms of discrepancy between the model prediction $a_i^\top x$ and the true label $b_i$. If, for a set $\mathcal{X}$ containing the optimal solution $x^*$ and a given data point $(a_i, b_i)$, the prediction always lies in~$\Iint$, then the data point can be discarded from the dataset. The data point screening procedure therefore consists in \textit{maximizing linear forms}, $a_i^\top x - b_i$ and $-a_i^\top x + b_i$ in regression (resp. minimizing $b_i a_i^\top x$ in classification), over a set $\mathcal{X}$ containing $x^*$ and check whether they are lower (resp. greater) than the threshold $\mu$. The smaller $\mathcal{X}$, the lower the maximum (resp. the higher the minimum) hence the more data points we can hope to safely delete. Finding a good test region~$\Xcal$ is critical however. We show how to do this in the next section.

%
\subsection{Building the Test Region \texorpdfstring{$\mathcal{X}$}{X}}

Screening rules aim at sparing computing resources, testing a data point should therefore be easy. As in \cite{safe} for screening variables, if $\mathcal{X}$ is an ellipsoid, the optimization problem detailed above admits a closed-form solution. 
Furthermore, it is possible to get a smaller set $\mathcal{X}$ by adding a first order optimality condition with a subgradient $g$ of the objective evaluated in the center $z$ of this ellipsoid. This linear constraint cuts the final ellipsoid roughly in half thus reducing its volume.

\begin{lemma}[Closed-form screening test]\label{lemma:test}
    Consider the optimization problem
    \BEQ
    \BA{ll}
    \mbox{maximize} & a_i^\top x - b_i\\
    \mbox{subject to} &  (x - z)^\top E^{-1} (x - z) \leq 1 \\
    & g^T(x - z) \leq 0
    \EA\EEQ
    in the variable $x$ in $\mathbb{R}^p$ with $E$ defining an ellipsoid with center $z$ and $g$ is in $\mathbb{R}^p$. Then the maximum is
    \begin{equation*}
    \begin{cases}
    a_i^\top z + (a_i^\top E a_i)^{\frac{1}{2}} - b_i \text{ if } g^\top E a_i < 0 \\
    a_i^\top \left( z + \frac{1}{2 \gamma} E ( a_i - \nu g ) \right) - b_i \text{ otherwise},
    \end{cases}
    \end{equation*}
    with $ \nu = \frac{g^\top E a_i}{g^\top E g}$ and $\gamma = \left(\frac{1}{2} (a_i - \nu g)^\top E(a_i - \nu g)\right)^{\frac{1}{2}}$.
    \label{test}
\end{lemma}

The proof can be found in Appendix~\ref{sec:add_proofs} and it is easy to modify it for minimizing $b_i a_i^\top x$. We can obtain both $E$ and $z$ by using a few steps of \textit{the ellipsoid method} \citep{Nemi79,ellipsoids_survey}. This first-order optimization method starts from an initial ellipsoid containing the solution~$x^*$ to a given convex problem (here~\ref{eq:erm}) . It iteratively computes a subgradient in the center of the current ellipsoid, selects the half-ellipsoid containing $x^*$, and computes the ellipsoid with minimal volume containing the previous half-ellipsoid before starting all over again. Such a method, presented in Algorithm~\ref{algo:ell_method}, performs closed-form updates of the ellipsoid. It requires $O(p^2\log(\frac{RL}{\epsilon}))$ iterations for a precision $\epsilon$ starting from a ball of radius $R$ with the Lipschitz bound $L$ on the loss, thus making it impractical for accurately solving high-dimensional problems. Finally, the ellipsoid update formula was also used to screen primal variables for the Lasso problem
\citep{ellipsoids}, although not iterating over ellipsoids in order to get
smaller volumes.

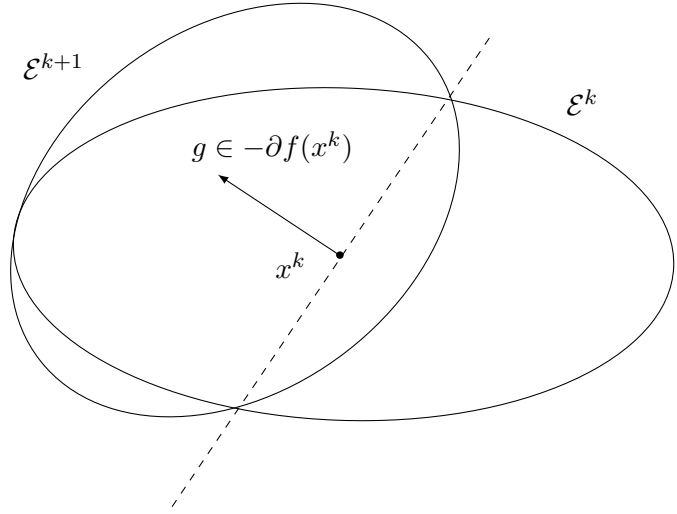
\begin{figure}
    \centering
    \begin{tikzpicture}[scale=1.2]
        \coordinate (a) at (-1.24,1.24);
        \coordinate (b) at (5,1);
        \ellipsebyfoci{draw}{a}{b}{1.16}
        \coordinate (c) at (-0.7661771140981, 0.4732063199958);
        \coordinate (d) at (2.1461475766973, 2.7417539738785);
        \ellipsebyfoci{draw}{c}{d}{1.47}
        \draw[dashed] (0,-1.67) -- (3.5,3.54);
        \draw[latex-] (0.5,2) --  (1.85,1.1);
        \draw (1.1, 2.3) node[scale=1] {$g \in -\partial f(x^k)$};
        \draw (1.3, 1) node[scale=1] {$x^k$};
        \draw (4.5, 2.8) node[scale=1] {$\mathcal{E}^k$};
        \draw (-1.3, 3.2) node[scale=1] {$\mathcal{E}^{k+1}$};
        \node at (1.84,1.11) [circle,fill,inner sep=1pt]{};
    \end{tikzpicture}
    \caption{One step of the ellipsoid method.}
    \label{fig:ellipsoid_method}
\end{figure}


\begin{algorithm}
\caption{Building ellipsoidal test regions}
\label{algo:ell_method}
\begin{algorithmic}[1]
\State \textbf{initialization:} Given $\mathcal{E}^0(x_0, E_0)$ containing $x^*$;
\While{$k < nb_{\text{steps}}$}
   \State $\sbullet$ Compute a subgradient $g$ of~(\ref{eq:erm}) in $x_k$;
  \State $\sbullet$ $\Tilde{g} \gets g / \sqrt{g^\top E_k g}$;
  \State $\sbullet$ $x_{k+1} \gets x_{k} - \frac{1}{p+1} E_k \Tilde{g}$;
  \State $\sbullet$ $E_{k+1} \gets \frac{p^2}{p^2 - 1}( E_k - \frac{2}{p + 1} E_k \Tilde{g} \Tilde{g}^\top E_k)$;
\EndWhile
\State For regression problems:
\For{each sample $a_i \text{ in } A$}
    \If{${\text{max}} |a_i^\top x - b_i| \leq \mu \text{ for } x \in \mathcal{E}^{nb_{\text{steps}}}$}
    \State Discard $a_i$ from $A$.
    \EndIf
\EndFor
\State For classification, replace condition $|a_i^\top x - b_i| \leq \mu$ by $b_i a_i^\top x \geq \mu$ in the above expression.
\end{algorithmic}
\end{algorithm}

\paragraph{Initialization.}
The algorithm requires an initial ellipsoid $\mathcal{E}^0(x_0, E_0)$ that contains the solution. This is
typically achieved by defining the center $x_0$ as an approximate solution of
the problem, which can be obtained in various ways. For instance, one may run a few
steps of a solver on the whole dataset, or one may consider the solution obtained previously
for a different regularization parameter when computing a regularization path, or the solution obtained
for slightly different data, \eg, for tracking applications where an optimization problem has to be solved at every time step $t$, with slight modifications from time $t-1$.

Once the center~$x_0$ is defined, there are many cases where the initial ellipsoid can
be safely assumed to be a sphere. For instance, if the objective---let us call it~$F$---is $\kappa$-strongly convex,
we have the basic inequality $\frac{\kappa}{2}\|x_0-x^\star\|^2 \leq F(x_0)-F^\star$, which can often be 
upper-bounded by several quantities, \eg, a duality gap~\citep{double_screening} or simply $F(x_0)$ if $F$ is non-negative as in typical ERM problems.
Otherwise, other strategies can be used depending on the problem at hand. If the problem is not strongly convex but constrained (\textit{e.g.} often a norm constraint in ERM problems), the initialization is also natural (\textit{e.g.}, a sphere containing the constraint set). We will see below that one of the most successful applications of screening methods is for computing regularization paths. Given that regularization path for penalized and constrained problems coincide (up to minor details), computing the path for a penalized objective amounts to computing it for a constrained objective, whose ellipsoid initialization is safe as explained above. Even though we believe that those cases cover many (or most) problems of interest, it is also reasonable to believe that guessing the order of magnitude of the solution is feasible with simple heuristics, which is what we do for $\ell_1$-safe logistic regression. Then, it is possible to check \textit{a posteriori} that screening was safe and that indeed, the initial ellipsoid contained the solution.
 
\paragraph{Efficient implementation.}
 Since each update of the ellipsoid matrix $E$ is rank one, it is possible to parametrize $E_k$ at step $k$ as
\begin{equation*}
    E_{k} = s_k \text{I} - L_kD_kL_k^\top,
\end{equation*} 
with $I$ the identity matrix, $L_k$ is in $\mathbb{R}^{p \times k}$ and $D_k$ in $\mathbb{R}^{k \times k}$ is a diagonal matrix. Hence, we only have to update $D$ and $L$ while the algorithm proceeds. 

\paragraph{Complexity of our screening rules.}
For each step of Algorithm~\ref{algo:ell_method}, we compute a subgradient $g$ in $O(np)$ operations. The ellipsoids are modified using rank one updates that can be stored. As a consequence, the computations at this stage are dominated by the computation of $Eg$, which is $O(pk)$. As a result, $k$ steps cost $O(k^2p + npk)$. Once we have the test set $\mathcal{X}$, we have to compute the closed forms from Lemma \ref{test} for each data point. This computation is dominated by the matrix-vector multiplications with $E$, which cost $O(kp)$ using the structure of $E$. Hence, testing the whole dataset costs $O(npk)$.
Since we typically have $n \gg k$, the cost of the overall screening procedure is therefore $O(n p k)$. In contrast, solving the ERM problem without screening would cost $O(n p T)$ where $T$ is the number of passes over the data, with $T \gg k$. With screening, the complexity becomes $O(n s T + np k)$, where $s$ is the number of data points accepted by the screening procedure. 

\subsection{Extension to Kernel Methods}

It is relatively easy to adapt our safe rules to kernel methods. Consider for example (\ref{eq:regression}), where $A$ has been replaced by $\phi(A) = [\phi(a_1),\ldots,\phi(a_n)]^\top$ in $\mathcal{H}^{n}$, with $\mathcal{H}$ a RKHS and $\phi$ its mapping function $\Real^p \rightarrow \mathcal{H}$. The prediction function $x \colon \Real^p \rightarrow \Real$ lives in the RKHS, thus it can be written $x(a) = \langle x, \phi(a) \rangle$, $\forall a \in \Real^p$. In the setting of an ERM strictly increasing with respect to the RKHS norm and each sample loss, the Representer theorem ensures $x(a) = \sum_{i=1}^n \alpha_i K(a_i,a)$ with $\alpha_i \in \Real$ and $K$ the kernel associated to $\mathcal{H}$. If we consider the squared RKHS norm as the regularization term, which is typically the case, the problem becomes:
\begin{equation}
   \min_{\alpha \in \Real^n, t \in \Real^n} f(t) + \lambda \sum_{i,j = 1}^n \alpha_i \alpha_j K(a_i, a_j) \st   t = \mathbf{K}\alpha - b, \label{eq:kernelized_regression}
\end{equation}
with $\mathbf{K}$ the Gram matrix. The constraint is linear in $\alpha$ (thus satisfying to Lemma~\ref{lemma:reg}) while yielding non-linear prediction functions. The screening test becomes maximizing the linear forms $[ \mathbf{K}]_i\alpha - b_i$ and $- [\mathbf{K}]_i \alpha + b_i$ over an ellipsoid $\mathcal{X}$ containing $\alpha^*$. When the problem is convex (it depends on $K$), $\mathcal{X}$ can still be found using the ellipsoid method.

We now have an algorithm for selecting data points in regression or classification problems with linear or kernel models. As detailed above, the rules require a sparse dual, which is not the case in general except in particular instances such as support vector machines. We now explain how to induce sparsity in the dual. 

\section{Constructing Safe Losses}
\label{sec:theory}
In this section, we introduce a way to induce sparsity in the dual of empirical risk minimization problems.


\subsection{Inducing Sparsity in the Dual of ERM}

When the ERM problem does not admit a sparse dual solution, safe screening is not possible. To fix this issue, consider the ERM problem~(\ref{eq:regression}) and replace $f$ by $f_\mu$ defined in Section~\ref{sec:tools}:
\begin{equation}
   \min_{x \in \Real^p, t \in \Real^n} f_\mu(t) + \lambda R(x) \st   t = A x - b, \label{eq:regression_mod}\tag{${\mathcal P}'_1$}
\end{equation}
%
We have the following result connecting the dual of \eqref{eq:regression} with that of \eqref{eq:regression_mod}.

\begin{lemma}[Regularized dual for regression]
    The dual of \eqref{eq:regression_mod} is
\BEQ \label{dual_formula}
\max_{\nu \in \Real^n} - \langle b, \nu \rangle - f^*(\nu) - \lambda R^*\left(-\frac{A^\top \nu}{\lambda}\right) - \mu \Omega(\nu),
\EEQ
and the dual of \eqref{eq:regression} is obtained by setting $\mu = 0$.
\label{lemma:reg}
\end{lemma}
The proof can be found in Appendix~\ref{sec:add_proofs}. We remark that is possible, in many cases,
to induce sparsity in the dual if $\Omega$ is the $\ell_1$-norm, or another
sparsity-inducing penalty. This is notably true if the unregularized dual is smooth with bounded gradients. In such a case, it is possible to show that the optimal dual solution would be $\nu^\star=0$ as soon as $\mu$ is large enough~\citep{bach2012optimization}.

We consider now the classification problem~(\ref{eq:classification}) and show that the previous remarks about sparsity-inducing regularization for the dual of regression problems
also hold in this new context.
\begin{lemma}[Regularized dual for classification]
Consider now the modified classification problem
\begin{equation}
   \min_{x \in \Real^p, t \in \Real^n} f_\mu(t) + \lambda R(x) \st   t=\diag(b)A x. \label{p_classif}\tag{${\mathcal P}_2'$}
\end{equation}
%
The dual of \ref{p_classif} is
\BEQ \label{dual_formula_classif}
\max_{\nu \in \Real^n} - f^*(-\nu) - \lambda R^*\left(\frac{A^\top \diag(b)\nu}{\lambda}\right) - \mu \Omega(-\nu).
\EEQ
\end{lemma}

\begin{proof}
We proceed as above with a linear constraint $\Tilde{A}\Tilde{x} = 0$ and $\Tilde{A} = (Id , - \diag(b)A)$.
\end{proof}

Note that the formula directly provides the dual of regression and classification ERM problems with a linear model such as the Lasso and SVM.

\subsection{Link Between the Original and Regularized Problems} 
\label{subsec:link}
The following results should be understood as an indication that $f$ and $f_{\mu}$ are similar objectives.

 \begin{lemma}[Smoothness of $f_{\mu}$]
     If $f^* + \Omega$ is strongly convex, then $f_{\mu}$ is smooth.
 \label{lemma:smoothness_f}
 \end{lemma}

 \begin{proof}
    The lemma follows directly from the fact that $f_{\mu} = (f^* + \mu \Omega)^*$ (see the proof of Lemma~\ref{lemma:reg}). The conjugate of a closed, proper, strongly convex function is indeed smooth (see \textit{e.g.}~\cite{Hiriart1993}, chapter X).
 \end{proof}

\begin{lemma}[Bounding the value of~\ref{eq:regression}]
\label{lemma:obj_ineq}
Let us denote the optimum objectives of \ref{eq:regression}, \ref{eq:regression_mod} by $P_{\lambda}$, $P_{\lambda, \mu}$. If $\Omega$ is a norm, we have the following inequalities:
\begin{equation*}
    P_{\lambda} - \delta^* \leq P_{\lambda, \mu} \leq P_{\lambda},
\end{equation*}
with $\delta^*$ the value of $\delta$ at the optimum of $P_{\lambda}(t) - \delta(t)$.
\end{lemma}


The proof can be found in Appendix~\ref{sec:add_proofs}. When $\mu \to 0$, $\delta(t) \to 0$ hence the objectives can be arbitrarily close.


\subsection{Effect of Regularization and Examples}
\label{subsec:examples}
We start by recalling that the infimum convolution is traditionally used for smoothing an objective when~$\Omega$ is strongly convex, and then we discuss the use of sparsity-inducing regularization in the dual.
%
%

\paragraph{Euclidean distance to a closed convex set.} It is known that convolving the indicator function of a closed convex set $\mathcal{C}$ with a quadratic term $\Omega$ (the Fenchel conjugate of a quadratic term is itself) yields the euclidean distance to $\mathcal{C}$
\begin{align*}
    f_{\mu}(t) = & \underset{z \in \mathbb{R}^n}{\text{min}} I_{\mathcal{C}}(z) + \frac{1}{2\mu}\|t - z\|_2^2
    =  \underset{z \in \mathcal{C}}{\text{min}} \frac{1}{2\mu}\|t - z\|_2^2.
\end{align*}

\paragraph{Huber loss.} The $\ell_1$-loss is more robust to outliers than the $\ell_2$-loss, but is not differentiable in zero which may induce difficulties during the optimization. A natural solution consists in smoothing it: \cite{huber} for example show that applying the Moreau-Yosida smoothing, \textit{i.e} convolving $|t|$ with a quadratic term $\frac{1}{2} t^2$ yields the well-known Huber loss, which is both smooth and robust:
\begin{equation*} f_{\mu}(t) = \begin{cases}
               \frac{t^2}{2 \mu} & \text{if } |t| \leq \mu, \\
              |t| - \frac{\mu}{2} & \text{otherwise}.
               \end{cases}
\end{equation*}

Now, we present examples where $\Omega$ has a sparsity-inducing effect.

\paragraph{Squared Hinge loss.} Let us now consider a problem with a quadratic loss $f \colon t \mapsto \| 1-t\|_2^2/2$ designed for a classification problem, and consider
$\Omega(x)= \|x\|_1 + \mathbf{1}_{x \preceq 0}$. We have $\Omega^*(y) = \mathbf{1}_{y \succeq -1}$, and
\begin{align*}
    f_{\mu}(t) = &
     \sum_{i=1}^n [1- t_i - \mu, 0]_+^2,
\end{align*}
which is a squared Hinge Loss with a threshold parameter~$\mu$ and $[.]_+ = \max(0,.)$.

\paragraph{Hinge loss.} Instead of the quadratic loss in the previous example, choose a robust loss $f \colon t \mapsto \|1-t\|_1$. By using the same function $\Omega$, we obtain the classical Hinge loss of support vector machines
$$
 f_\mu(t) = \sum_{i=1}^n \frac{1}{2}[1- t_i - \mu, 0]_+.
$$
We see that the effect of convolving with the constraint $\mathbf{1}_{x \preceq
0}$ is to turn a regression loss (\eg, square loss) into a classification loss.
The effect of the $\ell_1$-norm is to encourage the loss to be flat (when $\mu$ grows, $[1- t_i - \mu, 0]_+$ is equal to zero for a larger range of values $t_i$), which corresponds to the sparsity-inducing effect in the dual that we will exploit for screening data points.

\paragraph{Screening-friendly regression.}
\label{ex:sreg}
Consider now the quadratic loss $f: t \mapsto {\|t\|^2}/{2}$ and $\Omega(x) = \|x\|_1$. Then $\Omega^*(y) = {\mathbf 1}_{\|y\|_\infty \leq 1}$ (see \textit{e.g.}~\cite{bach2012optimization}), and
\begin{equation}
   f_\mu(t) = \sum_{i=1}^n \frac{1}{2}[|t_i|-\mu]_+^2. \label{eq:sreg}
\end{equation}
A proof can be found in Appendix~\ref{sec:add_proofs}. As before, the parameter $\mu$ encourages the loss to be flat (it is exactly $0$ when $\|t\|_\infty \leq \mu$).

\paragraph{Screening-friendly logistic regression.} Let us now consider the logistic loss $f(t) = \log{(1 + e^{-t})}$, which we define only with one dimension for simplicity here. It is easy to show that the infimum convolution with the $\ell_1$-norm does not induce any sparsity in the dual, because the dual of the logistic loss has unbounded gradients, making classical sparsity-inducing penalties ineffective.  However,
we may consider instead another penalty to fix this issue: $\Omega(x) = - x \log{(-x)} + \mu |x|$ for $x \in [-1,0]$. We have $\Omega^*(y) = - e^{y + \mu - 1}$. Convolving $\Omega^*$ with~$f$ yields
\begin{equation}
    f_{\mu}(x) = \begin{cases}
    e^{x + \mu - 1} - (x + \mu) & \: \text{if} \: x + \mu - 1 \leq 0, \\
    0 & \: \text{otherwise}.\label{eq:sclass}
    \end{cases}
\end{equation}
Note that this loss is asymptotically robust. Moreover, the entropic part of $\Omega$ makes this penalty strongly convex hence $f_{\mu}$ is smooth~\citep{Nest03}. Finally, the $\ell_1$ penalty ensures that the dual is sparse thus making the screening usable. Our regularization mechanism thus builds a smooth, robust classification loss akin to the logistic loss on which we can use screening rules. If $\mu$ is well chosen, the safe logistic loss maximizes the log-likelihood of the data for a probabilistic model which slightly differs from the sigmoid in vanilla logistic regression. The effect of regularization parameter in a few previous cases are illustrated in Figure~\ref{fig:curves}.

\begin{figure}
\centering
\begin{minipage}{0.45\linewidth}
\includegraphics[width=\linewidth]{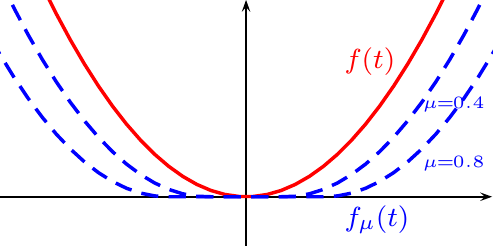}
\end{minipage} 
\begin{minipage}{0.45\linewidth}
\includegraphics[width=\linewidth]{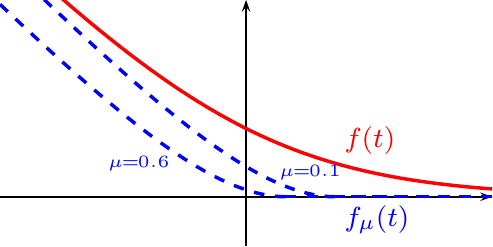}
\end{minipage}
\caption{Effect of the dual sparsity-inducing regularization on the quadratic loss~(\ref{eq:sreg}) (left) and logistic loss~(\ref{eq:sclass}) (right). After regularization, the loss functions have flat areas. Note that both of them are smooth.}\label{fig:curves}
\end{figure}

In summary, regularizing the dual with the $\ell_1$ norm induces a flat region in the loss, which induces sparsity in the dual. The geometry is preserved elsewhere. Note that we do not suggest to use~\ref{eq:regression_mod} and~\ref{p_classif} to screen for~\ref{eq:regression} and~\ref{eq:classification}.

\section{Experiments}
\label{sec:experiments}
We now present  experimental results demonstrating the effectiveness of the data screening procedure. 

\paragraph{Datasets.}
We consider three real datasets, SVHN, MNIST, RCV-1, and a synthetic one.
MNIST ($n=60000$) and SVHN ($n=604388$) both represent digits, which we encode by using the output of a two-layer convolutional kernel network \citep{mairal2016end} leading to feature dimensions $p=2304$. RCV-1 ($n=781265$) represents sparse TF-IDF vectors of categorized newswire stories ($p=47236$). For classification, we consider a binary problem consisting of discriminating digit 9 for MNIST vs. all other digits (resp. digit 1 vs rest for SVHN, 1st category vs rest for RCV-1).
For regression, we also consider a synthetic dataset, where data is generated by  
$b = Ax + \epsilon$,
where $x$ is a random, sparse ground truth, $A \in \mathbb{R}^{n \times p}$ a data matrix with coefficients in $[-1,1]$ and $\epsilon \sim \mathcal{N}(0, \sigma)$ with $\sigma = 0.01$. Implementation details are provided in Appendix. We fit usual models using Scikit-learn~\citep{scikit-learn} and Cyanure~\citep{Mairal2019Cyan} for large-scale datasets.

\subsection{Accuracy of our safe logistic loss} 

The accuracies of the Safe Logistic loss we build is similar to the accuracies obtained with the Squared Hinge and the Logistic losses on the datasets we use in this paper thus making it a realistic loss function, see Table~\ref{table:accuracies}.

\begin{table*}
\centering
\begin{tabular}{l | c | c | c}
\toprule
Dataset & MNIST & SVHN & RCV-1 \\ \midrule
Logistic + $\ell_1$ & 0.997 (0.01) & 0.99 (0.0003) & 0.975 (1.0) \\
Logistic + $\ell_2$ & 0.997 (0.001) & 0.99 (0.0003) & 0.975 (1.0) \\
\midrule
Squared Hinge + $\ell_1$ & 0.997 (0.03) & 0.99 (0.03) & 0.975 (1.0) \\
Squared Hinge + $\ell_2$ & 0.997 (0.003) & 0.99 (0.003) & 0.974 (1.0) \\
\midrule
Safelog + $\ell_1$ & 0.996 (0.0) & 0.989 (0.0) & 0.974 (1e-05) \\
Safelog + $\ell_2$ & 0.996 (0.0) & 0.989 (0.0) & 0.975 (1e-05) \\
\bottomrule
\end{tabular}
\caption{Averaged best accuracies on test set (best $\lambda$ in a logarithmic grid from $\lambda=0.00001$ to $1.0$).}
\label{table:accuracies}
\end{table*}

\subsection{Safe Screening}

Here, we consider problems that naturally admit a sparse dual solution, which allows safe screening.

\paragraph{Interval regression.} We first illustrate the practical use of the screening-friendly regression loss~\eqref{eq:sreg} derived above. It corresponds indeed to a particular case of a supervised learning task called interval regression \citep{hocking}, which is widely used in fields such as economics. In interval regression, one does not have scalar labels but intervals $\mathcal{S}_i$ containing the true labels $\Tilde{b}_i$, which are unknown. The loss is written
\begin{equation}
    \ell(x) = \sum_{i=1}^n \underset{b_i \in \mathcal{S}_i}{\text{inf}}(a_i^\top x - b_i)^2,
    \label{eq:general_ir}
\end{equation}
where $\mathcal{S}_i$ contains the true label $\Tilde{b}_i$. For a given data point, the model only needs to predict a value inside the interval in order not to be penalized. When the intervals $\mathcal{S}_i$ have the same width and we are given their centers $b_i$, ~\eqref{eq:general_ir} is exactly~\eqref{eq:sreg}. Since~\eqref{eq:sreg} yields a sparse dual, we can apply our rules to safely discard intervals that are assured to be matched by the optimal solution. We use an $\ell_1$ penalty along with the loss. As an illustration, the experiment was done using a toy synthetic dataset $(n = 20, p = 2)$, the signal to recover being generated by one feature only. The intervals can be visualized in Figure~\ref{fig:ir}. The ``difficult'' intervals (red) were kept in the training set. The predictions hardly fit these intervals. The ``easy'' intervals (blue) were discarded from the training set: the safe rules certify that the optimal solution will  fit these intervals. Our screening algorithm was run for 20 iterations of the Ellipsoid method. Most intervals can be ruled out afterwards while the remaining ones yield the same optimal solution as a model trained on all the intervals. 

\begin{figure}
\centering
\includegraphics[width=0.82\linewidth]{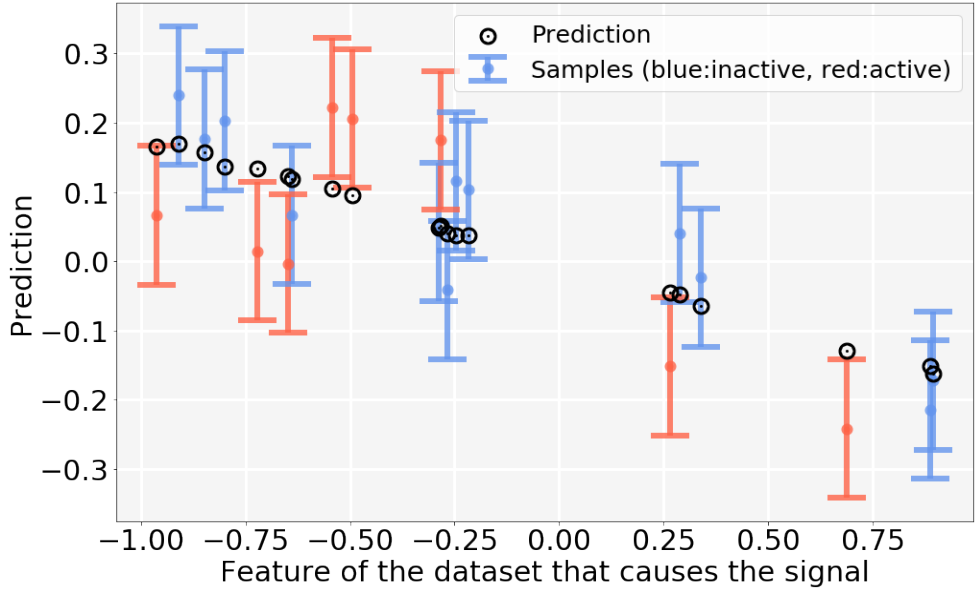}
\captionof{figure}{Safe interval regression on synthetic dataset. Most ``easy'' samples (in blue) can be discarded while the ``difficult'' ones (in red) are kept.}
\label{fig:ir}
\end{figure}

\paragraph{Classification.} 
  Common sample screening methods such as~\cite{double_screening} require a strongly convex objective. When it is not the case, there is, to the best of our knowledge, no baseline for this case. Thus, when considering classification using the non strongly convex safe logistic loss derived in Section~\ref{sec:theory} along with an $\ell_1$ penalty, our algorithm is still able to screen samples, as shown in Table~\ref{table:screened_frac_l1}. The algorithm is initialized using an approximate solution to the problem, and the radius of the initial ball is chosen depending on the number of epochs ($100$ for $10$ epochs, $10$ for $20$ and $1$ for $30$ epochs), which is valid in practice.

\begin{table*}
\centering
\begin{tabular}{l | c | c | c | c | c | r}
\toprule
Epochs & \multicolumn{3}{|c}{20} & \multicolumn{3}{|c}{30} \\
\midrule
$\lambda$ & MNIST & SVHN & RCV-1 & MNIST & SVHN & RCV-1 \\ \midrule
$10^{-3}$ & 0 & 0 & 1 & 0 & 2 & 12 \\ 
$10^{-4}$ & 0.3 & 0.01 & 8 & 27 & 17 & 42 \\ 
$10^{-5}$ & 35 & 12 & 45 & 65 & 54 & 75 \\ \bottomrule
\end{tabular}
\caption{Percentage of samples screened (\textit{i.e} that can be thrown away) in an $\ell_1$ penalized Safe Logistic loss given the epochs made at initialization. The radius is initialized respectively at $10$ and $1$ for MNIST and SVHN at Epochs $20$ and $30$, and at $1$ and $0.1$ for RCV-1.}
\label{table:screened_frac_l1}
\end{table*}

\begin{table}
\centering
\begin{tabular}{l | c | c | c | r}
\toprule
Epochs & \multicolumn{2}{|c}{20} & \multicolumn{2}{|c}{30} \\ \midrule
$\lambda$ & MNIST & SVHN & MNIST & SVHN \\ \midrule
$1.0 $ & 89 / 89 & 87 / 87 & 89 / 89 & 87 / 87\\ 
$10^{-1}$ & 95 / 95 & 11 / 47 & 95 / 95 & 91 / 91\\
$10^{-2}$ & 16 / 84 & 0 / 0 & 98 / 98 & 90 / 92\\ 
$10^{-3}$ & 0 / 0 & 0 / 0 & 34 / 50 & 0 / 0 \\ \bottomrule
\end{tabular}
\caption{Percentage of samples screened in an $\ell_2$ penalized SVM with Squared Hinge loss (Ellipsoid (ours) / Duality Gap~\citep{double_screening}) given the epochs made at initialization.}
\label{table:screened_frac_strongly_convex}
\end{table}
The Squared Hinge loss allows for safe screening (see~\ref{lemma:margin_sparsity}). Combined with an $\ell_2$ penalty, the resulting ERM is strongly convex. We can therefore compare our Ellipsoid algorithm to the baseline introduced by~\cite{double_screening}, where the safe region is a ball centered in the current iterate of the solution and whose radius is  $\frac{2\Delta}{\lambda}$ with $\Delta$ a duality gap of the ERM problem. Both methods are initialized by running the default solver of scikit-learn with a certain number of epochs. The resulting approximate solution and duality gap are subsequently fed into our algorithm for initialization. Then, we perform one more epoch of the duality gap screening algorithm on the one hand, and the corresponding number of ellipsoid steps computed on a subset of the dataset on the other hand, so as to get a fair comparison in terms of data access. The results can be seen in Table~\ref{table:screened_frac_strongly_convex} for MNIST and SVHN, and in Table~\ref{table:safe_sqhinge_rcv1} for RCV-1.
While being more general (our approach is neither restricted to classification, nor requires strong convexity), our method performs similarly to the baseline. Figure~\ref{fig:tradeoff} highlights the trade-off between optimizing and evaluating the gap (Duality Gap Screening) versus performing one step of Ellipsoid Screening. Both methods start screening after a correct iterate (i.e. with good test accuracy) is obtained by the solver (blue curve) thus 
suggesting that screening methods would rather be of practical use when computing a regularization path, or when the computing budget is less constrained (e.g. tracking or anomaly detection) which is the object of next paragraph.

\begin{table}[H]
\centering
\begin{tabular}{l | c | r}
\toprule
Epochs & 10 & 20\\ \midrule
$\lambda = 1$ & 7 / 84 & 85 / 85 \\ 
$\lambda = 10$ & 80 / 80 & 80 / 80 \\ 
$\lambda = 100$ & 68 / 68 & 68 / 68 \\ \bottomrule
\end{tabular}
\caption{RCV-1 : Percentage of samples screened in an $\ell_2$ penalized SVM with Squared Hinge loss (Ellipsoid (ours) /
Duality Gap) given the epochs made at initialization.}
\label{table:safe_sqhinge_rcv1}
\end{table}

\begin{figure}
\centering
\includegraphics[width=0.95\linewidth]{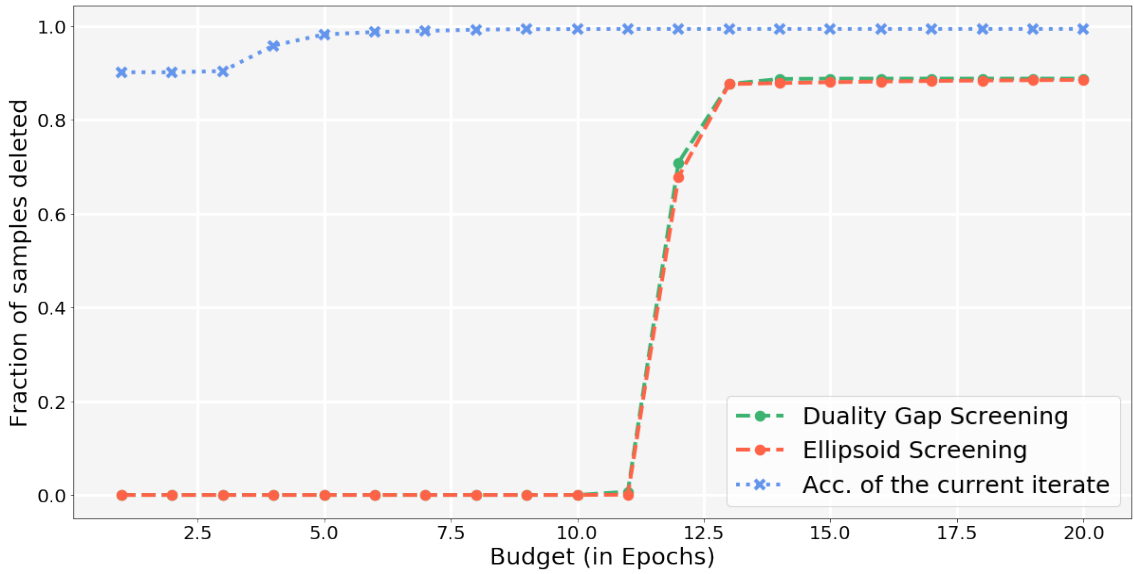}
\captionof{figure}{Fraction of samples screened vs Epochs done for two screening strategies along with test accuracy of the current iterate (Sq. Hinge + $\ell_2$ trained on MNIST).}
\label{fig:tradeoff}
\end{figure}

\paragraph{Computational gains.} As demonstrated in Figure~\ref{fig:comp_gains}, computational gains can indeed be obtained in a regularization path setting (MNIST features, Squared Hinge Loss and L2 penalty). Each point of both curves represents an estimator fitted for a given lambda against the corresponding cost (in epochs). Each estimator is initialized with the solution to the previous parameter lambda. On the orange curve, the previous solution is also used to initialize a screening. In this case, the estimator is fit on the remaining samples which further accelerates the path computation.
\begin{figure}
\centering
\includegraphics[width=0.95\linewidth]{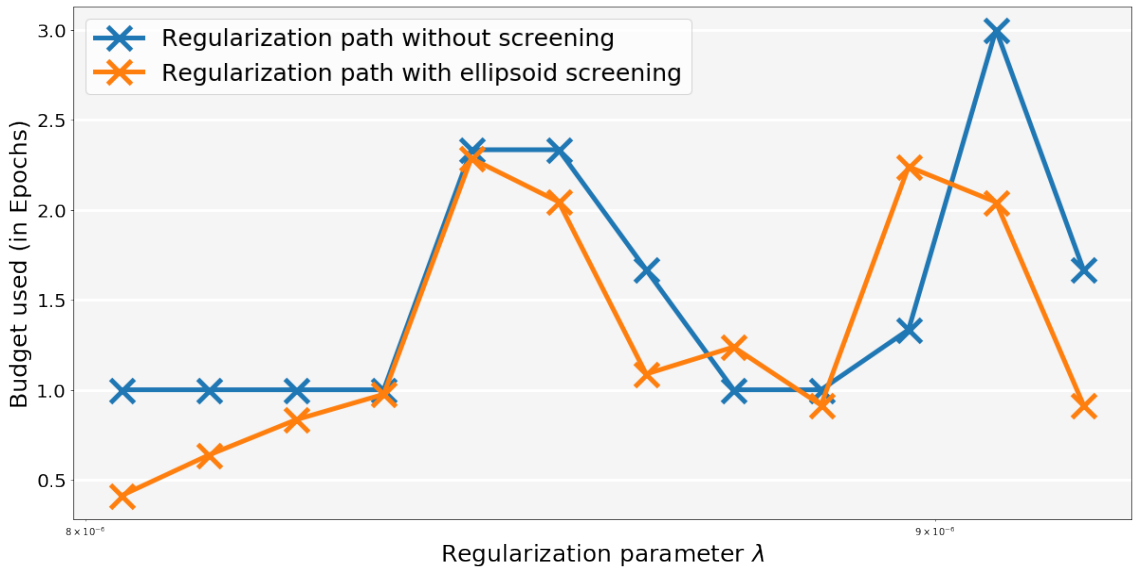}
\captionof{figure}{Regularization path of a Squared Hinge SVM trained on MNIST. The screening enables computational gains compared to a classical regularization path.}
\label{fig:comp_gains}
\end{figure}
\subsection{Dataset Compression}
We now consider the problem of dataset compression, where the goal is to maintain a good accuracy while using less examples from a dataset. This section should be seen as a proof of concept. A natural scheme consists in choosing the samples that have a higher margin since those will carry more information than samples that are easy to fit. In this setting, our screening algorithm can be used for compression by using the scores of the screening test as a way of ranking the samples. In our experiments, and for a given model, we progressively delete data points according to their score in the screening test for this model, before fitting the model on the remaining subsets. We compare those methods to random deletions in the dataset and to deletions based on the sample margin computed on early approximations of the solution when the loss admits a flat area (``margin screening''). 

\paragraph{Classification.} We first apply our compression scheme in the context of classification, using again MNIST, SVHN, and RCV-1. We test the combinations of the $\ell_1$ penalty with the safe logistic loss, and the $\ell_2$ penalty with the squared hinge loss. Our strategy is effective, as can be seen in Figure~\ref{fig:compression_classif}.

\begin{figure}
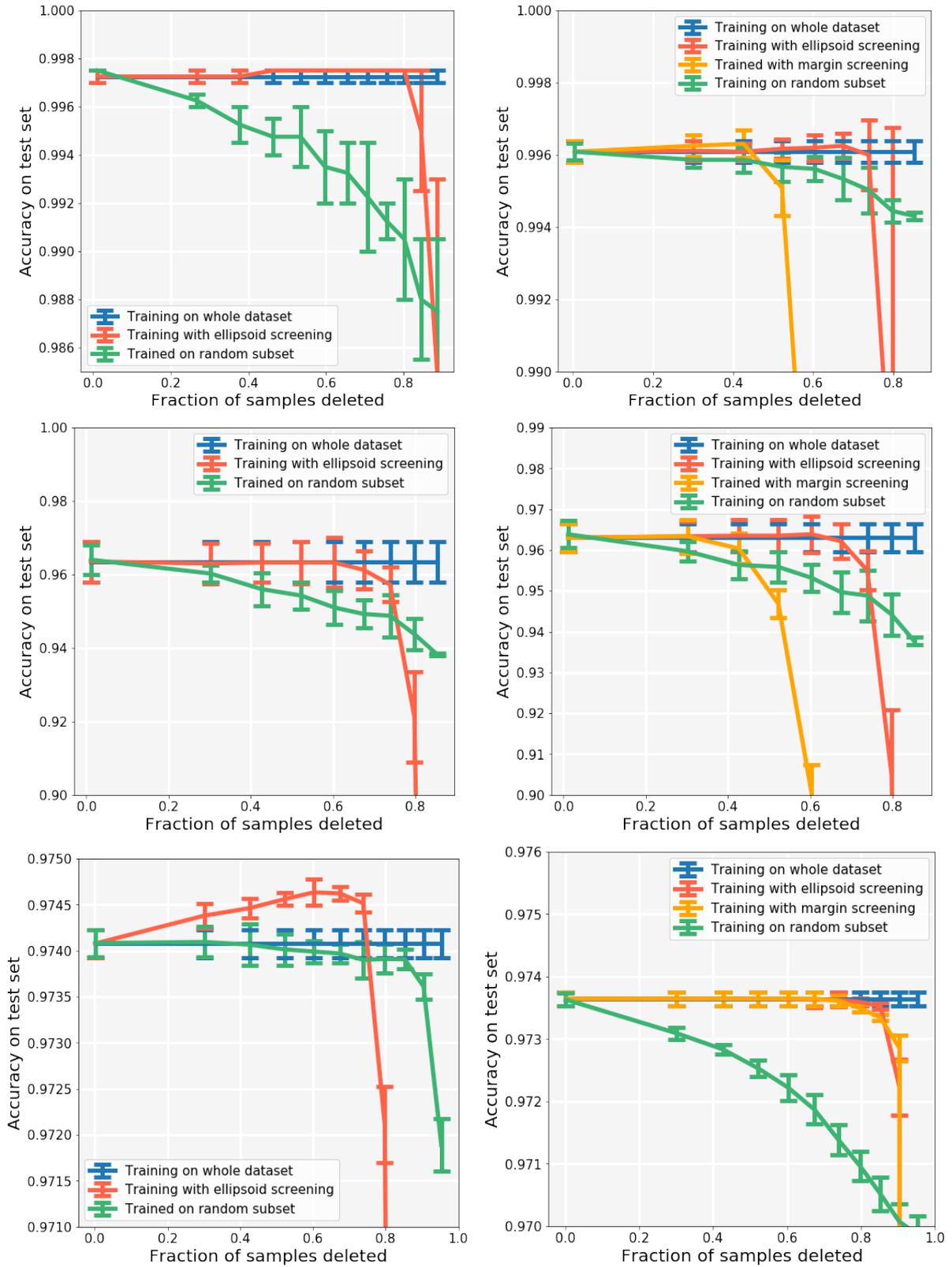

  \begin{subfigure}{.5\textwidth}
    \centering
    \includegraphics[width=0.95\linewidth]{chapters/2_samplescreening/figures/compression_mnist_safelog.pdf}
  \end{subfigure} \hfill
  \begin{subfigure}{.5\textwidth}
    \centering
    \includegraphics[width=0.95\linewidth]{chapters/2_samplescreening/figures/compression_mnist_squared_hinge.pdf}
  \end{subfigure}
  
  \begin{subfigure}[t]{.5\textwidth}
    \centering
    \includegraphics[width=0.95\linewidth]{chapters/2_samplescreening/figures/compression_svhn_safelog.pdf}
  \end{subfigure} \hfill
  \begin{subfigure}[t]{.5\textwidth}
    \centering
    \includegraphics[width=0.95\linewidth]{chapters/2_samplescreening/figures/compression_svhn_squared_hinge.pdf}
  \end{subfigure}

  \begin{subfigure}[t]{.5\textwidth}
    \centering
    \includegraphics[width=\linewidth]{chapters/2_samplescreening/figures/compression_rcv1_safelog.pdf}
  \end{subfigure} \hfill
  \begin{subfigure}[t]{.5\textwidth}
    \centering
    \includegraphics[width=\linewidth]{chapters/2_samplescreening/figures/compression_rcv1_squared_hinge.pdf}
  \end{subfigure}

\caption{Dataset compression in classification. \textit{Left:} $\ell_1$ penalty + Safe Logistic loss. \textit{Right:} $\ell_2$ penalty + Squared Hinge loss. \textit{Top:} MNIST. \textit{Middle:} SVHN. \textit{Down:} RCV-1.}
  \label{fig:compression_classif}
\end{figure}

\paragraph{Lasso regression.} The Lasso objective combines an $\ell_2$ loss with an $\ell_1$ penalty.
Since its dual is not sparse, we will instead apply the safe rules offered by the screening-friendly regression loss~\eqref{eq:sreg} derived in Section \ref{subsec:examples} and illustrated in~Figure~\ref{fig:curves}, combined with an $\ell_1$ penalty.
We can draw an interesting parallel with the SVM, which is naturally sparse in data points. At the optimum, the solution of the SVM can be expressed in terms of data points (the so-called support vectors) that are close to the classification boundary, that is the points that are \textit{the most difficult} to classify. Our screening rule yields the analog for regression: the points that are easy to predict, i.e. that are close to the regression curve, are less informative than the points that are harder to predict. 
In our experiments on \emph{synthetic data} ($n=100$), this does consistently better than random subsampling as can be seen in Figure~\ref{fig:synthetic_compression}.  

\begin{figure}
\centering
  \includegraphics[width=0.9\linewidth]{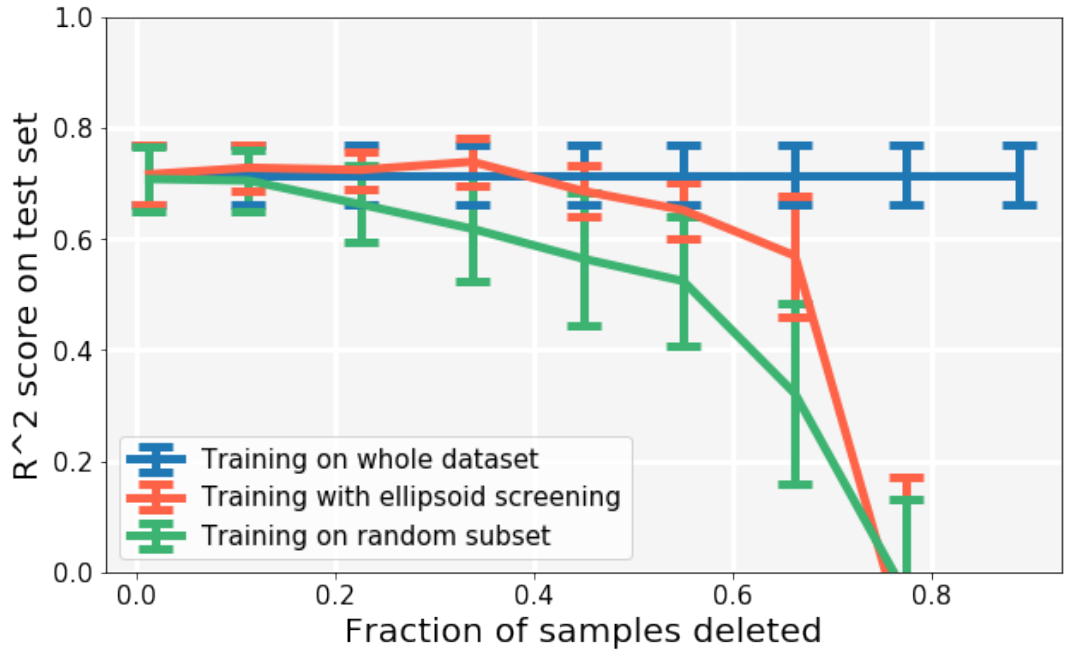}
  \captionof{figure}{Dataset compression for the Lasso trained on a synthetic dataset. The scores given by the screening yield a ranking that is better than random subsampling.}
  \label{fig:synthetic_compression}
\end{figure}

\paragraph{Discussion.} For all methods, the degradation in performance is lesser than with random deletions. Nevertheless, in the regime where most samples are deleted (beyond $80\%$), random deletions tend to do better. This is not surprising since the screening deletes the samples that are ``easy'' to classify. 
Then, only the difficult ones and outliers remain, making the prediction task harder compared to a random subsampling.

\section{Conclusion}
\label{sec:conclusion}
In this chapter, we saw how common losses could be regularized in their dual to yield safe losses. However, adding a regularization term enforcing sparsity in the dual does not always result in a loss with a flat region. For example, we showed in Section~\ref{subsec:examples} that the logistic loss could not be regularized with the $\ell_1$ norm only. Thus, it could be insightful to better understand which regularization is needed for which functions to yield safe losses. More generally, the work initiated in Section~\ref{subsec:link} to study the link between an ERM problem and its regularized dual could be further pursued. 

A limit of this work is the size of the screening region $\mathcal{X}$, whose price for being tractable and generic (\textit{i.e} not requiring strong convexity of the objective) is its high volume, making typical safe screening (\textit{i.e} deleting samples before solving the ERM problem) too computationally intensive to be beneficial in most problems. Would it be possible to find a region with smaller volume while remaining generic and tractable?

Finally, safe screening techniques where introduced in 2010, at a time when machine learning often consisted in working on a dataset on its laptop or working machine. In this context, reducing the size of a dataset in terms of features and/or samples was valuable to fit models with lower running time. The advent of deep learning popularized the use of massive datasets for training models via clusters of computers, thus making screening rules slightly less relevant today. However, and as stated in introduction, screening rules are still of use to speed solvers up, such as for the Lasso~\citep{massias2018celer}, and may have other surprising applications, such as in differential privacy. More precisely, safe losses might be used in the context of black-box membership inference (see for example~\cite{sablayrolles2019white} for a more detailed introduction on these problems). In this setting, the parameters of the model are unknown but it is possible to query it. The objective is to guess whether a particular sample was part of the training set. When samples are in the flat region of a safe loss at the optimum, it means that they are non-support, hence can be deleted without modifying the loss or the model. It may therefore be impossible to guess whether such samples have been used for training. In the future, we plan to work on a new notion of differential privacy, where a model would be private for non-support samples and public for the support samples.

\clearpage

\vspace*{0.3cm}
\begin{center}
   {\huge \textbf{Appendix}}
\end{center}
\vspace*{0.5cm}

\section{Proofs.}
\label{sec:add_proofs}

%

\subsection{Proof of Lemma~\ref{lemma:margin}}
\begin{proof}
At the optimum,
\begin{align*}
    P(x^*) - D(\nu^*) ={} \frac{1}{n} \sum_{i=1}^n f_i(a_i^\top x) + f_i^*(\nu_i) + \\ \lambda R(x) + \lambda R^*\left(-\frac{A^T \nu}{\lambda n}\right) = 0.
\end{align*}
Adding the null term $\langle x, - \frac{A^\top \nu}{n} \rangle - \langle x, - \frac{A^\top \nu}{n} \rangle$ gives
\begin{align*}
    \frac{1}{n} \sum_{i=1}^n \underbrace{f_i(a_i^\top x) + f_i^*(\nu_i) - a_i^\top x \nu_i}_{\geq 0} + \\ \lambda \underbrace{ \left( R(x) + R^*\left(-\frac{A^\top \nu}{\lambda n}\right) - \left\langle x, - \frac{A^\top \nu}{\lambda n} \right\rangle \right)}_{\geq 0} = 0,
\end{align*}
since Fenchel-Young's inequality states that each term is greater or equal to zero. We have a null sum of non-negative terms; hence, each one of them is equal to zero. We therefore have for each $i = 1 \dots n$:
\begin{equation*}
    f(a_i^\top x) + f^*(\nu_i) =  a_i^\top x \nu_i,
\end{equation*}
which corresponds to the equality case in Fenchel-Young's relation, which is equivalent to $\nu^*_i \in \partial f_i(a_i^\top x^*)$. 
\end{proof}

\subsection{Proof of Lemma~\ref{lemma:test}}
\begin{proof}
\label{closed_form_optim}
 The Lagrangian of the problem writes:
\begin{align*}
    L(x, \nu, \gamma) = a_i^\top x - b_i + \nu \left( 1 - (x - z)^TE^{-1}(x - z) \right) - \\ \gamma g^T(x - z),
\end{align*}
with $\nu, \gamma \geq 0$. When maximizing in $x$, we get:
\begin{align*}
    \frac{\partial L}{\partial x} & = a_i + 2 \nu (E^{-1}z - E^{-1}x) - \gamma = 0.
\end{align*}
We have $\nu > 0$ since the opposite leads to a contradiction. This yields $x = z + \frac{1}{2 \nu}(Ea_i - \gamma Eg)$ and $(x - z)^T E^{-1} (x - z) = 1$ at the optimum which gives $\nu = \frac{1}{2}\sqrt{(a_i - \gamma)^T E (a_i - \gamma)}$. 

Now, we have to minimize
\begin{align*}
   g(\nu, \gamma) = a_i\left(z + \frac{1}{2\nu}(Ea_i - \gamma Eg)\right) - \\ \gamma^\top\left(\frac{1}{2\nu}(Ea_i - \gamma Eg)\right). 
\end{align*}
To do that, we consider the optimality condition
\begin{align*}
    \frac{\partial g}{\partial \gamma} & = - \frac{1}{2\nu}a_iEg - \frac{1}{2\nu}g^TEa_i + \frac{\gamma}{\nu} g^TEg = 0,
\end{align*}
which yields $\gamma = \frac{g^TEa_i}{g^TEg}$. If $g^TEa_i < 0$ then $\gamma = 0$ in order to avoid a contradiction.

In summary, either $g^TEa_i \leq 0$ hence the maximum is attained in $x = z + \frac{1}{2\nu}Ea_i$ and is equal to $a_iz + \sqrt{a_i^T E a_i} - y_i$, or $g^TEa_i > 0$ and the maximum is attained in $x = z + \frac{1}{2\nu}E(a_i - \gamma Eg)$ and is equal to $a_i\left(z + \frac{1}{2 \nu}E(a_i - \gamma g)\right) - b_i$ with $\nu = \frac{1}{2}\sqrt{(a_i - \gamma)^T E (a_i - \gamma)}$ and $\gamma = \frac{g^TEa_i}{g^TEg}$.
\end{proof}

\subsection{Proof of Lemma~\ref{lemma:reg}}

\begin{proof} We can write \ref{eq:regression_mod} as
\BEQ
\BA{ll}
\mbox{minimize} & \Tilde{f}(\Tilde{x}) + \lambda \Tilde{R}(\Tilde{x})\\
\mbox{subject to} & \Tilde{A}\Tilde{x} = - b
\EA\EEQ
   in the variable $\Tilde{x} = (t,x) \in \reals^{n + p}$ with $\Tilde{f} \colon \Tilde{x} \mapsto f_{\mu}(t) $ and $\Tilde{R} \colon \Tilde{x} \mapsto R(x)$ and $\Tilde{A} \in \mathbb{R}^{n \times (n+p)} = \left( \text{Id} , - A \right)$. Since the constraints are linear, we can directly express the dual of this problem in terms of the Fenchel conjugate of the objective (see \textit{e.g.} \cite{boyd_van}, 5.1.6). Let us note $f_0 = \Tilde{f} + \lambda \Tilde{R}$. For all $y \in \mathbb{R}^{n+p}$, we have
\begin{align*}
    f_0^*(y) & = \underset{x \in \mathbb{R}^{n+p}}{\text{sup}} \langle x, y \rangle - \Tilde{f}(x) - \lambda \Tilde{R}(x) \\
    & = \underset{x_1 \in \mathbb{R}^{n}, x_2 \in \mathbb{R}^p}{\text{sup}}
    \langle x_1, y_1 \rangle + \langle x_2, y_2 \rangle - f(x_1) - \lambda R(x_2) \\
    & = f_{\mu}^*(y_1) + \lambda R^*\left(\frac{y_2}{\lambda}\right).
\end{align*}
It is known from~\cite{huber} that $f_{\mu} = f ~\square~ \Omega^*_{\mu} = (f^* + \Omega_{\mu}^{**})^*$ with $\Omega_{\mu}^* = \mu \Omega^*(\frac{.}{\mu})$. Clearly, $\Omega_{\mu}^{**} = \mu \Omega$. If $\Omega$ is proper, convex and lower semicontinuous, then $\Omega = \Omega^{**}$ . As a consequence, $f_{\mu}^* = (f^* + \mu \Omega)^{**}$. If $f^* + \mu \Omega$ is proper, convex and lower semicontinuous, then $f_{\mu}^* = f^* + \mu \Omega$, hence
\begin{equation*}
    f_0^*(y) = f^*(y_1) + \lambda R^*\left(\frac{y_2}{\lambda}\right) + \mu \Omega(y_1).
\end{equation*}
Now we can form the dual of \ref{eq:regression_mod} by writing 
\BEQ
\BA{ll}
\mbox{maximize} & - \langle - b, \nu \rangle - f_0^*(-\Tilde{A}^T\nu)
\EA\EEQ
in the variable $\nu \in \mathbb{R}^n$. Since $-\Tilde{A}^T \nu = (-\nu, A^T \nu)$ with $\nu \in \mathbb{R}^n$ the dual variable associated to the equality constraints,
\[
    f_0^*(-\Tilde{A}^T \nu) = f^*(-\nu) + \lambda R^*\left(\frac{A^T \nu}{\lambda}\right) + \mu \Omega(-\nu).
\]
Injecting $f_0^*$ in the problem and setting $\nu$ instead of $-\nu$ (we optimize in $\mathbb{R}$) concludes the proof.
\end{proof}

\subsection{Lemma~\ref{lemma:bounding_f}}

 \begin{lemma}[Bounding $f_{\mu}$]
      If $\mu \geq 0$ and $\Omega$ is a norm then
      \begin{equation*}
          f(t) - \delta(t) \leq f_{\mu}(t) \leq f(t),\quad \mbox{for all $t \in \mathrm{dom} f$}
      \end{equation*}
 with  $\delta(t) = \underset{\|\frac{u}{\mu}\|^* \leq 1}{\max} g^Tu$ and $g \in \partial f(t)$.
 \label{lemma:bounding_f}
 \end{lemma}

\begin{proof}
    If $\Omega$ is a norm, then $\Omega(0)=0$ and $\Omega^*$ is the indicator function of the dual norm of $\Omega$ hence non-negative. Moreover, if $\mu > 0$ then, $\forall z \in \text{dom}f$ and $\forall t \in \mathbb{R}^n$,
    \begin{equation*}
        f_{\mu}(t) \leq f(z) + \mu \Omega^*\left(\frac{t - z}{\mu}\right).
    \end{equation*}
    In particular, we can take $t = z$ hence the right-hand inequality. On the other hand,
    \begin{align*}
        f_{\mu}(t) - f(t) &= \underset{z}{\min} f(z) + \mu I_{\|\frac{z - t}{\mu}\|^* \leq 1} - f(t)\\
        & = \underset{\|\frac{u}{\mu}\|^* \leq 1}{\min} f(t + u) - f(t).
    \end{align*}
    Since $f$ is convex,
    \begin{equation*}
        f(t + u) - f(t) \geq g^Tu \text{ with } g \in \partial f(t).
    \end{equation*}
    As a consequence, 
    \begin{equation*}
        f_{\mu}(t) - f(t) \geq \underset{\|\frac{u}{\mu}\|^* \leq 1}{\min} g^Tu.
    \end{equation*}
\end{proof}

\subsection{Proof of Lemma~\ref{lemma:obj_ineq}}

\begin{proof}
   The proof is trivial given the inequalities in Lemma~\ref{lemma:bounding_f}.
\end{proof}

\subsection{Proof of Screening-friendly regression}

\begin{proof}   The Fenchel conjugate of a norm is the indicator function of the unit ball of its dual norm, the $\ell_\infty$ ball here. Hence the infimum convolution to solve
\begin{equation} \label{eq:lasso_loss_modified}
    f_{\mu}(x) = \underset{z \in \mathbb{R}^n}{\text{min }} \{f(z) + \mathbf{1}_{\|x - z\|_{\infty} \leq \mu}\}
\end{equation}
Since $f(x) = \frac{1}{2n} \|x\|_2^2$,
\begin{equation*}
    f_{\mu}(x) = \underset{z \in \mathbb{R}^{n}}{\text{min }} \frac{1}{2n} z^Tz + \mathbf{1}_{\|x - z\|_{\infty} \leq \mu}.
\end{equation*}
If we consider the change of variable $t = x - z$, we get:
\begin{equation*}
    f_{\mu}(x) = \underset{t \in \mathbb{R}^n}{\text{min }} \frac{1}{2n} \|x - t\|_2^2 + \mathbf{1}_{\|t\|_{\infty} \leq \mu}.
\end{equation*}
The solution $t^*$ to this problem is exactly the proximal operator for the indicator function of the infinity ball applied to $x$. It has a closed form
\begin{align*}
t^* & = \text{prox}_{\mathbf{1}_{\|.\|_{\infty} \leq \mu}}(x) \\
    & = x - \text{prox}_{\left(\mathbf{1}_{\|.\|_{\infty} \leq \mu}\right)^*}(x),
\end{align*}
using Moreau decomposition. We therefore have
\begin{align*}
t^* & = x - \text{prox}_{\mu \|.\|_1}(x). 
\end{align*}
Hence,
\[
    f_{\mu}(x) = \frac{1}{2n} \| x - t^*\|_2^2 = \frac{1}{2n} \| \text{prox}_ {\mu \|.\|_1}(x)\|_2^2.
\]
But, $\text{prox}_ {\mu \|.\|_1}(t) = \text{sgn}(t) \times [|t| - \mu]_+$ for $t \in \mathbb{R}$, where $[x]_+ = \text{max}(x, 0)$. 
\end{proof}

\section{Additional experimental results.}
\label{sec:add_results}
\paragraph{Reproducibility.}  The data sets did not require any pre-processing except \emph{MNIST} and \emph{SVHN} on which exhaustive details can be found in~\citet{mairal2016end}. For both regression and classification, the examples were allocated to train and test sets using scikit-learn's \textit{train-test-split} ($80\%$ of the data allocated to the train set). The experiments were run three to ten times (depending on the cost of the computations) and our error bars reflect the standard deviation. For each fraction of points deleted, we fit three to five estimators on the screened dataset and the random subset before averaging the corresponding scores. The optimal parameters for the linear models were found using a simple grid-search.

\cleardoublepage

\chapter{Conclusion and Future Work}
\section{Conclusion}

\paragraph{Main conclusions.} 
 In this thesis, we showed how inductive biases could be integrated into models using kernel methods. In the context of sets of features, and starting from the optimal transport kernel between two sets of features, which is interpretable, it is possible to derive an embedding which is reminiscent of transformers~\citep{mialon2021}. We can also integrate the structure of a graph in transformers using kernels on the graph~\citep{mialon2021graphit}. This creates a promising architecture for graphs that differs from GNNs. On a different note, we explored surprising properties of losses admitting a flat area, which can be used for data pruning, with for example potential applications in anomaly detection or differential privacy~\citep{mialon2020screening}. All these projects offer interesting directions for future work. 

\paragraph{Relationship between transfer learning from pre-trained models and inductive biases.} At the beginning of this thesis, we introduced inductive biases as a point of view on the generalization problem that appeared to be opposed to pre-trained models: whereas one requires plenty of data, the other is meant to be data efficient.
As demonstrated in Chapter~\ref{chapt:2_otke} and Chapter~\ref{chapt:3_graphit}, these practices can be complementary: it is possible to use a module relying on inductive biases on top of a pre-trained model in a context where labelled data is scarce. This approach was shown to be very effective in the context of classifying protein foldings~\citep{mialon2021}.
Another example is the architecture of a CNN, which can be seen as an inductive bias focused on image processing. Indeed, small or medium convolution filters exploit locality properties of natural images. It is believed that such an inductive bias is nearly optimal for image processing: self-attention in vision transformers seems to learn the same kind of inductive bias~\citep{cordonnier2020on}, thus requiring more data to attain the performance of CNNs~\citep{dosovitskiy2021an}. Note however that the aforementioned work also suggests that when the amount of data grows to be huge -- one or two orders of magnitude bigger than ImageNet -- transformers perform better than CNNs\footnote{Recent work however shows that improved training procedure of transformers can close the gap with CNNs trained on ImageNet only although the training procedure differs: the transformer is trained using knowledge distilled from a CNN teacher~\citep{touvron2021training}.}.
In summary, inductive biases can also be useful in context where data is big but not huge. In the latter case, the picture is not clear as of today: the quantity of data and thz size of the model seems more important than its architecture as suggested by the success of different computer vision architectures such as ResMLP~\citep{touvron2021resmlp}, DeiT~\citep{touvron2021training}, and BiT~\citep{kolesnikov2020big} \footnote{Thus making~\href{http://www.incompleteideas.net/IncIdeas/BitterLesson.html}{the bitter lesson of machine learning} maybe more true than ever.}.

\paragraph{Relationship between inductive biases and geometric deep learning.} An exciting field of deep learning coined as geometric deep learning recently emerged~\citep{bronstein2021geometric}. From the point of view of geometric deep learning, and for many problems, there exists regularities under the form of symmetries and invariances. For example, the energy of a molecule is invariant given any rotation of the molecule in the 3D space. These regularities strongly suggest a preferred architecture, and the zoo of deep learning architectures can in fact be classified depending on the invariance they encode. As an example, a transformer is permutation equivariant: any permutation of two elements in its input will result in the same permutation in the output. This makes the transformer an appropriate model for handling unordered sets of features. Some inductive biases discussed in this thesis can naturally be put under the hood of geometrical deep learning, see our discussion on different possible models for handling graphs in Chapter~\ref{chapt:3_graphit}, between GNNs and transformers. Other inductive biases such as using an optimal transport geometry for representing sets may require more thinking to find their place in the geometric deep learning perspective. 

\section{Future work}

At the end of the chapters, we provided short term directions for further work. In this last section, we detail two general avenues of work starting from the conclusion of this thesis.

\paragraph{Applications of inductive biases: scientific discovery.} Although inductive biases are typically outperformed in the more and more common huge data setting by large pre-trained models, they are still an important tool for machine learning. We believe scientific discovery is an appropriate application for recent machine learning models. First, as opposed to models that will be deployed in production for use by non-experts, which must comply with the norm, often requiring interpretability, guarantees on the error, and other constraints that are not yet compatible with deep learning, scientific machine learning models may stay in the lab. This means that norms are less binding and more human supervision is allowed, thus enabling to leverage the most recent advances in deep learning. Second, scientific data has often been generated by some potentially complex physical law, meaning that there is probably properties such as invariances in the problem or more generally domain knowledge to incorporate in the learning model.
AlphaFold2~\citep{jumper2021} is a good example for our discussion: first, it is both a complex system relying on deep neural networks (a core block of AlphaFold2 is a deep transformer), yet also a model integrating expert knowledge on proteins via inductive biases (the triangular attention mechanism helps the network to learn a pairwise representations of amino-acids which is compatible with a 3D structure in terms of pairwise distances). As AlphaFold2 should largely be used by experts in bioinformatics, there is to the best of our knowledge no norm to comply with that would be incompatible with deep learning blocks. As many interesting scientific data can be seen as graphs or sequences, and because scientific data may be scarce for some domains such as are disease or econometrics studies, we believe a promising application of our work to be science problems. 

\paragraph{Inductive biases are elsewhere, for example in different learning paradigms.} This thesis focused on integrating inductive biases directly within models. But, inductive biases can be found elsewhere, namely in the learning procedure. The main learning paradigm in machine learning is Empirical Risk Minimization (ERM) that we studied in Chapter~\ref{chapt:4_screening}. ERM consists in minimizing a prediction error on hypothetically i.i.d samples with possibly regularity constraints for the model to hopefully ensure better generalization. Such a learning procedure suffers from various issues, two of them being that labels are required, and that spurious correlations may be learned. The former requires manually annotated samples which can be tedious, expensive or even impossible to obtain.
The latter makes generalization more difficult: a famous example being a classifier that would learn to recognize cows. As cow pictures are often taken in green grass, the classifier trained on such data could have difficulties to recognize a cow lying on a beach~\citep{beery2018recognition}. Different learning procedures can be used to circumvent these issues. In computer vision, self-supervised learning relies on a pretext task to learn general visual representations~\citep{he2020momentum}. These approaches typically aim to minimize the difference between two representations of two different augmentations of the same image. These augmentations are human-engineered: in this sense, this component is an inductive bias and could be used to inject domain knowledge for various data modalities. Invariant Risk Minimization~\citep{arjovsky2020invariant} (IRM) tackles the problem of spurious correlations in ERM: the inductive bias here is that the model should learn correlations that are stable across different environments, \textit{i.e} different training distributions. In this sense, inductive biases are broader than integrating domain knowledge into the learning architecture, and progress may be sought by improving the learning paradigms used in machine learning.  

\cleardoublepage

%
%

\newpage

{\small \bibliographystyle{plainnat}
\bibsep 1ex
\addcontentsline{toc}{chapter}{Bibliography}
\bibliography{bib_files/biblio_one_file.bib}}

\end{document}